\documentclass[twoside,11pt]{article}

\usepackage{blindtext}

\usepackage{jmlr2e}

\usepackage{lastpage}

\ShortHeadings{Approximate full-conformal prediction in an RKHS}{Razafindrakoto, Celisse and Lacaille}
\firstpageno{1}

\begin{document}

\title{Approximate full-conformal prediction in an RKHS}

\author{\name Davidson Lova Razafindrakoto
      \email davidson-lova.razafindrakoto@proton.me \\
      \addr Laboratoire SAMM\\
      Universit{\'e} Paris 1 Panth{\'e}on-Sorbonne\\
      90, rue de Tolbiac\\
      75634 PARIS CEDEX 13, FRANCE
      \AND
      \name Alain Celisse
      \email alain.celisse@univ-paris1.fr\\
      \addr Laboratoire SAMM\\
      Université Paris 1 Panth{\'e}on-Sorbonne\\
      90, rue de Tolbiac\\
      75634 PARIS CEDEX 13, FRANCE
      \AND
      \name J{\'e}r{\^o}me Lacaille
      \email jerome.lacaille@safrangroup.com \\
      \addr Safran Aircraft Engines\\
      Rond-Point Ren\'e Ravaud, R{\'e}au,\\
      77550 Moissy-Cramayel CEDEX, France}

\editor{}

\maketitle

\begin{abstract}
    Full-conformal prediction is a framework that 
    implicitly formulates distribution-free confidence prediction-regions
    for a wide range of estimators.
    However, a classical limitation of the full-conformal framework is the computation of the confidence prediction-regions, which is usually impossible since it requires training infinitely many estimators (for real-valued prediction for instance).
    The main purpose of the present work is to describe a generic strategy for designing a tight approximation to the full-conformal prediction-region that can be efficiently computed.
    Along with this approximate confidence-region, a theoretical quantification of the tightness of this approximation is developed, depending on the smoothness assumptions on the loss- and score-functions.
    The new notion of thickness is introduced for quantifying the discrepancy between the approximate confidence-region and the full-conformal one.
    %
    
\end{abstract}

\begin{keywords}
RKHS, Kernel ridge regression, confidence prediction-region, conformal prediction, influence function
\end{keywords}

\section{State of the art}

From a given predictor $\hat f$ evaluated at $X\in\mathcal{X}$, conformal prediction \citep{vovk2005algorithmic} is a framework that implicitly formulates a \emph{confidence prediction-region} (depending on $\hat f(X)$), which contains the unobserved $Y\in \mathcal{Y}$ with a prescribed confidence control-level.
The first formulated strategy for conformal prediction is called the full-conformal prediction \citep[Section~2.2.4]{vovk2022}. However full-conformal prediction requires training the underlying predictor as many times as the cardinality of the set $\mathcal{Y}$ of outputs of the predictor \citep[Section~2.2.4]{vovk2022}. This strong limitation makes the full-conformal prediction-region usually impossible to compute in practice.

    


Alternatives to full-conformal prediction have been designed to overcome this difficulty: \emph{split-conformal} \citep{papadopoulos2008inductive} or
\emph{cross-conformal} \citep{vovk2015cross} to enumerate but a few.
By contrast with full-conformal, both split- and cross-conformal prediction only require to compute a fixed (small) number of predictors: as many predictors as the number of splits of the data into training and test sets.
The main drawback of these approximations is the loss of information they incur by only considering a subset of the available observations for training. This loss of information usually results in larger prediction-regions in practice (see Section~3.7 in \citet{papadopoulos2008inductive} and Section~4.1 in \citet{vovk2022}). 
%


Another line of research exploits specific settings where closed-form formulas are available which express the trained predictor as a function of the possible output value. For instance \citet{nouretdinov2001ridge} bypass the computational bottleneck using this trick in the context of ridge regression, and \citet{burnaev2016conformalized} address kernel ridge regression similarly.
This approach has been further extended to regularized variants of linear regression such as Elastic Net and LASSO by \citet{ndiaye2019computing} and \citet{lei2019fast}.
More recently, \citet{ndiaye2023root} exploit the assumed hat-shape of the p-value function and resort to root-finding methods for computing full-conformal prediction-region.
Unfortunately such an explicit computation of the full-conformal prediction-region is limited to specific settings and remains prohibited in full generality.

    

The present work rather describes a \emph{generic (distribution-free) strategy for approximating the full-conformal prediction-region} without any preliminary splitting of the data.
The idea of designing such an approximation is not new. Uniform stability-bounds are harnessed by \citet{ndiaye2022stable} for building an approximate full-conformal prediction-region, which requires training only one predictor. Although their approximation contains the full-conformal one, no theoretical guarantees are provided on the size of the approximating region. Moreover, the computed approximation suffers from the worst-case behavior encoded by the uniform stability notion used at the core of the approximation.
In the context of classification, \citet{martinez2023approximating}
approximate a trained predictor by means of \emph{Influence Functions} \citep{chernick1982influence}.
This approximation allows them to derive \emph{asymptotic-confidence prediction-regions}, which converge to the full-conformal prediction-region as the number of observations grows.

\medskip

By contrast, the main contribution of the present work is to provide a generic approach for efficiently computing a \emph{non-asymptotic} \emph{approximate} full-conformal prediction-region. 
This approximation does contain the full-conformal prediction-region with a prescribed precision level. 
It also comes with non-asymptotic bounds quantifying the quality of the approximation.

To be more specific, Section~\ref{sec.conformal.prediction} briefly introduces the main concepts about conformal prediction (full-conformal prediction-region, conformal $p$-values, scores, \dots) and notations that are useful along this paper.
Section~\ref{sec.non.smooth} describes the generic scheme leading to a new effective approximation to the full-conformal (prediction) region that is tractable in practice (Section~\ref{sec.approx.prediction.region}). This approximation is derived for a family of Tikhonov-like predictors (Section~\ref{sec.predictor}) defined in a Reproducing Kernel Hilbert Space (RKHS) \citep{aronszajn1950theory} for various loss-functions including robust ones.
By contrast with the previous section, where almost no
specific assumptions are made on the loss or the score-functions, the purpose of Sections~\ref{sec.smooth} and~\ref{sec.very.smooth} is to illustrate how the present strategy can be instantiated when stronger smoothness assumptions are fulfilled by the loss- and score-functions. In particular influence functions are leveraged along Section~\ref{sec.very.smooth} to provide a tighter approximation to the full-conformal prediction-region. Let us emphasize that the tightness of these approximations is quantified in terms of volume (and not only in terms of coverage-probability).
%



\section{Conformal prediction}
\label{sec.conformal.prediction}

The main purpose of the present section is to briefly introduce the main concepts of conformal prediction and discuss the strengths and weaknesses of the existing approaches. 
Section~\ref{sec.fcp} focuses on full-conformal prediction and explains reasons for the computational bottleneck.
More efficient approximations to the full-conformal approach are then exposed along Section~\ref{sec.scp}, namely the split- and cross-conformal approximations. Their practical limitations
are also briefly exposed.

\subsection{Full-conformal prediction}
\label{sec.fcp}

Assume $D = \{(X_1, Y_1), \ldots, (X_{n}, Y_{n})\}$ denotes a set of $n$ \emph{independent and identically distributed} random variables with unknown probability distribution $P$ on $\mathcal{X}\times \mathcal{Y} \subset \mathbb{R}^d \times \mathbb{R}$.

Given a new \emph{independent} $(X_{n+1},Y_{n+1}) \sim P$, assume that one only observes $X_{n+1}$.
The purpose of conformal prediction is to design a \emph{confidence prediction-region} $\widehat{C}_{\alpha}\paren{X_{n+1}}$ based on a predictor $\hat f$ evaluated at $X_{n+1}$ such that, for any control-level $\alpha \in (0,1)$,
\begin{align}\label{eq.confidence.prediction.region}
    \mathbb{P}\croch{ Y_{n+1} \in \widehat{C}_{\alpha}\paren{X_{n+1}} } \geq 1-\alpha.
\end{align}

The \emph{full-conformal (prediction-) region} $\fcpr$ \citep{vovk2005algorithmic} writes as
\begin{align}
\label{eq.conformalRegion}
    \fcpr
    = 
    \left\{ 
        y \in \mathcal{Y} : 
        \fcpv{y}> \alpha
    \right\},
\end{align}
where $\fcpv{\bullet}$ denotes the so-called \emph{conformal p-value function} defined in the following Definition~\ref{def.conformal.pvalues} from \emph{non-conformity scores}.
\begin{definition}(Non-conformity score and Conformal p-value)
\label{def.conformal.pvalues}
Let $s(\bullet,\bullet):$ $\mathcal{Y}\times \mathcal{Y} \to \mathbb{R}_+$ denote
a \emph{non-conformity function} \citep[Definition 1.3]{balasubramanian2014conformal}, and
$\hat{f}_{D^y}$ stand for a predictor trained from $D^y := \{(X_{1}, Y_{1}), \ldots, (X_{n+1}, y)\}$, for every $y\in\mathcal{Y}$.
Then the non-conformity score of the couple $(X_i,Y_i)$ given $y$ and the predictor $\hat{f}_{D^y}$, is defined as
\begin{align}  
\label{def.scores}  
\begin{aligned}
    S_{D^{y}} \paren{X_i, Y_i} & := s\paren{Y_i, \hat{f}_{D^y}(X_i)},
&&\mbox{if }1\leq i\leq n \\
S_{D^{y}} \paren{X_{i}, y} & := s\paren{y, \hat{f}_{D^y}(X_{i})},
&&\mbox{if } i=n+1.
\end{aligned}
\tag{NCScores}  
\end{align}
Furthermore, the conformal p-value function $\fcpv{\bullet}$ is given by
\begin{equation}
\label{def.conformal.pvalue.function}
    \tag{Cp-value}
        y \in \mathcal{Y}
        \mapsto
        \fcpv{y}
        := \frac{1
            +
            \sum_{i = 1}^{n}
            \mathbbm{1}\left\{
                S_{D^{y}} \paren{X_i, Y_i}
                \geq
                S_{D^{y}} \paren{X_{n+1}, y}
            \right\}
        }{n+1}.
    \end{equation}
\end{definition}
In the above definition, an important ingredient is the non-conformity function $s(\bullet, \bullet)$. Intuitively $S_{D^{y}}\paren{X_i, Y_i}$ measures the discrepancy
between the output $Y_i$ and the prediction of $\hat{f}$ at $X_i$ when $(X_{n+1}, y)$ is added to the training data set $D$.
This can be seen as a measure of the "non-conformity" of $(X_i, Y_i)$ with regards to the observed points $D$ and $(X_{n+1}, y)$ through the predictor $\hat{f}$
(for more examples of non-conformity functions
see \citet{balasubramanian2014conformal},
\citet[Section 2.2.9]{vovk2022}, and \citet{kato2023review}).
The p-value function measures how much ``conform" $(X_{n+1}, y)$ is relative to all other points through their \emph{non-conformity} score values.
In particular if $y$ seems to take values in accordance with the prediction output by $\hat{f}_{D^y}$ at $X_{n+1} $,
then $S_{D^{y}}\paren{X_{n+1}, y}$ is expected to be small compared to the other scores $S_{D^{y}}\paren{X_i, Y_i}$ and the corresponding conformal p-value must be large.

The main motivation for considering conformal prediction stems from the next property which guarantees a minimum confidence (called coverage) in a \emph{distribution-free} setting.
\begin{theorem}[in Section 2.2.5 of Vovk, Gammerman, and Shafer, 2022]
\label{thm.coverage}
\sloppy 
Assume $(X_{1}, Y_{1}), \ldots, (X_{n+1}, Y_{n+1})$ are exchangeable and $\hat{f}_D$ is invariant to permutations of the data in $D$. 
Then, for any control-level $\alpha \in (0,1)$, the full-conformal prediction-region $\fcpr$ enjoys the following coverage guarantee
\begin{equation*}
    \mathbb{P}\croch{Y_{n+1} \in \fcpr} \geq 1 - \alpha.
\end{equation*}
As such $\fcpr$ is called a \emph{confidence prediction-region}.
\sloppy Moreover, if the non-conformity scores $S_{D^{Y_{n+1}}} \paren{X_1, Y_1}$, $\ldots$, $S_{D^{Y_{n+1}}} \paren{X_{n+1}, Y_{n+1}}$ are almost-surely distinct, then 
\begin{equation*}
    \mathbb{P}\croch{Y_{n+1} \in \fcpr} \leq 1 - \alpha + \frac{1}{n+1}.
\end{equation*}
\end{theorem}
The second property ensures that the full-conformal prediction-region is not arbitrarily conservative as $n$ grows in terms of coverage.
Regarding the computational aspects, a closed-form expression for the predictor $\hat{f}_{D^{y}}$ as a function of $y$ does not exist in general. Therefore, the computation of the full-conformal prediction-region (see Eq.~\ref{eq.conformalRegion}) is not possible. 
Indeed, a brute-force approach would consist in going through all possible values of $y$, which will require training $\mathrm{Card}(\mathcal{Y})$ predictors. With $\mathcal{Y} = \mathbb{R}$ for instance, the computational cost would be infinite. That is why the \emph{full-conformal} approach is usually left aside for less costly alternative approximations such as \emph{split-conformal} \citep{papadopoulos2008inductive}, and
\emph{cross-conformal} \citep{vovk2015cross}
to name but a few.


\subsection{Approximations}
\label{sec.scp}

The following approximations attempt to overcome the computational bottleneck by redefining the \textit{conformal p-value function}. It then requires training only a finite number of predictors (at least one and at most $n$). The ensuing prediction-region is formulated
as in Eq.~\eqref{eq.conformalRegion} by only replacing the \textit{full-conformal p-value function} with its corresponding approximation.

\bigskip
\noindent{\emph{split-conformal \citep{papadopoulos2008inductive}.}}
The \emph{split-conformal p-value function} is computed in two steps. 
First, a subset $D_{\mathrm{train}}$ of $D$ is used to train a predictor $\hat{f}_{D_{\mathrm{train}}}$. 
Then, the \emph{non-conformity scores} are evaluated over $D_{\mathrm{cal}} = D \setminus D_{\mathrm{train}}$
of size $n_{\mathrm{cal}}$,
with the corresponding set of indices $J_{\mathrm{cal}}$ as, for every $y \in \mathcal{Y}$
\begin{align*}
\begin{aligned}
    S_{D_{\mathrm{train}}} \paren{X_i, Y_i} 
    &:= 
    s\paren{Y_i, \hat{f}_{D_{\mathrm{train}}}(X_i)},
    &&\mbox{if } i \in J_{\mathrm{cal}}
    \\
    S_{D_{\mathrm{train}}} \paren{X_i, y}
    &:= s\paren{y, \hat{f}_{D_{\mathrm{train}}}(X_{i})},
    &&\mbox{if } i=n+1.
\end{aligned}
\end{align*}
The resulting \emph{conformal p-value function} is given by, for every $y \in \mathcal{Y}$,
\begin{align}
    \label{def.split.pvalue}
    \scpv{y} =
    \frac{1
        + \sum_{i \in J_{\mathrm{cal}}}
        \mathbbm{1} \left\{
            S_{D_{\mathrm{train}}} \paren{X_i, Y_i}
            \geq S_{D_{\mathrm{train}}} \paren{X_{n+1}, y}
        \right\}
    }{n_{\mathrm{cal}}+1}.
\end{align}
%

If the observations in $D_{\mathrm{cal}}$ and $\paren{X_{n+1}, Y_{n+1}}$
are exchangeable,
then the split-conformal region $\scpr$ ensures the following guarantee,
\begin{align}
\label{eq.scp.coverage.control}
    1 - \alpha
    \leq \mathbb{P}\croch{Y_{n+1} \in \scpr}
    \leq 1 - \alpha + \frac{1}{n_{\mathrm{cal}}+1},
\end{align}
where the upper-bound holds if the non-conformity scores
$S_{D_{\mathrm{train}}} \paren{X_i, Y_i}$, where $i \in J_{\mathrm{cal}}$, and
$S_{D_{\mathrm{train}}} \paren{X_{n+1}, Y_{n+1}}$ are almost-surely distinct.

The computational bottleneck induced by the need to train $\mathrm{Card}(\mathcal{Y})$ predictors with full-conformal prediction is avoided by training only one predictor on $D_{\mathrm{train}}$.
However, the number of points used to train the predictor is reduced to
$\mathrm{Card}(D_{\mathrm{train}})$, the number of \emph{non-conformity scores} to
$\mathrm{Card}(D_{\mathrm{cal}})$, and finally the number of predictors to one.
These computational improvements come at the price of a less informative prediction-region since it incorporates a less precise predictor, deteriorates the quality of the \emph{non-conformity} scores, and uses fewer scores, which makes the p-value function less pertinent (see Section 3.7 in \citet{papadopoulos2008inductive} and Section 4.1 in \citet{vovk2022}).

Although the upper-bound indicates that coverage-probability gets closer to
$1 - \alpha$ as $n_{\mathrm{cal}}$ grows,
this does not prevent the split-conformal prediction-region to be arbitrarily
more conservative than the full-conformal prediction-region.
For instance, the split-conformal prediction-region could end up
covering areas with lower density (due to less informative predictor and/or
non-conformity functions), which is balanced by an increase
of the size of the prediction-region (to maintain the coverage).

Our approximation (detailed in Section~\ref{sec.approx.prediction.region}) does not require splitting the data but similarly overcomes the computational bottleneck by requiring the training of only one predictor. Moreover, the number of data used to train the predictor and the number of scores for computing the p-value function are kept the same as in full-conformal prediction.
Let us also mention that unlike the split-conformal prediction-region for which  there is no quantification of how close it is to the full-conformal prediction-region (except in terms of coverage-probability), such a quantification is available for our approximation (see Section~\ref{sec.approx.prediction.region}).


\bigskip
\noindent{\emph{Cross-conformal \citep{vovk2015cross}.}}
In the cross-conformal framework, the \textit{conformal p-value function}
is computed from $V$ predictors trained over
$V$ folds from $D$. This is similar to the cross-validation idea \citep[see][for a survey]{arlot2010survey}.
Let $\left(D^{(1)}_{\mathrm{cal}}, \ldots, D^{(V)}_{\mathrm{cal}}\right)$ be a random partitioning of $D$ into $V$ folds,
with the corresponding set of indices $\left(J_{\mathrm{cal}}^{(1)}, \ldots, J_{\mathrm{cal}}^{(V)}\right)$,
and define for each $k \in \brac{1, \ldots, V}$, $D^{(k)}_{\mathrm{train}} := D \setminus D^{(k)}_{\mathrm{cal}}$.
Then for each fold $k \in \brac{1, \ldots, V}$,
the ``non-conformity" scores are given by, for every $y \in \mathcal{Y}$
\begin{align*}
\begin{aligned}
    S_{D^{(k)}_{\mathrm{train}}} \paren{X_i, Y_i}
    & := 
    s\left(Y_i, \hat{f}_{D^{(k)}_{\mathrm{train}}}(X_i)\right),
    &&\mbox{if } i \in J_{\mathrm{cal}}^{(k)}
    \\
    S_{D^{(k)}_{\mathrm{train}}} \paren{X_i, y}
    & := 
    s\left(y, \hat{f}_{D^{(k)}_{\mathrm{train}}}(X_{i})\right),
    &&\mbox{if }i=n+1.
\end{aligned}
\end{align*}

The resulting cross-\emph{conformal p-value}
$\ccpv{V}{\bullet}$ is given by, for every $y \in \mathcal{Y}$,
\begin{align}
\label{def.cross.pvalue}
    \ccpv{V}{y}
    :=
    \frac{1
    +\sum_{k=1}^{V}
    \sum_{i\in J_{\mathrm{cal}}^{(k)}}
    \mathbbm{1}
    \brac{
        S_{D^{(k)}_{\mathrm{train}}} \paren{X_i, Y_i}
        \geq S_{D^{(k)}_{\mathrm{train}}} \paren{X_{n+1}, y}
    }}{n+1}.
\end{align}
This strategy is more expensive than the split-conformal one.
Unlike the split-conformal p-values,
the above ones from Eq.~\eqref{def.cross.pvalue} require 
computing $V$ different predictors
(one for each of the $V$ folds), which turns out to be computationally demanding as $V$ increases.
However the larger number of predictors improves the quality of the conformal p-values in the same way as the full-conformal strategy does.
As for the coverage-probability, with $\alpha\in(0,1)$, \citet{barber2021predictive} proved that the prediction-region $\ccpr{V}$ satisfies
\begin{itemize}
	\item  with $V<n$ (Theorem~4):
\begin{align*}
	\mathbb{P}\croch{Y_{n+1} \in \ccpr{V}} 
	\geq 1 - 2 \alpha - \sqrt{2/n},
\end{align*}
	
	\item $V=n$ (Theorem 1):
\begin{align}
	\label{eq.cov.jack}
	\mathbb{P}\croch{Y_{n+1} \in \ccpr{n}} \geq 1 - 2 \alpha .
\end{align}
	
\end{itemize}
This last case (also) corresponds to the so-called Jackknife+ prediction \citep{barber2021predictive}.
The factor $2$ in front of $\alpha$ reveals the loss of information incurred by
the cross-conformal strategy when compared to the full-conformal one.
By contrast our approximations detailed along Sections~\ref{sec.smooth} and~\ref{sec.very.smooth} do not suffer such a high computational cost
while ensuring coverage at the control-level $\alpha$.

\section{Designing a new computable approximation}
\label{sec.non.smooth}

The main purpose of the present section is to expose a generic scheme leading to a new approximation to the full-conformal prediction-region based on approximate scores. 

More precisely, Section~\ref{sec.predictor} introduces the Ridge-like predictor considered here as well as the main notations.
Section~\ref{sec.approx.prediction.region} explains how approximating the scores at a prescribed level leads to an accurate approximation to the full-conformal prediction-region. The quality of this approximation is quantified by means of both the (usual) coverage-probability and a new notion of thickness.     
As a proof of concept, Section~\ref{sec.uniform.stability.bound} illustrates how the generic scheme introduced along Section~\ref{sec.approx.prediction.region} can be instantiated in the set-up where uniform stability-bounds are available.

\subsection{Predictor and loss-functions}
\label{sec.predictor}

The following definition formulates
a generalization of the kernel ridge regression \citep{vovk2013kernel} which incorporates loss-functions that are no longer limited to the quadratic loss.
The subsequent predictor has the following characteristics:
(1) modeling non-linear relationships between the input and the output via the kernel function \citep{hofmann2008kernel},
(2) controlling the over-fitting via ridge regularization \citep{mcdonald2009ridge},
(3) incorporating a loss-function tailored to the task.
The specified kernel function $\kappa_{\mathcal{H}}(\bullet,\bullet)$ \citep{aronszajn1950theory} induces a natural class of function $\mathcal{H} \subset \mathcal{F}(\mathcal{X}, \mathcal{Y})$ which turns out to be a reproducing kernel Hilbert space (RKHS) endowed with a scalar product $\scal{ \bullet , \bullet }_{\mathcal{H}}$ and corresponding $\mathcal{H}-$norm $\|\bullet\|_{\mathcal{H}}$ \citep{paulsen2016introduction}.




\begin{definition}{(Kernel regression with ridge regularization \citep[Theorem 22]{bousquet2002stability})}
\label{def.predictor}
Let $\mathcal{H} \subset \mathcal{F}(\mathcal{X}, \mathcal{Y})$ be an $RKHS$
with its corresponding norm $\|\bullet\|_{\mathcal{H}}$.
Let $D$ be a training data set. For every $\lambda>0$, the predictor $\pred{\lambda;}{}$
is defined as a regularized empirical-risk minimizer that is, 
\begin{align} \label{def.krr}
    \pred{\lambda;}{}
    \in
    \mathop{\mathrm{argmin}}_{f \in \mathcal{H}}
    \left\{
        \frac{1}{\mathrm{Card}(D)}
        \sum_{(x, y) \in D} 
        \ell\left(y, f(x)\right)
        + \lambda \|f\|_{\mathcal{H}}^2
    \right\},
\end{align}
where $\ell\paren{\bullet, \bullet} : \mathcal{Y} \times \mathcal{Y} \to \mathbb{R}$ can be any loss-function.
\end{definition}
%
Recasting the minimization problem in Eq.~\eqref{def.krr} 
by means of the so-called Representer theorem \citep[in Theorem 1]{scholkopf2001generalized} yields a regularized empirical-risk minimization which coincides with the well-known Tikhonov regularization \citep[]{fuhry2012new}.

It is important to note that $\ell\paren{\bullet, \bullet}$
can be chosen to be different from the quadratic loss-function. 
%
For instance, alternative loss-functions such as the absolute deviation, the Huber or Logcosh loss-functions (for robust regression),
or the pinball loss-function (for quantile regression) can be considered as reliable alternatives depending on the context:
\begin{enumerate}
    \item \emph{Logcosh} loss  \citep{saleh2022statistical}: $\ell\paren{y, u} = a \log\paren{\cosh\paren{\frac{y - u}{a}}}$, with $a \in \paren{0, +\infty}$,
    which leads to 
        Figures~\ref{fig.comp.bound.non.smooth.0},
        \ref{fig.comp.bound.smooth.0} and
        \ref{fig.comp.bound.very.smooth.0},
        
    \item \emph{Pseudo-Huber} loss \citep{charbonnier1994two}:
        $\ell\paren{y, u} = a^2 \paren{\sqrt{1 + \paren{
            \frac{y - u}{a}
        }^2} - 1}$, with $a \in \paren{0, +\infty}$ , 
    which leads to 
    Figures~\ref{fig.comp.bound.non.smooth.2},
    \ref{fig.comp.bound.smooth.2} and
    \ref{fig.comp.bound.very.smooth.2},
    
    \item \emph{Smoothed-pinball} loss \citep{zheng2011gradient}: 
        $\ell\paren{y, u} =  t \paren{y - u}
        + a \log\paren{
            1 + \mathrm{e}^{-\frac{y-u}{a}}
        }$,  with $a \in \paren{0, +\infty}$ and
    $t \in \paren{0, 1}$,
        which leads to 
        Figures~\ref{fig.comp.bound.non.smooth.1},
        \ref{fig.comp.bound.smooth.1} and
        \ref{fig.comp.bound.very.smooth.1}.
        It is a smooth approximation to the pinball loss used in quantile regression \citep{zheng2011gradient}.   
\end{enumerate}
Although the full-conformal prediction-region is explicitly known when $\ell\paren{\bullet, \bullet}$ is quadratic \citep{burnaev2016conformalized}, to the best of our knowledge, this is no longer true in general (and in particular for the above loss-functions).

On the contrary, one main contribution of the the present approach
(described in Section~\ref{sec.approx.prediction.region}) is
to enable the full computation of (an approximation to) the full-conformal prediction-region with general loss-functions.

\bigskip
\noindent{\emph{Notation.}}
The Gram matrix associated with the reproducing kernel $\kappa_{\mathcal{H}}$ is denoted by  $K := (k_{\mathcal{H}}(X_i, X_j))_{1 \leq i, j \leq n+1} \in \mathbb{R}^{(n+1) \times (n+1)}$, and $\mu_*^{(n+1)}>0$ is the smallest non-zero eigenvalue of the normalized Gram matrix $K/(n+1)$.
For every $i \in \brac{1, \ldots, n+1}$, $K_{i, i} = \kappa_{\mathcal{H}}(X_i,X_i)$ denotes the $i$th diagonal term of the Gram matrix, $K_{i, \bullet}$ its $i$th row and $K_{\bullet, i}$ its $i$th column.
Finally, the covariance operator $T : \mathcal{H} \mapsto \mathcal{H}$ is
defined by
\begin{align}
	\label{eq.cov.op}
	T := \mathbb{E}_{X}[K_X \otimes K_X] 
	= \int_{\mathcal{X}} (K_{x} \otimes K_{x}) \mathrm{d}P_X(x),
\end{align}
where, for every $x\in \mathcal{X}$, $K_x := \kappa_{\mathcal{H}}(x,\bullet) \in \mathcal{H}$ denotes the evaluation function such that, for every $h \in \mathcal{H}$,
\begin{align*}
	\scal{ K_x,h}_{\mathcal{H}} & = h(x), \qquad \mbox{(Reproducing property)}\\
	(K_x \otimes K_x) (h) &= h(x) K_x ,\\
	(K_x \otimes K_x) (h, h) & = \scal{ (K_x \otimes K_x) (h),h}_{\mathcal{H}}  = h^2(x),
\end{align*}
where the tensor product is defined by $ (f\otimes g) h = \scal{ g,h}_{\mathcal{H}} f$, for every $f,g,h \in\mathcal{H}$.

\subsection{Approximate full-conformal prediction-region}
\label{sec.approx.prediction.region}

The following definition describes a generic scheme
that formulates confidence prediction-regions in Eq.~\eqref{eq.confidence.prediction.region}, containing the full-conformal one.
The main ingredients are approximate \emph{non-conformity} scores and upper-bounds on the scores approximation quality.
\begin{definition}
\label{def.approx.prediction.region}
    For any control-level $\alpha \in (0, 1)$,
    and every $y \in \mathcal{Y}$, $\widetilde{S}_{D^{y}} \paren{X_i, Y_i}$ is an approximation to the \emph{non-conformity score} $S_{D^{y}} \paren{X_i, Y_i}$ if there exists an upper-bound
    $0 \leq \widehat{\tau}_i(y) < +\infty$ such that
    \begin{align}
    \label{eq.approx.non.conformity.scores}
        \begin{aligned}
            \abss{S_{D^{y}} \paren{X_i, Y_i}  - \widetilde{S}_{D^{y}} \paren{X_i, Y_i}}
            &\leq \widehat{\tau}_i(y),
            &&\mbox{if }1\leq i \leq n
            \\
            \abss{S_{D^{y}} \paren{X_i, y}  - \widetilde{S}_{D^{y}} \paren{X_i, y}}
            &\leq \widehat{\tau}_i(y),
            &&\mbox{if }i=n+1.
        \end{aligned}
    \end{align}
Then, the \emph{approximate full-conformal} p-value function
    $\ufcpv{}{\bullet}$ is given by, for every $y \in \mathcal{Y}$,
    \begin{align}
        \label{def.approx.pvalue}
        \ufcpv{}{y}
        := \frac{
            1 
            + 
            \sum_{i=1}^{n} 
            \mathbbm{1}
            \left\{
                \widetilde{S}_{D^{y}} \paren{X_i, Y_i} + \widehat{\tau}_i(y)
                \geq 
                \widetilde{S}_{D^{y}} \paren{X_{n+1}, y} - \widehat{\tau}_{n+1}(y)
            \right\}
        }{
            n+1
        },
    \end{align}
    leading to the induced approximate full-conformal prediction-region
    \begin{align}
    \label{eq.generic.approx.confidence.region}
        \ufcpr{}
        := 
        \left\{
            y \in \mathcal{Y} : 
            \ufcpv{}{y} > \alpha
        \right\}. 
    \end{align}
\end{definition}

The novelty of the present scheme owes to introducing, for every $y \in \mathcal{Y}$,
both the approximate scores $\widetilde{S}_{D^{y}} \paren{X_i, Y_i}, \ldots, \widetilde{S}_{D^{y}} \paren{X_{n+1}, y}$ and
the related upper-bounds $\widehat{\tau}_1(y), \ldots, \widehat{\tau}_{n+1}(y)$
on their approximation quality.
The present scheme is similar to alternative approaches from
the literature such as the one of \citet{ndiaye2022stable} for
designing an approximation to full-conformal prediction-regions based on uniform stability. 
However one main originality of the present scheme lies in
the dependence of the approximate scores
$\widetilde{S}_{D^{y}} \paren{X_i, Y_i}, \ldots, \widetilde{S}_{D^{y}} \paren{X_{n+1}, y}$ and
upper-bounds $\widehat{\tau}_1(y), \ldots, \widehat{\tau}_{n+1}(y)$ with respect to $y$.
Unlike \citet[see Section 3]{ndiaye2022stable},
this local dependence of $\widehat{\tau}_1(y), \ldots, \widehat{\tau}_{n+1}(y)$ allows for
more flexibility and leads to a more precise approximation to the full-conformal prediction-region (see Eq.~\ref{eq.conformalRegion}).
An important property enjoyed by the approximation described in Definition~\ref{def.approx.prediction.region} is that
it yields the desired coverage for the approximate confidence prediction-region, while it also allows controlling the tightness of this approximation.
\begin{lemma}
\label{lm.coverage.approximate.region}
    For any control-level $\alpha \in (0, 1)$,
    the prediction-region $\ufcpr{}$ from Eq.~\eqref{eq.generic.approx.confidence.region}
    contains the full-conformal prediction-region that is, 
    $
        \fcpr 
        \subseteq
        \ufcpr{}
    $. As a consequence,
    assuming the first set of assumptions in Thereom~\ref{thm.coverage},
    it results that
    \begin{align*}
        \mathbb{P}\croch{
            Y_{n+1} 
            \in 
            \ufcpr{}
        }
        \geq \mathbb{P}\croch{
            Y_{n+1} 
            \in 
            \fcpr
        }
         \geq 1 - \alpha,
    \end{align*}
    making $\ufcpr{}$ a confidence prediction-region.
\end{lemma}
The proof is deferred to Appendix~\ref{proof.coverage.approximate.region}.
The approximate full-conformal prediction-region $\ufcpr{}$ ensures the desired coverage
at the price of containing (and therefore being
larger than) the intractable full-conformal prediction-region.

\medskip

The next definition introduces the notion of ``thickness",
which quantifies the precision of this approximation by comparison with the full-conformal prediction-region. This quantification is given in terms of symmetric difference and Lebesgue measure.
\begin{definition}
\label{def.thickness}
For any control-level $\alpha \in (0, 1)$, 
define the \emph{thickness} 
$\thicc{}$
of the prediction-region 
$\ufcpr{}$
as the volume of its symmetric difference with 
the full-conformal prediction-region $\fcpr$
that is,
\begin{align*}
    \mathrm{THK}_{\alpha}(X_{n+1}) 
    := \leb{
        \ufcpr{} \Delta \fcpr
    } = \leb{
        \ufcpr{} \setminus \fcpr
    },
\end{align*}
where $\mathcal{V}$ is the Lebesgue measure.
\end{definition}
Unfortunately exactly computing the \emph{thickness} would require
computing the intractable full-conformal prediction-region itself.
Notice that the prediction-region $\ufcpr{}$ defined right above can be seen as an upper approximation (for the inclusion)
to the full-conformal prediction-region.
Therefore a corresponding lower-approximation to $\fcpr$ can be formulated by introducing a correction in Eq.~\eqref{def.approx.pvalue} that is,
for any control-level $\alpha \in (0,1)$
\begin{align*}
    y \in \mathcal{Y}
    \mapsto \lfcpv{}{y}
    :=
    \frac{1 +
    \sum_{i=1}^{n}
    \mathbbm{1}
    \left\{
        \widetilde{S}_{D^{y}} \paren{X_i, Y_i}
        - \widehat{\tau}_i(y)
        \geq
        \widetilde{S}_{D^{y}} \paren{X_{n+1}, y}
        + \widehat{\tau}_{n+1}(y) 
    \right\}
    }{n+1},
\end{align*}
which leads to
\begin{align}
    \lfcpr{}
    :=
    \left\{
        y \in \mathcal{Y} : 
        \widetilde{\pi}^{\mathrm{lo}}_D(X_{n+1}, y) > \alpha
    \right\}. \label{eq.lower.approx.region}
\end{align}

Following the same reasoning, it is easily proved that the prediction-region from Eq.~\eqref{eq.lower.approx.region} is
a lower-approximation to the full-conformal prediction-region.
This results in a strategy for upper-bounding the \emph{thickness} as suggested by the next result.

\begin{lemma}[Sandwiching lemma]
\label{lm.sandwiching}
The full-conformal prediction-region is sandwiched by its upper- and lower-approximations as
\begin{align*}
    \lfcpr{} \subseteq \fcpr \subseteq \ufcpr{}.
    \quad \mbox{a.s.}
\end{align*}
It results that the thickness satisfies
\begin{align}
\label{eq.confidence.region.gap.bound}
    \thicc{} \leq
    \mathcal{V}\left(
        \ufcpr{} 
        \setminus
        \lfcpr{}
    \right).
    \quad \mbox{a.s.}
\end{align}
\end{lemma}
The proof is deferred to Appendix \ref{proof.sandwiching}.
Let us mention two main features of the above approximations. 
On the one hand, these upper- and lower-approximations can be explicitly computed (unlike the classical full-conformal prediction-region).
On the other hand, it is also possible to exploit these approximations to derive an explicit upper-bound on the thickness in terms of convergence rate with respect to $n$ and other influential quantities (see for instance Theorem~\ref{thm.non.smooth}).

%
%
%

\subsection{First application to non-smooth loss-functions}
\label{sec.uniform.stability.bound}

Let us now consider a first ``toy example" illustrating how the above strategy can already recover existing bounds in its simplest version, that is, when used in a scenario where almost no smoothness can be exploited. 
Deriving more refined upper-bounds is precisely the goal of Sections~\ref{sec.smooth} and~\ref{sec.very.smooth} where higher smoothness assumptions are made on the loss- and score-functions.

Therefore, let us start by reviewing a few assumptions.
\begin{assumption} For every $y \in \mathcal{Y}$,
    \begin{equation}
        \label{asm.loss.convex}
        \tag{pConvL}
            \mbox{$u \in \mathcal{Y} \mapsto  \ell\paren{y, u}$ is a proper convex function.}    
    \end{equation}
\end{assumption}
Convexity is a classical assumption for the loss-function.
Combined with Ridge regularization, it ensures (strong convexity) the existence and uniqueness of the predictor defined as the minimizer of the Ridge cost function in Eq.~\eqref{def.krr}.

\begin{assumption} 
There exists a constant $\rho \in \paren{0, \infty}$ such that, for every $y \in \mathcal{Y}$,
\begin{equation}
    \label{asm.loss.lipschitz}
    \tag{$\rho$-LipL}
    \mbox{ $u \in \mathcal{Y} \mapsto \ell\paren{y, u}$ is $\rho$-Lipschitz continuous.} 
\end{equation}
\end{assumption}
The $\rho$-Lipschitz property of the loss in its second argument is classically used for establishing the uniform stability property \citep[][Definition~19]{bousquet2002stability}. 

\begin{assumption}
There exists a constant $\gamma \in \paren{0, \infty}$ such that, for every $y \in \mathcal{Y}$,
\begin{equation}
    \label{asm.score.lipschitz}
    \tag{$\gamma$-LipS}
    \mbox{ $u \in \mathcal{Y} \mapsto s\paren{y, u}$ is $\gamma$-Lipschitz continuous.}      
\end{equation}
\end{assumption}
Note that the above assumptions hold for the three loss-functions presented at the end of Section~\ref{sec.predictor} (see Appendix~\ref{sec.example.loss.function}).
Moreover, a loss-function verifying both \eqref{asm.loss.lipschitz} and \eqref{asm.loss.convex} turns out to be $\rho$-admissible in the sense described by \citet[in Definition 19]{bousquet2002stability}. As a consequence, uniform stability-bounds do already exist when Learning in an RKHS \citep[in Theorem 22]{bousquet2002stability}.

\bigskip

\noindent{\emph{Designing an approximation to the predictor.}}
The main reason for the high computational cost of the full-conformal prediction-region is the expression of the predictor
$\pred{\lambda;}{y}$ which depends on $y$ and should be (in general) recomputed each time the value of $y$ changes.
To reduce this computational burden, a natural idea is then to approximate the value of $\pred{\lambda;}{y}$ by
$\pred{\lambda;}{z}$, where $z \in\mathcal{Y}$ denotes a fixed point in a neighborhood of $y$.

In what follows, let us define the approximation $\tilde{f}_{\lambda; D^y} := \pred{\lambda;}{z}$ to $\pred{\lambda;}{y}$,
for any fixed $z \in \mathcal{Y}$.
For such a given $z \in\mathcal{Y}$, the non-conformity scores approximation are then obtained from $\tilde{f}_{\lambda; D^y}$ as
\begin{align}
\label{eq.approximate.score.0}
\begin{aligned}
    S_{\lambda; D^{y}}\paren{X_i, Y_i}
    \approx \widetilde{S}_{\lambda; D^{y}}\paren{X_i, Y_i} 
    &:= s\paren{Y_i, \tilde{f}_{\lambda; D^y}(X_i)},
    &&\mbox{if } 1 \leq i \leq n\\
    S_{\lambda; D^{y}}\paren{X_i, y}
    \approx
    \widetilde{S}_{\lambda; D^{y}}\paren{X_i, y} 
    &:= s\paren{y, \tilde{f}_{\lambda; D^y}(X_i)},
    &&\mbox{if } i=n+1.
\end{aligned}
\end{align}

The next result yields the corresponding upper-bounds on the scores approximation by applying techniques developed by \citet{bousquet2002stability}
for deriving uniform stability-bounds.

\begin{theorem}
\label{thm.score.bound.non.smooth} 
    With the same notations as in Eq.~\eqref{def.krr}, let us assume \eqref{asm.loss.convex}, \eqref{asm.loss.lipschitz}, and \eqref{asm.score.lipschitz} hold true.
    Then, for every $y\in \mathcal{Y}$, it comes that
    \begin{align*}
    \begin{aligned}
        &&\abss{
            S_{\lambda; D^{y}}\paren{X_i, Y_i}
            - \widetilde{S}_{\lambda; D^{y}}\paren{X_i, Y_i}
        }
        \leq \widehat{\tau}_{\lambda; i}^{(0)},\qquad        
        \mbox{if }1 \leq i \leq n\\
 \mbox{and }\quad       &&\abss{
            S_{\lambda; D^{y}}\paren{X_i, y}
            - \widetilde{S}_{\lambda; D^{y}}\paren{X_i, y}
        }
        \leq \widehat{\tau}_{\lambda; i}^{(0)},\qquad        
        \mbox{if }i=n+1,
    \end{aligned}
    \end{align*}
    where, for every $i \in \brac{1, \ldots, n+1}$,
    \begin{align*}
        \widehat{\tau}_{\lambda; i}^{(0)}
        := \sqrt{K_{i, i}} \sqrt{K_{n+1, n+1}}
        \frac{\gamma \rho}{\lambda (n+1)}.
    \end{align*}
\end{theorem}
The proof is postponed to Appendix~\ref{proof.score.bound.non.smooth}.
The approximation improves at the rate $O(n^{-1})$ as $n$ is increasing. 
The higher the Lipschitz-smoothness of the score-function (respectively of the loss), the lower the $\rho$ (resp. $\gamma$) constant, and the better the upper-bound. 
The choice of the kernel $\kappa_{\mathcal{H}}$ also impacts the convergence.
For instance, a bounded kernel or a translation-invariant one (such as the Gaussian RBF kernel) would imply that
there exists a constant $\kappa^2>0$ such that $K_{j, j} \leq \kappa^2$,
for every $1\leq j\leq n+1$.
In the present context, the upper-bound can be tightened by keeping track of the variation of $K_{i, i}$ with the index $i$ (that is, when $X_i$s vary within their domain).

\bigskip

\noindent{\emph{Performance quantification.}}
Let us $\ufcprr{,(0)}{\lambda; \alpha}$ denote the upper-approximate full-conformal prediction-region defined as
\begin{align*}
    \ufcprr{, (0)}{\lambda; \alpha}
    := \brac{
        y \in \mathcal{Y} :
        \ufcpvv{,(0)}{\lambda; D}{y}
        > \alpha
    },
\end{align*}
where for every $y \in \mathcal{Y}$,
\begin{align*}
    \ufcpvv{,(0)}{\lambda; D}{y}
    :=
    \frac{1 +
    \sum_{i=1}^{n}
    \mathbbm{1}
    \left\{
        \widetilde{S}_{\lambda; D^{y}} \paren{X_i, Y_i}
        - \widehat{\tau}_{\lambda; i}^{(0)}
        \geq
        \widetilde{S}_{\lambda; D^{y}} \paren{X_{n+1}, y}
        + \widehat{\tau}_{\lambda; n+1}^{(0)} 
    \right\}
    }{n+1},
\end{align*}
resulting from the approximation scheme
in Definition~\ref{def.approx.prediction.region}
(where the non-conformity scores approximations are given in Eq.~\ref{eq.approximate.score.0}
and the upper-bounds come from Theorem~\ref{thm.score.bound.non.smooth}).
Let us note $\lfcprr{,(0)}{\lambda; \alpha}$,
the corresponding lower-approximate full-conformal prediction-region
(analog to the prediction-region defined in Eq.~\ref{eq.lower.approx.region}).
Finally recalling the thickness $\thick{(0)}{\lambda; \alpha}$ of the prediction-region $\ufcprr{,(0)}{\lambda; \alpha}$ (see Definition~\ref{def.thickness}) given by 
\begin{align*}
    \thick{(0)}{\lambda; \alpha} := \leb{
        \ufcprr{,(0)}{\lambda; \alpha} \setminus \fcprr
    },
\end{align*}
the following theorem provides an explicit upper-bound on $\thick{(0)}{\lambda; \alpha}$.

%
\begin{theorem}
\label{thm.non.smooth}
    Assume \eqref{asm.loss.convex}, \eqref{asm.loss.lipschitz}, and \eqref{asm.score.lipschitz}
    along with the first set of assumptions in Thereom~\ref{thm.coverage} hold true.
    Then, for any control-level $\alpha \in (0, 1)$,
    the prediction-region $\ufcprr{,(0)}{\lambda; \alpha}$
    is a confidence prediction-region, that is,
    \begin{align*}
    	\mathbb{P}
    	\croch{Y_{n+1} \in \ufcprr{,(0)}{\lambda; \alpha}} \geq 1 - \alpha.
    	\end{align*}

    Furthermore, if the non-conformity function is
    $s(\bullet, \bullet) : \mathcal{Y} \times \mathcal{Y} \to \mathbb{R},$ 
    $(y, u) \to s\paren{y, u} = \abss{y - u}$, then its thickness is upper-bounded by
    \begin{align}
    \label{eq.bound.thickness.non.smooth}
        \thick{(0)}{\lambda; \alpha}
        \leq
        \frac{8 \rho}{\lambda (n+1)}
        \max_{i \in \brac{1, \ldots, n+1}}
        \kappa_{\mathcal{H}} \paren{X_i, X_i}.
        \quad \mbox{a.s.}
    \end{align}
\end{theorem}
The proof is deferred to Appendix~\ref{proof.non.smooth}.
This result serves as a sanity-check by showing that our strategy allows to recover the approximate full-conformal prediction-region earlier developed by
\citet{ndiaye2022stable} and called stable conformal.
This approximation tightens as $n$ grows at a rate which depends on the reproducing kernel as well as distributional assumptions on the $X_i$s.
For instance, if $\lambda$ is fixed and if the kernel is bounded then the rate would be $O\paren{n^{-1}}$. If the kernel is not bounded, then assuming $\kappa_{\mathcal{H}}(X_i,\bullet) \in\mathcal{H}$ is a sub-Gaussian random variable would lead to a rate $O(\log(n)/n$) due to the maximum over $n$ sub-Gaussian variables.

\begin{corollary}
    \label{cor.non.smooth}
	With the same notations and assumptions as Theorem~\ref{thm.non.smooth},
        let us further assume that there exists a constant $c \in \paren{0, +\infty}$
        such that the conditional density $p(\bullet|D, X_{n+1}) : \mathcal{Y} \to \mathbb{R}_+$, is bounded from above by $c$, and
        that the non-conformity scores $S_{D^{Y_{n+1}}} \paren{X_1, Y_1}$, $\ldots$, $S_{D^{Y_{n+1}}} \paren{X_{n+1}, Y_{n+1}}$ are almost-surely distinct. 
Then the coverage-probability of the approximate full-conformal prediction-region
$\ufcprr{,(0)}{\lambda; \alpha}$ satisfies that
{\small \begin{align*}
	0
	\leq
	\mathbb{P}
	\croch{Y_{n+1} \in \ufcprr{,(0)}{\lambda; \alpha}} - (1 - \alpha)
	\leq \frac{1}{n+1} 
	+ \frac{8 c \rho}{\lambda (n+1)}
	\mathbb{E}_{D, X_{n+1}}\croch{
		\max_{i \in \brac{1, \ldots, n+1}}
		\kappa_{\mathcal{H}} \paren{X_i, X_i}
	}.
\end{align*}}
\end{corollary}
The proof is postponed to Appendix~\ref{proof.cor.non.smooth}.
When the conditional density $p(\bullet|D, X_{n+1})$ is bounded by $c$, \citet{barber2021predictive} used a related stability-bound~$\epsilon$
\citep{bousquet2002stability} to bound the thickness of
the Jackknife+ prediction-region by $4\epsilon$, which results in a more conservative lower bound on the coverage that is, $1 - \alpha - 4c\epsilon$ (compared to the previous $1 - 2\alpha$ in Eq.~\ref{eq.cov.jack}).
By contrast, our full-conformal prediction-region approximation directly
ensures the desired control of $1 - \alpha$.

\bigskip

\noindent{\emph{Choice of non-conformity function.}}
It is actually possible to relax the requirement on the function $s$. To be more specific, let $s$ be a non-conformity function
$s: \paren{y, u} \in \mathcal{Y} \times \mathcal{Y} \mapsto s(y, u) = \widetilde{S} \paren{y - u}$, where $\widetilde{S}$ is a non-negative, even, and increasing function over $\left[0, \infty\right)$.
Then the full-conformal prediction-region resulting from
this non-conformity function is the same as the one output by
$s(y, u) = \abss{y - u}$.
Besides, the abolute value has been already considered by
\citet{burnaev2016conformalized},
\citet{barber2021predictive}, \citet{ndiaye2022stable}
to enumerate but a few.

\bigskip

\noindent{\emph{Illustration.}}
%
The aim of next Figure~\ref{fig.comp.bound.non.smooth.0} is to
illustrate the tightness of the upper-bound
on the thickness (see Eq.~\eqref{eq.bound.thickness.non.smooth}).
Since computing $\thick{(0)}{\lambda; \alpha}$ is intractable,
Figure~\ref{fig.comp.bound.non.smooth.0} only displays
the gap between $\Delta^{(0)} := \leb{\ufcprr{,(0)}{\lambda; \alpha}
	\setminus \lfcprr{,(0)}{\lambda; \alpha}}$ (from Eq.~\eqref{eq.confidence.region.gap.bound}), and its corresponding upper-bound from Eq.~\eqref{eq.bound.thickness.non.smooth}.
The code is available at \url{https://github.com/Davidson-Lova/approximate_full_conformal_prediction_RKHS.git}.



\begin{figure}[H]
    \centering
    \includegraphics[width=0.60\textwidth]{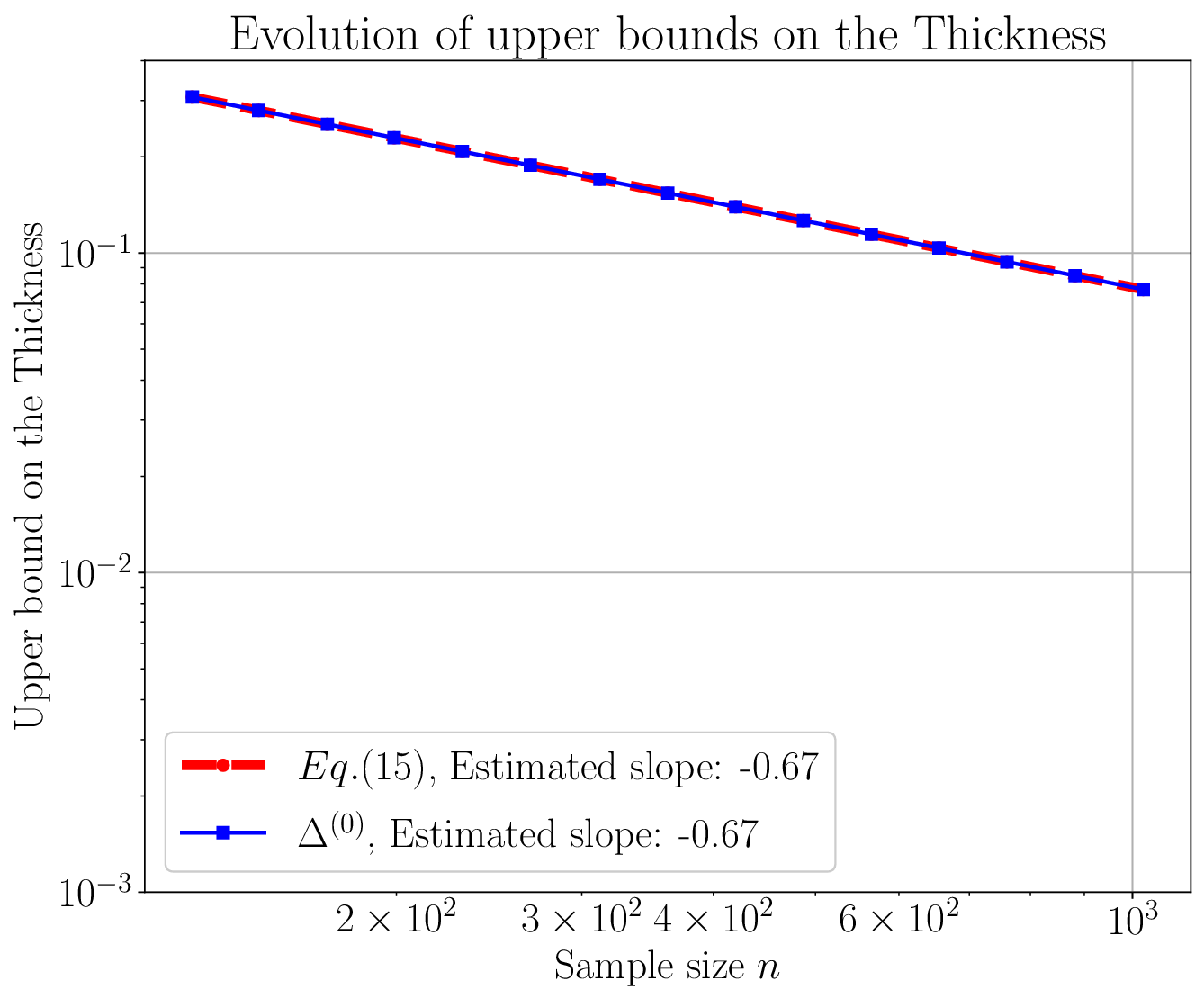}
    \caption{
        Evolution of
        the upper-bound in Eq.~\eqref{eq.bound.thickness.non.smooth} (dashed red-line)
        and the quantity $\Delta^{(0)}$ (solid blue-line)
        as a function of the sample size $n$
        in $\log\log$ scale (to appreciate the rate).
        The data is sampled from $\mathrm{sklearn}$ synthetic data set
        make\_friedman1(sample\_size=n).
        The kernel $k_\mathcal{H} \paren{\bullet, \bullet}$
        is set to be the Laplacian kernel (gamma=None).
        The loss-function $\ell\paren{\bullet, \bullet}$
        is set to be the Logcosh Loss ($a=1.0$).
        The regularization parameter is set to decay as $\lambda \propto (n+1)^{-0.33}$.
        The fixed output value is set at $z=0$.        
        The non-conformity function is set to be $s(\bullet, \bullet) : (y, u) \mapsto \abss{y - u}$.
    }
    \label{fig.comp.bound.non.smooth.0}
\end{figure}

The two  curves seem to coincide (at least on the grid of 15 values of $n$ between 128 and 1024). This suggests the upper-bound in Eq.~\eqref{eq.bound.thickness.non.smooth}
is a  tight approximation to the quantity $\Delta^{(0)}$ for upper-bounding the intractable thickness.
Moreover, the estimated slope of both lines is close to the expected value $-1 + r = -0.67$, when the regularization parameter $\lambda \propto (n+1)^{-r}$.
The upper-bound in Eq.~\eqref{eq.bound.thickness.non.smooth}
reflects the same dependence on $r$ as the one of $\Delta^{(0)}$.


\section{Designing an approximation from a smooth loss-function}
\label{sec.smooth}
The present section aims at describing how additional smoothness properties of the loss-function can be harnessed to exhibit an improved confidence-region.

More precisely, additional assumptions exploiting smoothness are considered and new upper-bounds are derived. The quality of the ensuing approximate prediction-region is also compared with the one from Section~\ref{sec.uniform.stability.bound}.
%
%
%


\subsection{\textit{Local} stability-bounds}
\label{sec.local.bound.score}

%
%
%

As in the previous section, the non-conformity scores are built from a predictor $\apred{\lambda;}{y}{} = \pred{\lambda;}{z}$, for some fixed $z \in \mathcal{Y}$, defined as a solution of an optimization problem (see Eq.~\ref{def.krr}).
The purpose of the next definition is to introduce a quantity allowing to capture the dependence of the predictor with respect to each training data.

\begin{definition}{(Predictor as a function of weights)}
\label{def.weighted.predictor}
Let $z \in \mathcal{Y}$ be a fixed prediction value,
$y \in \mathcal{Y}$ be a candidate prediction value.

Let $v=(v_1, \ldots, v_n,v_{n+1},v_{n+2}) \in \mathbb{R}^{n+2}$ be
vector of $n+2$-weights applying to 
the $n$ points in $D$, $\paren{X_{n+1}, z}$, and $\paren{X_{n+1},y}$.
For a regularization parameter $\lambda >0$, let $\wpred{\lambda}{v}{\bullet} : \mathcal{X} \to \mathcal{Y}$ denote a predictor given by
\begin{align*}
    \wpred{\lambda}{v}{\bullet}
    \in
    & 
    \argmin_{f \in \mathcal{H}} 
    \wer{\lambda}{v}{f},
\end{align*}
where $\wer{\lambda}{v}{f}$ is the $v$-weighted ($\lambda$-regularized) empirical-risk for $\paren{v, f} \in \mathbb{R}^{n+2} \times \mathcal{H}$, given by
\begin{align}\label{eq.Emp.Risk.weighted.rgularized}
    \wer{\lambda}{v}{f}
    &
    := 
    \sum_{i=1}^{n}
    \brac{\frac{v_i}{n+1} \ell\paren{Y_i, f(X_i)}}
    +  \frac{v_{n+1}}{n+1}
    \ell\paren{z, f(X_{n+1})}
    + \lambda \normh{f}^2 \notag
    \\
    &
    \qquad
    + \frac{v_{n+2}}{n+1}
    \ell\paren{y, f(X_{n+1})}.
\end{align}

\end{definition}
Under \eqref{asm.loss.convex},
Lemma~\ref{lm.predictor.func.well.defined.1} establishes that,
for $\uu := (1, \ldots, 1, 1, 0)$ (respectively for $\uw := (1, \ldots, 1, 0, 1)$),
the minimizer $\wpred{\lambda}{\uu}{\bullet}$ (resp.  $\wpred{\lambda}{\uw}{\bullet}$) does exist and is unique.
Furthermore since it also holds that $\wpred{\lambda}{\uu}{\bullet} = \pred{\lambda;}{z}$ and $\wpred{\lambda}{\uw}{\bullet} = \pred{\lambda;}{y}$, it gives rise to a strategy for providing an upper-bound for
$\normh{\pred{\lambda;}{z} - \pred{\lambda;}{y}}$
by simply considering $\normh{\wpred{\lambda}{\uu}{\bullet} - \wpred{\lambda}{\uw}{\bullet}}$. 
This idea is for instance at the core of the derivation of Proposition~\ref{prop.predictor.stability.smooth} and Theorem~\ref{thm.score.bound.smooth}.
%
%
%




Another important ingredient in our derivation is the stronger smoothness assumption on the loss-function (see Section~\ref{sec.predictor}), which is required to be once differentiable with respect to its second argument. This is formally specified by the next assumption.
\begin{assumption}
	For every $y \in \mathcal{Y}$,
	\begin{equation}
		\label{asm.loss.c1}
		\tag{C1L}
		\mbox{
			$u \mapsto \ell(y, u)$ is once continuously differentiable,
		}    
	\end{equation}
	with $u \mapsto \partial_2 \ell(y, u)$
	its first derivative.
\end{assumption}
This requirement is fulfilled by all the loss-functions listed along Section~\ref{sec.predictor}.
For every $v \in \mathbb{R}^{n+1}$, the $v$-weighted $\lambda$-regularized empirical-risk $\wer{\lambda}{v}{\bullet}$ (see Eq.~\ref{eq.Emp.Risk.weighted.rgularized}) is
once differentiable with respect to its second argument, and $\Dtwo \wer{\lambda}{v}{\bullet}$ is well-defined.
This quantity is key to derive an upper-bound on the predictor approximation quality (which vary with $y$)
by using the characterization of the minimizer of $\wer{\lambda}{v}{\bullet}$ in terms
of its first differential.


\medskip

The next proposition establishes an important stability-bound in terms of the $\mathcal{H}$-norm of the predictor.
\begin{proposition}[Stability-bound on the approximate predictor]
\label{prop.predictor.stability.smooth}
	Assume \eqref{asm.loss.convex} and \eqref{asm.loss.c1} hold true. Then, for every $y,z \in\mathcal{Y}$, and every $\lambda>0$, the predictor given by Eq.~\eqref{def.krr} satisfies
	\begin{align*}
		\left\|
		\pred{\lambda;}{y} - \pred{\lambda;}{z}
		\right\|_{\mathcal{H}}	& \leq \sqrt{K_{n+1, n+1}} \frac{\rho^{(1)}_{\lambda}(y)}{
			\lambda (n+1) }, 
	\end{align*}
	where
	\begin{align}
		\label{eq.predictor.approx.bound.smooth}
		\rho^{(1)}_{\lambda}(y) := \frac{1}{2}
		\left|
		- \partial_2 \ell\paren{z, \pred{\lambda;}{z}(X_{n+1})}
		+ \partial_2 \ell\paren{y, \pred{\lambda;}{z}(X_{n+1})}
		\right|.
	\end{align}
\end{proposition}
The proof is provided in Appendix \ref{proof.score.bound.smooth}.
    This upper-bound holds almost-surely.
    Let us note that the dependence with respect to $\lambda n$ is not new (see for instance \citet[in Theorem 22]{bousquet2002stability}).
    However, the present upper-bound is more general
    in that: (1) it highlights the importance of
    the regularity of the first derivative of the loss-function,
    and (2) it reflects the influence of the value
    of kernel function at the input point $X_{n+1}$.
    In fact in the bound derived by \citet[Theorem 22]{bousquet2002stability},
    the term $\rho_{\lambda}^{\paren{1}}\paren{y}$ is
    uniformly bounded from above by a Lipschitz constant $\rho$, for every output value $y \in \mathcal{Y}$, while
    the term $\sqrt{K_{n+1, n+1}}$ is similarly bounded by an upper-bound $\kappa^2 >0$ on the (bounded) kernel function $\kappa_{\mathcal{H}}\paren{\bullet, \bullet}$.

\medskip

By inserting the previous upper-bound on the stability of the predictor inside the score-function, the following result describes the new upper-bounds on the non-conformity scores approximation quality.
\begin{theorem}\label{thm.score.bound.smooth}
    Assume \eqref{asm.loss.convex}, \eqref{asm.loss.c1}, and
    \eqref{asm.score.lipschitz} hold true.
    Then, for every $y \in \mathcal{Y}$,
    \begin{align*}
        \begin{aligned}
            \left|
                S_{\lambda; D^{y}} \paren{X_i, Y_i}
                - \widetilde{S}_{\lambda; D^{y}} \paren{X_{i}, Y_{i}}
            \right| 
            \leq \widehat{\tau}_{\lambda; i}^{(1)} \paren{y},
            &&
            \mbox{if }1\leq i \leq n
            \\
            \left|
                S_{\lambda; D^{y}} \paren{X_i, y}
                - \widetilde{S}_{\lambda; D^{y}} \paren{X_{i}, y}
            \right| 
            \leq \widehat{\tau}_{\lambda; i}^{(1)} \paren{y},
            &&
            \mbox{if }i=n+1,
        \end{aligned}
    \end{align*}
    where
    \begin{align*}
        \forall i \in \brac{1, \ldots, n+1},\qquad & \widehat{\tau}_{\lambda; i}^{(1)} \paren{y}
        :=  \sqrt{K_{i, i}}  \sqrt{K_{n+1, n+1 }} \frac{ \gamma \rho^{(1)}_{\lambda}(y)}{\lambda(n+1)} ,         \\
        \mbox{with  }\qquad
    	&    \rho^{(1)}_{\lambda}(y) := \frac{1}{2}
    	\left|
    	- \partial_2 \ell\paren{z, \pred{\lambda;}{z}(X_{n+1})}
    	+ \partial_2 \ell\paren{y, \pred{\lambda;}{z}(X_{n+1})}
    	\right|.
    \end{align*}
\end{theorem}
The proof is deferred to Appendix~\ref{proof.score.bound.smooth}.
By contrast with Theorem~\ref{thm.score.bound.non.smooth}, a localy varying function $y \mapsto \rho_{\lambda}^{(1)}(y)$
replaces the uniform Lipschitz constant $\rho$ in the former bound. This avoids paying the worst possible bound that is, the uniform $\rho$.

\subsection{Improved thickness}
Consistently with the strategy exposed along Section~\ref{sec.approx.prediction.region}, let $\ufcprr{,(1)}{\lambda; \alpha}$ denote the upper-approximate full-conformal prediction-region defined by
\begin{align*}
    \ufcprr{,(1)}{\lambda; \alpha}
    := \brac{
        y \in \mathcal{Y} :
        \ufcpvv{,(1)}{\lambda; \alpha}{y}
        > \alpha
    },
\end{align*}
where, for every $y \in \mathcal{Y}$, the upper approximation to the conformal $p$-value is
\begin{align*}
    \ufcpvv{,(1)}{\lambda; \alpha}{y}
    := \frac{
        1 
        + 
        \sum_{i=1}^{n} 
        \mathbbm{1}
        \left\{
            \widetilde{S}_{\lambda; D^{y}} \paren{X_i, Y_i}
            +
            \widehat{\tau}_{\lambda; i}^{(1)} \paren{y}
            \geq 
            \widetilde{S}_{\lambda; D^{y}} \paren{X_{n+1}, y}
            -
            \widehat{\tau}_{\lambda; n+1}^{(1)} \paren{y}
        \right\}
    }{
        n+1
    }.
\end{align*}
The corresponding lower-approximate full-conformal prediction-region (analogous to the one of Eq.~\ref{eq.lower.approx.region}) is denoted by $\lfcprr{,(1)}{\lambda; \alpha}$.
The thickness of the prediction-region $\ufcprr{,(1)}{\lambda; \alpha}$ is then given by
\begin{align}\label{eq.thikness.First.order}
	\thick{(1)}{\lambda; \alpha} := \leb{
		\ufcprr{,(1)}{\lambda; \alpha} \setminus \fcprr
	},
\end{align}
as prescribed in Definition~\ref{def.thickness}.
All of this is an instance of the generic approximation scheme
exposed along Section~\ref{sec.approx.prediction.region} (see Definition~\ref{def.approx.prediction.region} for the approximate predictive confidence-region, Eq.~\ref{eq.approximate.score.0} for the non-conformity scores approximation, and the upper-bounds proved in Theorem~\ref{thm.score.bound.smooth}).
The next result establishes that allowing for a local control of the score approximation quality leads to an improved (tighter) approximation.
\begin{corollary}
\label{cor.smooth}
Under the same assumptions as Theorem~\ref{thm.score.bound.smooth}, for any control-level $\alpha \in (0, 1)$, the prediction-region $\ufcprr{,(1)}{\lambda; \alpha}$ is a confidence prediction-region.
Furthermore, $\ufcprr{,(0)}{\lambda; \alpha}$ is outperformed by  $\ufcprr{,(1)}{\lambda; \alpha}$ that is,
    \begin{align*}
        \thick{(1)}{\lambda; \alpha} &\leq \thick{(0)}{\lambda; \alpha}.
    \end{align*}
\end{corollary}
The proof is postponed to Appendix \ref{proof.smooth}.
The new approximate region $\ufcprr{,(1)}{\lambda; \alpha}$ is a closer approximation to the full-conformal prediction-region $\fcprr$ compared to $\ufcprr{,(0)}{\lambda; \alpha}$. This results from keeping track of the dependence of the approximation with respect to $y$.
Unfortunately this improvement only translates in terms of better constants in the worst case, which does not change the convergence rate compared to the non-smooth setup as illustrated by Figure~\ref{fig.comp.bound.smooth.0} (compared to Figure~\ref{fig.comp.bound.non.smooth.0}).

\bigskip

\noindent{\emph{Illustration.}}
Figure~\ref{fig.comp.bound.smooth.0} displays the gap between
 $\Delta^{(1)} := \leb{\ufcprr{,(1)}{\lambda; \alpha}
	\setminus \lfcprr{,(1)}{\lambda; \alpha}}$,
and the corresponding upper-bound on the thiskness $\thick{(0)}{\lambda; \alpha}$ given by Eq.~\eqref{eq.bound.thickness.non.smooth} for the sake of comparison.
Therefore the dashed red-line in Figure~\ref{fig.comp.bound.smooth.0} is the same as the one form Figure~\ref{fig.comp.bound.non.smooth.0}.

%
%
%

\begin{figure}[H]
    \centering
    \includegraphics[width=0.60\textwidth]{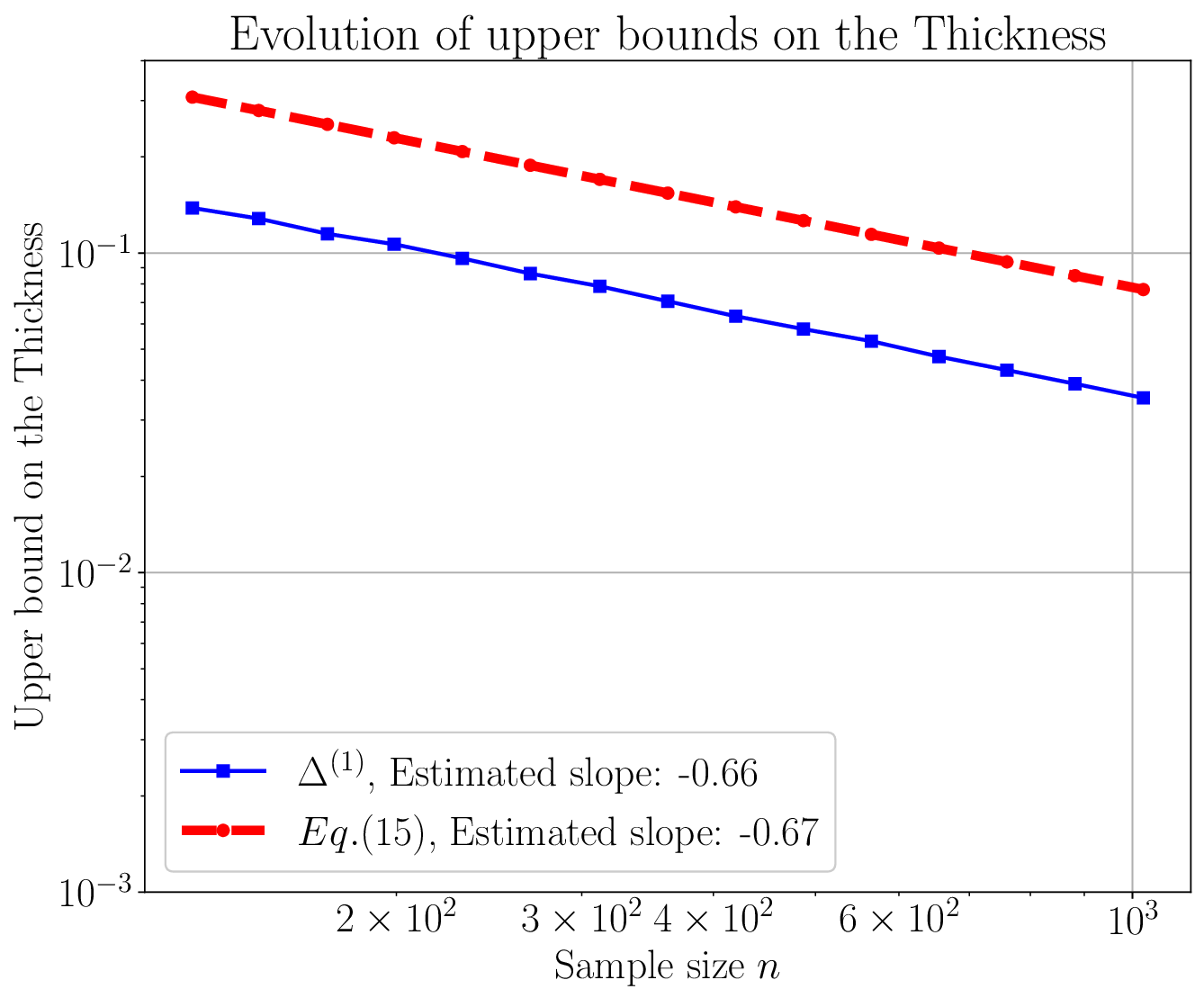}
    \caption{
        Evolution of
        the upper-bound in Eq.~\eqref{eq.bound.thickness.non.smooth} (dashed red-line)
        and the quantity $\Delta^{(1)}$ (solid blue-line)
        as a function of the sample size $n$
        in $\log\log$ scale (to appreciate the rate).
        The data is sampled from $\mathrm{sklearn}$'s synthetic data set
        make\_friedman1(sample\_size=n).
        The kernel $k_\mathcal{H} \paren{\bullet, \bullet}$
        is set to be the Laplacian kernel (gamma=None).
        The loss-function $\ell\paren{\bullet, \bullet}$
        is set to be the Logcosh Loss ($a=1.0$).
        The regularization parameter is set to decay as $\lambda \propto (n+1)^{-0.33}$.
        The fixed output value is set at $z=0$.        
        The non-conformity function is set to be $s(\bullet, \bullet) : (y, u) \mapsto \abss{y - u}$.
    }
    \label{fig.comp.bound.smooth.0}
\end{figure}
The plain blue-line is noticeably below the dashed red-line,
which strongly contrasts with Figure~\ref{fig.comp.bound.non.smooth.0}.
This indicates the \emph{local} approximation improves upon the uniform one in terms of the size of the upper-approximate predictive confidence-region.
As already mentioned above, the slopes of both lines are almost the same, which suggests that the upper-bound from Eq.~\eqref{eq.bound.thickness.non.smooth} and  $\Delta^{(1)}$ share the same rate of convergence with respect to $n$.
All of this supports the conclusion that, in the worst case, the local approximation only improves the size of the predictive confidence-region in terms of constants, but not in terms of convergence rate with respect to $n$.

\section{Designing an approximation from a very smooth loss-function}
\label{sec.very.smooth}
A refined approximation is described here, which involves influence functions as a means to exploit stronger available smoothness assumptions. 
By contrast with Section~\ref{sec.smooth}, this new approximation leads to faster convergence rates.

The influence function (IF) is introduced along Section~\ref{sec.ifunc} as a means to tighten the predictor approximation. Additional assumptions are also reviewed. They are needed to derive a closed-form expression of this IF-based predictor approximation.
Section~\ref{sec.quality.ifunc} focuses on quantifying the quality of the IF-based approximations to the predictor and non-conformity score-functions, while Section~\ref{sec.quality.ifunc.region} provides new upper-bounds on the thickness, and on the probability of coverage of the new approximate full-conformal prediction-region.

\subsection{Approximating predictors via influence function}
\label{sec.ifunc}

Let us start by describing the new approximation to the predictor relying on the \emph{influence function}. 
Approximating predictors via influence function is not
new in the context of conformal prediction.
\citet{martinez2023approximating} harnessed such an approximation
to derive approximate full-conformal prediction-regions
in the context of multiclass classification, and more recently,
\citet{tailorApproximatingFullConformal2025} exploit the same idea in regression.
However, unlike our approximate prediction-region from Definition~\ref{def.IF-based.approximate.region},
theirs are not guaranteed to be \emph{confidence prediction-regions} that is, they do not come with coverage guarantees.


%
%
%
%

\begin{definition}[Bouligand influence function \citep{christmann2008bouligand}]
	Let $\mathcal{M}$ denote the space of probability distribution
	on $\mathcal{X}\times \mathcal{Y}$.
	For every $\lambda \in \paren{0, +\infty}$,
	let $T_{\lambda} : \mathcal{M} \to \mathcal{H}$ denote
	any measurable function such that, for every distribution $P \in \mathcal{M}$,
	\begin{align}
		\label{eq.estimator}
		T_{\lambda}\paren{P}
		\in \argmin_{
			f \in \mathcal{H}
		}
		\mathbb{E}_{\paren{X, Y} \sim P}
		\croch{
			\ell\paren{Y, f\paren{X}}
		}
		+ \lambda \normh{f}^2.
	\end{align}
	Then the Bouligand influence function of $T_{\lambda} : \mathcal{M} \to \mathcal{H}$ for the distribution $P$ in the direction of the distribution $Q \neq P$ is denoted by $\mathrm{BIF}\paren{Q; T_{\lambda}, P} \in \mathcal{H}$ (if it exists) and defined as
	\begin{align*}
		\mathrm{BIF}\paren{Q; T_{\lambda}, P}
		:=
		\lim_{t \rightarrow 0^{+}}
		\frac{
			\croch{
				T_{\lambda}\paren{
					\paren{1 - t}P
					+ tQ
				} - T_{\lambda}\paren{P}
			}
		}{t}.
	\end{align*}
\end{definition}
It appears that BIF (Bouligand influence function) quantifies how much $T_\lambda(\bullet)$ departs from $T_\lambda(P)$ when it is perturbed by the mixture $\paren{1 - t}P
+ tQ$, as $t$ goes to 0.  
For instance, a constant function $T_\lambda(\bullet)$ from Eq.~\eqref{eq.estimator} would output 0, whereas a large $\norm{\mathrm{BIF}\paren{Q; T_{\lambda}, P}}_{\mathcal{H}}$ is expected for a locally quickly varying (but still Fr{\'e}chet differentiable) $T_\lambda(\bullet)$.

\medskip

The next assumptions on $\ell\paren{\bullet, \bullet}$
are fulfilled by the loss-functions listed at the end of Section~\ref{sec.predictor}. (See also Appendix~\ref{sec.example.loss.function}.)
\begin{assumption}
	For every $i \in \brac{1, \ldots, n+1}$,
	\begin{equation}\label{asm.finite.kern}
		\kappa_{\mathcal{H}}\paren{X_{i}, X_{i}} < \infty.
		\quad \mbox{ a.s.}
		\tag{FinK}
	\end{equation}
\end{assumption}
Assumption \eqref{asm.finite.kern} only requires the kernel to be finite at any point $x$ in the support of the probability distribution of the $X_i$s. In particular this does not require that there exists a constant uniformly  upper-bounding the kernel on the support.
{
\begin{assumption}
	For every $y \in \mathcal{Y}$,
	\begin{equation}
		\label{asm.loss.c2}
		\tag{C2L}
		\mbox{
			$u \mapsto \ell(y, u)$ is twice continuously differentiable,
		}    
	\end{equation}
	with $u \mapsto \partial_2^2 \ell(y, u)$
	its second derivative.
\end{assumption}
This requirement is fulfilled by all the loss-functions listed along Section~\ref{sec.predictor}.
Moreover, it is key in deriving an expression for the influence function since
its expression involves the second-order differential of the regularized empirical-risk (see Lemma~\ref{lm.influence.function.expression}).
}
\begin{assumption}
	There exists a constant $\beta_{\ell; 2} \in \paren{0, \infty}$ such that,
	for every $y \in \mathcal{Y}$,
	\begin{equation}
		\label{asm.loss.smooth}
		\tag{$\beta_{\ell; 2}$-LipDL.2}
		\mbox{
			$u\mapsto \partial_2 \ell(y,u)$ is $\beta_{\ell; 2}$-Lipschitz continuous.
		}
	\end{equation}
\end{assumption}
In other words, the loss $\ell(\bullet, \bullet)$ has to be $\beta_{\ell; 2}$-smooth with respect to its second argument.
This assumption is widespread in the machine learning literature.
For instance \citet{bubeckConvexOptimizationAlgorithms2015a}
and later \citet{garrigosHandbookConvergenceTheorems2024} used
this assumption to provide convergence guarantees for optimization algorithms
such as stochastic gradient descent.

\medskip

From the above assumptions, the following lemma establishes the existence of BIF in the our framework. This requires considering empirical distributions $\widehat{P}_{n+1}^{z^{\prime}}$ (for every $z^{\prime} \in \mathcal{Y}$) defined as
	\begin{align}
		\label{eq.empirical.distribution}
		\widehat{P}_{n+1}^{z^{\prime}} :=
		\frac{1}{n+1}
		\delta_{\paren{X_{n+1}, z^{\prime}}}
		+
		\sum_{i=1}^{n}
		\frac{1}{n+1}\delta_{\paren{X_i, Y_i}},
	\end{align}
	where, for $\xi\in\mathcal{X}\times \mathcal{Y}$, $\delta_{\xi}$ denotes the Dirac measure supported on the singleton $\{\xi\}$.
\begin{lemma}
	\label{lm.influence.function.expression}
        Assume \eqref{asm.finite.kern}, \eqref{asm.loss.convex}, \eqref{asm.loss.smooth},
        and \eqref{asm.loss.c2} hold true. Then,
	for every $y,z \in \mathcal{Y}$, the Bouligand influence function
	$\mathrm{BIF}\paren{
		\widehat{P}_{n+1}^{y};
		T_{\lambda},
		\widehat{P}_{n+1}^{z}
	} \in\mathcal{H}$ of $T_{\lambda} : \mathcal{M} \to \mathcal{P}$
	for the distribution $\widehat{P}_{n+1}^{z}$ in the direction of the distribution $\widehat{P}_{n+1}^{y}$ does exist and satisfies
		\begin{align}\label{eq.BIF.Differential}
		\mathrm{BIF}\paren{\widehat{P}_{n+1}^{y}; T_{\lambda}, \widehat{P}_{n+1}^{z}}
		= \mathcal{D}_1\hat{f}_{\lambda}\paren{
			\mathbf{u}; \bullet
		}\paren{\mathbf{w} - \mathbf{u}} \in\mathcal{H} ,
	\end{align}
	where $ \uu = \paren{1, \ldots, 1, 1, 0} \in \mathbb{R}^{n+2}$ and $\uw = \paren{1, \ldots, 0, 1, 0} \in \mathbb{R}^{n+2}$.
	
	Furthermore, its explicit expression is given by
	\begin{align}\label{eq.BIF.function.expression}
		\mathrm{BIF}\paren{
			\widehat{P}_{n+1}^{y};
			T_{\lambda},
			\widehat{P}_{n+1}^{z}
		}
		= -\mathbf{I}_{\hat{f}}\paren{X_{n+1}, z}
		+ \mathbf{I}_{\hat{f}}\paren{X_{n+1}, y},
	\end{align}
	where, for every $z^{\prime} \in \mathcal{Y}$,
	\begin{align*}
		\mathbf{I}_{\hat{f}}\paren{X_{n+1}, z^{\prime}}
		:= - \frac{1}{n+1}
		\dtwo \ell\paren{z^{\prime},
			\hat{f}_{\lambda; D^{z}}\paren{X_{n+1}}}
		\left[
		\DDtwo \wer{\lambda}{\uu}{
			\hat{f}_{\lambda; D^{z}}
		}
		\right]^{+} 
		K_{X_{n+1}} \in \mathcal{A}.
	\end{align*}
\end{lemma}
The proof is detailed in Appendix~\ref{sec.IF.proofs}.
The first statement given by Eq.~\eqref{eq.BIF.Differential} establishes the connection between BIF and the differential of $v \mapsto \hat f_\lambda(v;\bullet ) \in \mathcal{H}$ evaluated at $\uu$ in the direction of $\uw - \uu$.
Recalling that $\wpred{\lambda}{\uu}{\bullet} = \pred{\lambda;}{z}$ and $\wpred{\lambda}{\uw}{\bullet} = \pred{\lambda;}{y}$, it appears that
$\pred{\lambda;}{y}$ can be approximated by a first-order Taylor expansion as
\begin{align*}
	\pred{\lambda;}{y} - \pred{\lambda;}{z} = \wpred{\lambda}{\uw}{\bullet}  - \wpred{\lambda}{\uu}{\bullet}  = \mathcal{D}_1\hat{f}_{\lambda}\paren{
		\mathbf{u}; \bullet
	}\paren{\mathbf{w} - \mathbf{u}} + o( \mathbf{w} - \mathbf{u} ) .
\end{align*}
Fortunately, Eq.~\eqref{eq.BIF.function.expression} asserts that this approximation can be made fully explicit since $\mathbf{I}_{\hat{f}}\paren{X_{n+1}, z^{\prime}}$ can be computed for either $z^\prime = z$ or $z^\prime = y$.	
A possible interpretation of Eq.~\eqref{eq.BIF.function.expression} is that $\mathbf{I}_{\hat{f}}\paren{X_{n+1}, z}$ quantifies the impact (``influence'') of observing $z$ at $X_{n+1}$ on the training of $\pred{\lambda;}{z}$. Therefore, $\mathrm{BIF}\paren{
	\widehat{P}_{n+1}^{y};
	T_{\lambda}, 	\widehat{P}_{n+1}^{z} }$ intuitively quantifies the change of the predictor value as $z$ is replaced by $y$ along the training.

\medskip

From Eq.~\eqref{eq.BIF.function.expression}, one could worry that the underlying quantities are actually functions in $\mathcal{H}$, which could be difficult to compute.
The purpose of next Lemma~\ref{lemma.vector.expression.influence.function} is precisely to provide a vector representation of these quantities which makes them tractable in practice.  
\begin{lemma}[Vector representation of the influence function]
	\label{lemma.vector.expression.influence.function}
	With the same notations and assumptions as Lemma~\ref{lm.influence.function.expression}, for every $z^{\prime} \in \mathcal{Y}$, it comes that $\mathbf{I}_{\hat{f}}\paren{X_{n+1}, z^{\prime}} \in \mathcal{A} \subset \mathcal{H}$ can be represented by the vector $I_{\hat{f}}\paren{X_{n+1}, z^{\prime}} \in \mathbb{R}^{n+1}$
	defined by
	\begin{align*}
		I_{\hat{f}}\paren{X_{n+1}, z^{\prime}}
		:= - \frac{1}{n+1}
		\partial_2 \ell\paren{z^{\prime}, \hat{a}_{\lambda}(\uu)^T K_{\bullet, n+1}} 
		\croch{\nabla_2^2 \vwer{\lambda}{\uu}{\hat{a}_{\lambda}(\uu)}}^+
		K_{\bullet, n+1} \in \mathbb{R}^{n+1},
	\end{align*}
	with
	$\ifunc{z^{\prime}} = \sum_{i=1}^{n+1} \croch{I_{\hat{f}}\paren{X_{n+1}, z^{\prime}}}_i K_{X_i}$, and the vector $\hat{a}_{\lambda}(\uu)$ is the unique minimizer of 
	$\vwer{\lambda}{\uu}{\bullet}$ (see Eq.~\ref{eq.vec.wer})
	in the range of the Gram matrix $K$.
\end{lemma}
The proof is deferred to Appendix~\ref{sec.IF.proofs}.

\medskip


The following definition yields an expression
for the new predictor approximation using influence functions.
\begin{definition}[IF-based predictor approximation]
	\label{def.BIF.based.approximation}
	With the same notations and assumptions as Lemmas~\ref{lm.influence.function.expression} and~\ref{lemma.vector.expression.influence.function}, for any candidate prediction value $y \in \mathcal{Y}$, the influence function-based predictor approximation 
$\tilde{f}_{\lambda; D^y}^{\mathrm{IF}}$ is defined, for any $z\in\mathcal{Y}$ as
\begin{align*}
	\tilde{f}_{\lambda; D^y}^{\mathrm{IF}} 
	&
	:=	\pred{\lambda;}{z}
	- \ifunc{z}
	+ \ifunc{y} \in\mathcal{H},
\end{align*}
where
$\ifunc{z^{\prime}} = \sum_{i=1}^{n+1} \croch{I_{\hat{f}}\paren{X_{n+1}, z^{\prime}}}_i K_{X_i}$, for every $z^{\prime} \in \mathcal{Y}$,
and 
\begin{align*}
	I_{\hat{f}}\paren{X_{n+1}, z^{\prime}}
	:= - \frac{1}{n+1}
	\partial_2 \ell\paren{z^{\prime}, \hat{a}_{\lambda}(\uu)^T K_{\bullet, n+1}} 
	\croch{\nabla_2^2 \vwer{\lambda}{\uu}{\hat{a}_{\lambda}(\uu)}}^+
	K_{\bullet, n+1}.
\end{align*}
\end{definition}
Compared with the previous predictor approximation $\hat{f}_{\lambda; D^{z}}$,
the new one makes use of a first-order corrective term
using influence functions (IF).
This allows the new IF-based predictor to depend on the candidate prediction value $y$, unlike the previous approximations from Sections~\ref{sec.non.smooth} and~\ref{sec.smooth} where only a fixed $z\in\mathcal{Y}$ was considered.
Moreover, compared to the full-conformal setting
where a complete retraining of a predictor is required for
each value of $y$, the only term varying with $y$
is the first-order derivative of the loss-function evaluated at $\paren{y, \hat{a}_{\lambda}(\uu)^T K_{\bullet, n+1}}$.
Every other terms only need to be computed once. 
On the one hand, all of this makes the new approximation $\tilde{f}_{\lambda; D^y}^{\mathrm{IF}}$ more versatile than the previous ones from Sections~\ref{sec.non.smooth} and~\ref{sec.smooth}. On the other hand, it also leads to a computationally tractable approximation to the predictive full-conformal prediction-region.

\begin{remark}\label{remark.RLA}
Although Definition~\ref{def.BIF.based.approximation} displays a computationally tractable
expression for the predictor approximation, computing the pseudo-inverse $\croch{\nabla_2^2 \vwer{\lambda}{\uu}{\hat{a}_{\lambda}(\uu)}}^+$ remains challenging. It is well known that inverting such a $(n+1) \times (n+1)$  psd matrix generally requires
$O(n^{3})$ flops, which can make computations quite heavy as $n$ grows.

Fortunately, several techniques from the fields of RLA (Random Linear Algebra) can be harnessed for approximating the inverse of such a large matrix at a low computational price. or instance, we refer interested readers to 
\citep{gittens2016revisiting,musco2017sublinear, wang2019scalable,tropp2023randomized} for a few references.
\end{remark}

\subsection{Assessing the quality of the IF-based approximate predictor/scores}
\label{sec.quality.ifunc}

Let us start by reviewing two additional assumptions that allow us to exploit the higher smoothness of the underlying loss to derive faster rates.
\begin{assumption}
	For every $y \in \mathcal{Y}$
	\begin{equation}
		\label{asm.loss.c3}
		\tag{C3L}
		\mbox{
			$u  \in \mathcal{Y} \mapsto \ell(y, u)$ is three times differentiable, 
		}    
	\end{equation}
	with $\dddtwo \ell(y, \bullet)$ its third derivative.
\end{assumption}
Requiring third-order derivatives is not a new requirement in the machine learning literature. For instance, \citet{nesterov1994interior} introduce and study the class of self-concordant functions, which has been further generalized to the Generalized Self-Concordant property by \citet{marteau2019globally,dvurechensky2023generalized} for instance.
Let us also mention that the three examples of loss-functions from Section~\ref{sec.predictor} all fulfill this requirement.

\begin{assumption}
	There exists a constant $\xi_{\ell} \in \paren{0, \infty}$ such that,
	for every $y \in \mathcal{Y}$,
	\begin{equation}
		\label{asm.loss_bounded_d3}
		\tag{$\xi_\ell$-BdD3}
		\mbox{
			$\sup_{u \in \mathcal{Y}}\abss{\dddtwo \ell(y, u)} \leq \xi_{\ell}$.
		}    
	\end{equation}
\end{assumption}
This assumption is equivalent to requiring that the second-order derivative $ u \mapsto \partial_2^2 \ell(y,u)$ is $\xi_\ell$-Lipschitz continuous, for every $y\in\mathcal{Y}$.
For instance, it holds true for the three loss-functions listed
    at the end of Section~\ref{sec.predictor} (see Appendix~\ref{sec.example.loss.function} for the detailed calculations):
    \begin{enumerate}
        \item Lemma~\ref{lm.logcosh.d3}
        establishes that this holds with $\xi_{\ell} = \frac{1}{a^2}$
        for the \textit{logcosh loss-function} function, and
        \item Lemma~\ref{lm.pseudo.huber.d3}
        proves that this holds with $\xi_{\ell} = \frac{1}{a} \frac{3}{2} \paren{\frac{4}{5}}^{5/2}$
        for the \textit{pseudo-Huber loss-function} function, and
        \item Lemma~\ref{lm.smooth.pinball.d3}
        settles that this holds with $\xi_{\ell} = \frac{1}{a^2}
        \frac{5 + 3\sqrt{3}}{\paren{\sqrt{3}+3}^3}$
        for the \textit{smoothed pinball loss-function} function.
    \end{enumerate}
As the values of $a \in \paren{0, +\infty}$ becomes smaller, the bounds worsen. This is consistent with the fact that the smaller the parameter $a$, the less smooth the corresponding loss-functions introduced at the end of Section~\ref{sec.predictor}.
Coming back to the present context, let us mention that our analysis of the predictor approximation quality involves applying the second-order mean-value theorem to the function $v \mapsto \hat{f}_{\lambda}\paren{v; \bullet}$.
Then \eqref{asm.loss_bounded_d3} is key to uniformly
upper-bound the third-order differential
of the weighted regularized empirical-risk function
that appears in the expression
of the second-order differential of the function
$v \mapsto \hat{f}_{\lambda}\paren{v; \bullet}$.
Let us emphasize that this assumption is not new when analyzing theoretical properties of optimization algorithms. For instance, \citet[in Theorem 3.5]{nocedalNumericalOptimization2006} used
this assumption to prove convergence results for Newton's method.
This is also related to self-concordance property \citep{nesterov1994interior},
where the third derivative is bounded from above
using the second derivative. 
In fact, if the loss-function is self-concordant and $\beta_{2; \ell}$-smooth,
then it fulfills the above assumption.

\medskip

We are now in position to establish the main approximation results of the present section.
The next upper-bound quantifies the quality of the IF-based approximation $\tilde{f}_{\lambda; D^y}^{\mathrm{IF}} $ to $\hat{f}_{\lambda; D^{y}}$ in terms of $\norm{\bullet}_{\mathcal{H}}$. 
\begin{proposition}
	\label{prop.predictor.stability.Verysmooth}
	Assume \eqref{asm.finite.kern}, \eqref{asm.loss.convex}, \eqref{asm.loss.smooth}, \eqref{asm.loss.c3}, and \eqref{asm.loss_bounded_d3} hold true.
	Then we have
	\begin{align*}
		\left\|
		\hat{f}_{\lambda; D^{y}}
		- \tilde{f}_{\lambda; D^{y}}^{\mathrm{IF}}
		\right\|_{\mathcal{H}}
		\leq\sqrt{K_{n+1, n+1}} \min \paren{
			\frac{\rho^{(2)}_{\lambda}(y)}{
				\lambda^3 (n+1)^2},
			\frac{2 \rho^{(1)}_{\lambda}(y)}{\lambda (n+1)}
		},
	\end{align*}	
	where
	\begin{align}
		\label{eq.predictor.approx.bound.very.smooth}
		\rho^{(2)}_{\lambda}(y) 
		:= & \frac{\xi_{\ell}}{2} 
		\sqrt{K_{n+1, n+1}}
		\left(
		\frac{1}{n+1} 
		\sum_{i=1}^{n+1} 
		\paren{K_{i, i}}^{3/2}
		\right)
		\paren{
			\tilde{\rho}^{(1)}_{\lambda}(y)
		}^2
		+ 2 \lambda (K_{n+1, n+1}) \beta_{\ell; 2} \tilde{\rho}^{(1)}_{\lambda}(y), \\
		\mbox{and}\quad         
		\tilde{\rho}^{(1)}_{\lambda}(y) &  :=
		\left(
		1 + K_{n+1, n+1} \frac{\beta_{\ell; 2}}{\lambda (n+1)}
		\right) \rho^{(1)}_{\lambda}(y), \label{eq.rho.tilde} 
	\end{align}
	with $\rho^{(1)}_{\lambda}(y) $ given by Eq.~\eqref{eq.predictor.approx.bound.smooth}.
\end{proposition}
The proof is postponed to Appendix~\ref{proof.score.bound.very.smooth}.
The main dependence of the upper-bound with respect to $n$ lies at the denominator in the right-hand side. To be more specific, if one assumes for instance that the kernel is bounded, then $\rho^{(1)}_{\lambda}(y)$, $\tilde{\rho}^{(1)}_{\lambda}(y)$, and $\rho^{(2)}_{\lambda}(y) $ become bounded. Therefore the decay rate is of order $O( \min\paren{ (\lambda n)^{-1}, (\lambda^3 n^2)^{-1}} )$, which is not worse than $O( (n\lambda)^{-1})$ from Proposition~\ref{prop.predictor.stability.smooth}.
This is consistent with the fact that Proposition~\ref{prop.predictor.stability.Verysmooth} exploits higher-order smoothness properties of the loss-function than Proposition~\ref{prop.predictor.stability.smooth}. 

It is also important to notice that learning algorithms relying on regularization terms (such as the one introduced in Definition~\ref{def.predictor}) usually perform well with a data-driven choice of the regularization parameter $\hat \lambda>0$ that is chosen to converge to 0 as $n$ grows. This is why the decay rate of the right-hand side should be understood as resulting from the interaction between $\lambda$ and $n$.

\medskip

Consistently with Definition~\ref{def.BIF.based.approximation}, for any given $z \in\mathcal{Y}$, the non-conformity scores approximation are then defined by
for every $y \in \mathcal{Y}$ as
\begin{align}
\label{eq.approximate.score.1}
\begin{aligned}
    S_{\lambda; D^{y}}\paren{X_i, Y_i}
    \approx
    \widetilde{S}_{\lambda; D^{y}}^{\mathrm{IF}}\paren{X_i, Y_i} 
    &:= s\paren{Y_i, \tilde{f}_{\lambda; D^y}^{\mathrm{IF}}(X_i)},
    &&\mbox{if } 1 \leq i \leq n\\
    S_{\lambda; D^{y}}\paren{X_i, Y_i}
    \approx
    \widetilde{S}_{\lambda; D^{y}}^{\mathrm{IF}}\paren{X_i, y} 
    &:= s\paren{y, \tilde{f}_{\lambda; D^y}^\mathrm{IF}(X_i)},
    &&\mbox{if } i=n+1.
\end{aligned}
\end{align}
As a consequence of the Lipschitz property \eqref{asm.score.lipschitz} of the scores, the next theorem is actually a straightforward  corollary of Proposition~\ref{prop.predictor.stability.Verysmooth}, 
\begin{theorem}
    \label{thm.score.bound.very.smooth}
        Assume \eqref{asm.finite.kern}, \eqref{asm.loss.convex}, 
        \eqref{asm.loss.smooth}, \eqref{asm.loss.c3}, \eqref{asm.loss_bounded_d3}, and
        \eqref{asm.score.lipschitz} hold true. Then,
    for every $y \in \mathcal{Y}$,
    \begin{align*}
        \begin{aligned}
            \abss{
                S_{\lambda; D^{y}} \paren{X_i, Y_i}
                - \widetilde{S}_{\lambda; D^{y}}^{\mathrm{IF}} \paren{X_i, Y_i}
            }
            &\leq \widehat{\tau}_{\lambda; i}^{(2)}(y),
            &&
            \mbox{ if }1\leq i \leq n\\
            \abss{
                S_{\lambda; D^{y}} \paren{X_i, y}
                - \widetilde{S}_{\lambda; D^{y}}^{\mathrm{IF}} \paren{X_i, y}
            }
            &\leq \widehat{\tau}_{\lambda; i}^{(2)}(y),
            &&
            \mbox{ if }i = n+1,
        \end{aligned}
    \end{align*}
    where
    \begin{align*}
        \widehat{\tau}_{\lambda; i}^{(2)}(y)
        := \sqrt{K_{i, i}}
        \sqrt{K_{n+1, n+1}}
        \min\left(
            \frac{
                \gamma    
                \rho^{(2)}_{\lambda}(y)
            }{
                \lambda^3 (n+1)^2
            },
            \frac{
                2 \gamma
                \rho^{(1)}_{\lambda}(y)
            }{
                \lambda (n+1)
            }
        \right),
    \end{align*}
    with $\rho^{(2)}_{\lambda}(y)$ is defined in Eq.~\eqref{eq.predictor.approx.bound.very.smooth} and $\rho^{(1)}_{\lambda}(y) $ given by Eq.~\eqref{eq.predictor.approx.bound.smooth}.
\end{theorem}
The proof is given in Appendix \ref{proof.score.bound.very.smooth}.
Assuming that the kernel is bounded implies that $y \mapsto \rho^{(2)}_{\lambda}(y)$ and $y \mapsto \rho^{(1)}_{\lambda}(y)$ are bounded.
Let us also assume that the regularization parameter $\lambda \propto 1/(n+1)^{r}$, for some $r \in [0, 1)$.
Then the upper-bound decreases at the rate $O\paren{n^{-1 + r + \min \paren{-1 + 2r, 0}}}$, compared to $O\paren{n^{-1 + r}}$ from Theorem~\ref{thm.score.bound.smooth}.
On the one hand if $r < 1/2$, the new upper-bound vanishes faster than the previous one by an exponent $-1 + 2r<0$. The most favorable case being $r=0$
that is, when $\lambda$ remains constant.
On the other hand if $r \geq 1/2$, both upper-bounds exhibit the same convergence rate.

Regarding the decay rate of the regularization parameter $\lambda$, \citet{caponnetto2007optimal} proved that the optimal (minimax) rate depends on the eigenvalues decay of the covariance operator
(see Eq.~\ref{eq.cov.op}) and the smoothness of the regression function.
To be more specific the smoother the regression function and the slower the decay rate of the eigenvalues, the closer $r\in[0,1) $ should be to 0.
(see also \citet{lin2017distributed,yu2021robust,blanchard2018optimal}).
Coming back to the present section where smoother functions are considered, the optimal value of $r$ should be intuitively taken close to 0 (and eventually $r<1/2$) by contrast to the optimal value corresponding to less smooth functions from Sections~\ref{sec.non.smooth} or~\ref{sec.smooth}.
In other words, smooth enough functions should yield instances for which the approximation error of the IF-based approximate scores should decay faster.

\subsection{Quality of the IF-based approximate full-conformal prediction-region}
\label{sec.quality.ifunc.region}

Following the approximation scheme described in Definition~\ref{def.approx.prediction.region}, it results that the upper-approximate full-conformal prediction-region, denoted by  $\ufcprr{,(2)}{\lambda; \alpha}$ is given by
\begin{definition}[IF-based approximate full-conformal prediction-region]
	\label{def.IF-based.approximate.region}
	From the non-conformity scores approximation are given in Eq.~\eqref{eq.approximate.score.0}
	and the upper-bounds in Theorem~\ref{thm.score.bound.very.smooth}, the IF-based approximate full-conformal prediction-region is defined, for any $\alpha\in(0,1)$, by 
	\begin{align}\label{eq.approximate.full.region.very.smooth}
		\ufcprr{,(2)}{\lambda; \alpha}
		:= \brac{
			y \in \mathcal{Y} :
			\ufcpvv{,(2)}{\lambda; D}{y}
			> \alpha
		},
	\end{align}
	where, for every $y \in \mathcal{Y}$,
	\begin{align*}
		\ufcpvv{,(2)}{\lambda; D}{y}
		:= \frac{
			1 
			+ 
			\sum_{i=1}^{n} 
			\mathbbm{1}
			\left\{
			\widetilde{S}_{\lambda; D^{y}}^{\mathrm{IF}} \paren{X_i, Y_i}
			+
			\widehat{\tau}_{\lambda; i}^{(2)} \paren{y}
			\geq 
			\widetilde{S}_{\lambda; D^{y}}^{\mathrm{IF}} \paren{X_{n+1}, y}
			-
			\widehat{\tau}_{\lambda; n+1}^{(2)} \paren{y}
			\right\}
		}{
			n+1
		}.
	\end{align*}
	Moreover let us note $\lfcprr{,(2)}{\lambda; \alpha}$,
	the corresponding lower-approximate full-conformal prediction-region,
	analog to the prediction-region defined in Eq.~\eqref{eq.lower.approx.region}.
\end{definition}

Applying Lemma~\ref{lm.coverage.approximate.region} to the above approximate full-conformal prediction-region immediately leads to the next coverage guarantee.
\begin{corollary}[Coverage]
	\label{cor.coverage.very.smooth}
	Assume %
        \eqref{asm.finite.kern}, \eqref{asm.loss.convex}, 
        \eqref{asm.loss.smooth}, \eqref{asm.loss.c3}, \eqref{asm.loss_bounded_d3}, and
        \eqref{asm.score.lipschitz}, along with the first set of assumptions in Thereom~\ref{thm.coverage} hold true.
	With the above notations, for any control-level $\alpha\in(0,1)$, the approximate full-conformal prediction-region $\ufcprr{,(2)}{\lambda; \alpha}$ given by Eq.~\eqref{eq.approximate.full.region.very.smooth} is a true confidence-region at level $1-\alpha$, that is,
	\begin{align*}
		\mathbb{P}
		\croch{Y_{n+1} \in \ufcprr{, (2)}{\lambda; \alpha} } \geq 1 - \alpha.
	\end{align*}
\end{corollary}
The detailed proof has been omitted since this is a straightforward consequence of 
Lemma~\ref{lm.coverage.approximate.region} exploiting that $  \fcpr 
\subseteq  \ufcprr{, (2)}{\lambda; \alpha} $. 
Let us also emphasize that Corollary~\ref{cor.coverage.very.smooth} holds true under the minimum requirements for the approximation $\tilde{f}_{\lambda; D^y}^\mathrm{IF}$ to exist.

\medskip

An important question for assessing the validity of such a confidence-region owes to quantifying its size: Any overly large confidence-region would be non informative in practice. 
The purpose in what follows is to assess the amount by which the approximate region $\ufcprr{, (2)}{\lambda; \alpha}$ from Eq.~\eqref{eq.approximate.full.region.very.smooth} is upper-bounding the full-conformal one.

\medskip

For quantifying the size of the IF-based approximate confidence-region, let us introduce two additional assumptions.

The first one requires the boundedness of the reproducing kernel.
\begin{assumption}
	There exists a constant $\kappa \in \paren{0, \infty}$ such that,
	\begin{equation}
		\label{asm.bounded.kernel}
		\tag{$\kappa$-BdK}
		\mbox{$\forall x \in \mathcal{X}$, $k_{\mathcal{H}}(x, x) \leq \kappa^2$.}
	\end{equation}
\end{assumption}
This is a very classical assumption in the machine learning literature.
Examples of such kernels are the Gaussian (RBF) kernel \citep{keerthiAsymptoticBehaviorsSupport2003} or
the Laplacian kernel $k(\bullet, \bullet): (x, \tilde{x}) \to \exp(-\norm{x - \tilde{x}}_1 / g)$ \citep{ahirFeatureMapsLaplacian2025}.

\medskip

The second assumption exploits the smoothness of the loss with respect to its first argument.
\begin{assumption}
    There exists a constant $\beta_{\ell; 1} \in \paren{0, \infty}$
    such that, for every $u \in \mathcal{Y}$,
    \begin{align*}
        \label{asm.loss.smooth.2}
        y \mapsto \partial_2 \ell\paren{y, u}
        \mbox{ is }\beta_{\ell; 1}\mbox{-Lipschitz continuous.}
        \tag{$\beta_{\ell; 1}$-LipDL.1}
    \end{align*}
\end{assumption}
    This assumption serves for proving that
    $y \in \mathcal{Y} \mapsto \tilde{f}_{\lambda; D^{y}}^{\mathrm{IF}} \in \mathcal{H}$ is Lipschitz continuous since it only depends on $y$
    through the function $\partial_2\ell\paren{y, \bullet}$
    (see Definition~\ref{def.BIF.based.approximation}).
Let us also notice that with a symmetric loss, the previous \eqref{asm.loss.smooth} implies that \eqref{asm.loss.smooth.2}
holds true with $\beta_{\ell; 1} = \beta_{\ell; 2}$.

\bigskip

The next theorem quantifies the size of the thickness $\thick{(2)}{\lambda; \alpha}$ of the prediction-region $\ufcprr{,(2)}{\lambda; \alpha}$, which is defined (see Definition~\ref{def.thickness}) by
\begin{align*}
    \thick{(2)}{\lambda; \alpha} := \leb{
        \ufcprr{,(2)}{\lambda; \alpha}
        \setminus \fcprr
    }.
\end{align*}
\begin{theorem}
\label{thm.very.smooth}
Assume %
    \eqref{asm.finite.kern}, \eqref{asm.loss.convex}, 
    \eqref{asm.loss.smooth}, \eqref{asm.loss.c3}, \eqref{asm.loss_bounded_d3}, and
    \eqref{asm.score.lipschitz} hold true (see Theorem~\ref{thm.score.bound.very.smooth}) as well as  \eqref{asm.loss.lipschitz}, \eqref{asm.bounded.kernel}, and \eqref{asm.loss.smooth.2}.
For any control-level $\alpha \in (0, 1)$, let the
IF-based prediction-region $\ufcprr{,(2)}{\lambda; \alpha}$ be given by Eq.~\eqref{eq.approximate.full.region.very.smooth}, and
consider the non-conformity function 
$s\paren{\bullet, \bullet} : \mathcal{Y} \times \mathcal{Y} \to \mathbb{R}$, $(y, u) \mapsto s(y, u) = \abss{y - u}$.
Then, 
\begin{itemize}
	\item if the integer $n>0$ is such that $ \lambda (n+1) > \frac{\kappa^2 \beta_{\ell; 1}}{2}$, we have 
	\begin{align}
		\label{eq.bound.thickness.very.smooth}
		\thick{(2)}{\lambda; \alpha}
		\leq
		\min \paren{
			O \paren{
				\frac{
					6 \kappa^6 \rho^2 \xi_{\ell}
				}{
					\lambda^3 (n+1)^2
				}
			},
			O \paren{\frac{12 \kappa^2 \rho}{\lambda (n+1)}}
		}.
	\end{align}
	
	\item otherwise, $0< \lambda (n+1) \leq \frac{\kappa^2 \beta_{\ell; 1}}{2}$ implies that
	\begin{align*}
		\thick{(2)}{\lambda; \alpha} \leq O \paren{ \frac{\kappa^2 \rho}{\lambda (n+1)} }.
	\end{align*}
\end{itemize}
\end{theorem}
The proof is deferred to Appendix~\ref{proof.thm.very.smooth}.
For this choice of non-conformity score, the accuracy of the new prediction-region improves at a rate that is proportional to the upper-bounds on the non-conformity scores. If the regularization parameter $\lambda \propto 1/(n+1)^{r}$ for some $r \in (0, 1)$, then the new upper-bounds improve at a rate of $O\paren{n^{-1 + r + \min \paren{-1 + 2r, 0}}}$ compared to $O\paren{n^{-1 + r}}$ for the upper-bounds from the previous sections.
\medskip

\begin{corollary}
    \label{cor.very.smooth}
	With the same assumptions as those of Theorem~\ref{thm.very.smooth}, assume that there exists a constant $c>0$ such that the conditional density $p(\bullet|D, X_{n+1}) : \mathcal{Y} \to \mathbb{R}_+$, is bounded from above by $c$, and that the non-conformity scores $S_{D^{Y_{n+1}}} \paren{X_1, Y_1}, \ldots, S_{D^{Y_{n+1}}} \paren{X_{n+1}, Y_{n+1}}$ are almost-surely distinct. 
	Then,
	\begin{itemize}
		\item if the integer $n>0$ is such that $ \lambda (n+1) > \frac{\kappa^2 \beta_{\ell; 1}}{2}$, we have 
		\begin{align*}
			& \mathbb{P}
			\croch{Y_{n+1} \in \ufcprr{, (2)}{\lambda; \alpha}}  \\
			& \leq  (1 - \alpha ) + \frac{1}{n+1}
			+ \min \paren{
				O \paren{
					\frac{
						c \kappa^6 \rho^2 \xi_{\ell}
					}{
						\lambda^3 (n+1)^2
					}
				},
				O \paren{\frac{c\kappa^2 \rho}{\lambda (n+1)}}
			}.
		\end{align*}
		
		\item otherwise, $0< \lambda (n+1) \leq \frac{\kappa^2 \beta_{\ell; 1}}{2}$ implies that
		\begin{align*}
			\mathbb{P}
			\croch{Y_{n+1} \in \ufcprr{, (2)}{\lambda; \alpha}}
			\leq 1 - \alpha + \frac{1}{n+1}
			+ O \paren{
				\frac{c\kappa^2 \rho}{\lambda (n+1)}
			}.
		\end{align*}
	\end{itemize}
\end{corollary}
The proof is postponed to Appendix~\ref{proof.cor.very.smooth}.
   The new upper-bound on the probability of coverage
    is not worse (up to a multiplicative constant)
    than the previous one described by Corollary~\ref{cor.non.smooth}.
    More precisely for large values of $\lambda\paren{n+1}$,
    there is a speed up in the decay rate reflecting
    the effect of a higher smoothness.
    Moreover, for $\lambda \propto (n+1)^{-r}$
    with $r \in \paren{0, 1/3}$, the coverage-probability upper-bounds
    $1-\alpha$ by an amount of order $O\paren{n^{-1}}$ for the new approximation compared to $O\paren{n^{-1+r}}$ for the former one.

\bigskip

\noindent{\emph{Illustration.}}
%
%
%
The next figure follows the same idea as
previous Figures~\ref{fig.comp.bound.non.smooth.0}
and~\ref{fig.comp.bound.smooth.0}.
It displays the gap between $\Delta^{(2)} := \leb{\ufcprr{,(2)}{\lambda; \alpha}
	\setminus \lfcprr{,(2)}{\lambda; \alpha}}$ (solid blue-line) which upper-bounds the thickness,
and its upper-bound detailed in Eq.~\eqref{eq.bound.thickness.very.smooth} (dashed red-line). 
\begin{figure}[H]
    \centering
    \includegraphics[width=0.60\textwidth]{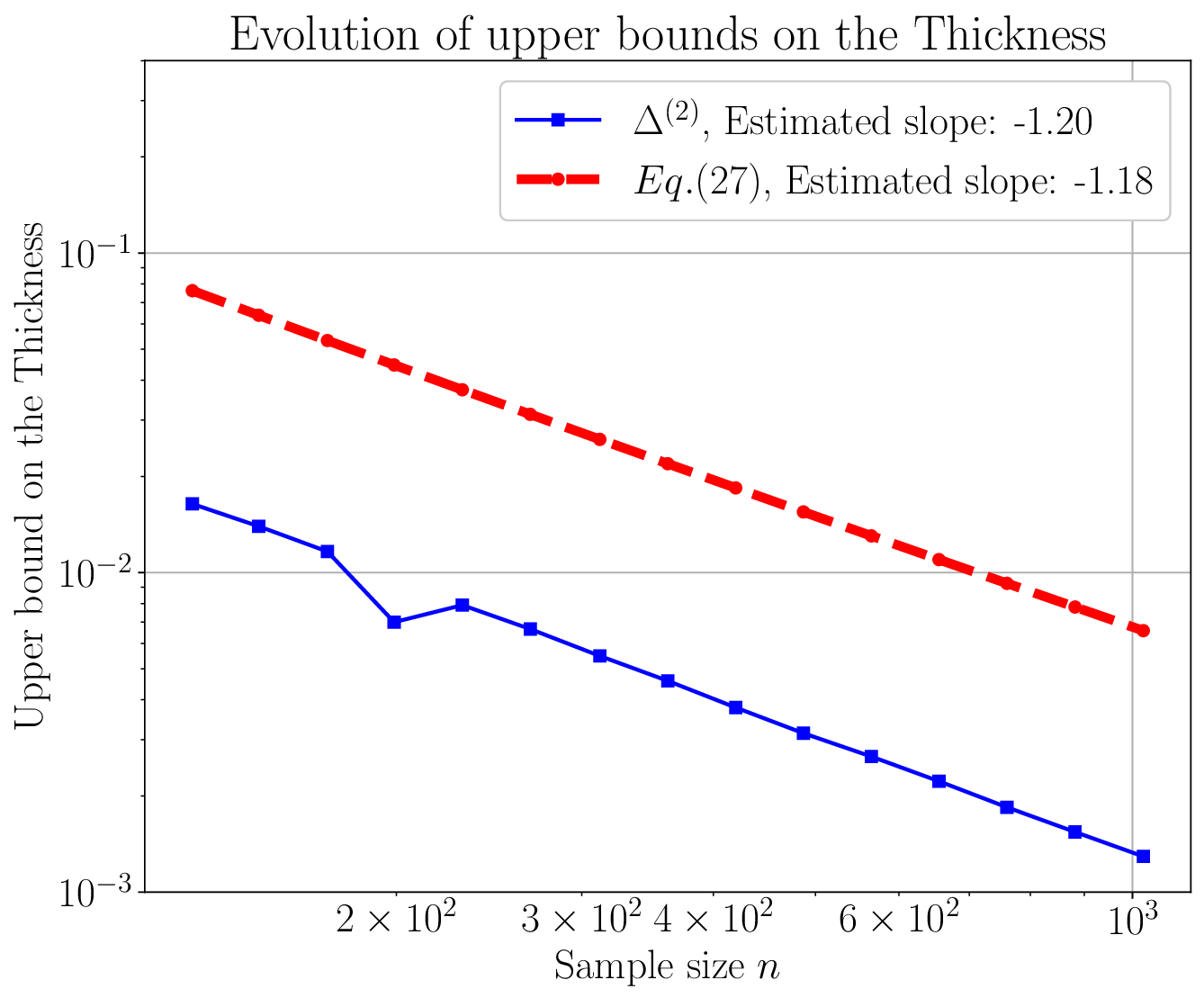}
    \caption{
        Evolution of
        the upper-bound in Eq.~\eqref{eq.bound.thickness.very.smooth} (dashed red-line)
        and the quantity $\Delta^{(2)}$ (solid blue-line)
        as a function of the sample size $n$
        in $\log\log$ scale (to appreciate the rate).
        The data is sampled from $\mathrm{sklearn}$ synthetic data set
        make\_friedman1(sample\_size=n).
        The kernel $k_\mathcal{H} \paren{\bullet, \bullet}$
        is set to be the Laplacian kernel (gamma=None).
        The loss-function $\ell\paren{\bullet, \bullet}$
        is set to be the Logcosh Loss ($a=1.0$).
        The regularization parameter is set to decay as $\lambda \propto (n+1)^{-0.33}$.
        The fixed output value is set at $z=0$.        
        The non-conformity function is set to be $s(\bullet, \bullet) : (y, u) \mapsto \abss{y - u}$.
    }
    \label{fig.comp.bound.very.smooth.0}
\end{figure}
The current dashed red-line is
lower and has a lower estimated slope of $-1.20$
compared to the previous dashed red-line in Figure~\ref{fig.comp.bound.smooth.0},
which has an estimated slope of $-0.67$.
The same observation can be made when comparing the current
solid blue-line with the previous ones in Figure~\ref{fig.comp.bound.non.smooth.0}
and Figure~\ref{fig.comp.bound.very.smooth.0}.
This is consistent with the claim that
refining the prediction approximation
using influence functions leads to a tighter sandwiching of
the full-conformal prediction-region.
Moreover, the current approximation improves faster than the previous ones
when $\lambda \propto (n+1)^{-r}$ with $r <1/2$.

The current dashed red-line and
the solid blue one have similar slopes.
This suggests that the upper-bound in
Eq.~\eqref{eq.bound.thickness.very.smooth} truly captures the same dependence on $r$ as the one of $\Delta^{(2)}$.
Let us also mention that for both quantities, the estimated slopes are a somewhat lower than $-2 + 3r = -1.01$ (with $r=0.33$).
For the dashed red-line, this slight difference is due to a multiplicative factor $\paren{1 - \frac{\kappa^2 \beta_{\ell; 1}}{\lambda(n+1)}}^{-1}
= 1 + o(1)$ that can be neglected as $n$ becomes large, but still skews the rate when $n$ remains small.

\section{Practical considerations}

\subsection{Empirical experiments on a synthetic data set}
The present experiments aim at exploring the size of several predictive regions obtained by different conformal prediction strategies.
The synthetic data set is generated using the `make\_friedman1(n\_samples=200)'
function \citep{friedmanMultivariateAdaptiveRegression1991} of the sklearn package, which has been already applied in the context of conformal prediction by \citet[in Figure 3, bottom right]{ndiaye2022stable}.

Four conformal prediction strategies are compared:
\begin{itemize}
	\item \textbf{SplitCP} denotes the split-conformal prediction described along Section~\ref{sec.approx.prediction.region} with half of the data set for proper training and the other half for calibration,
	\item \textbf{UStableCP} refers to the strategy exposed in Section~\ref{sec.non.smooth}, which relies on uniform stability-bounds (see Theorem~\ref{thm.score.bound.non.smooth}),
	\item \textbf{LocStableCP} denotes our first new strategy introduced in Section~\ref{sec.smooth}, which results from local stability-bounds (see Theorem~\ref{thm.score.bound.smooth}) with $z=0$,
	\item \textbf{InfluenceFunctionCP} stands for our second new method based on Influence Function. It is described in Section~\ref{sec.very.smooth} (see Theorem~\ref{thm.score.bound.very.smooth}) with $z=0$.
\end{itemize}
These four strategies are applied to the predictor detailed in Section~\ref{sec.predictor} that is instantiated with the \emph{Logcosh} loss-function, the parameter value $a = 1.0$,
and the Laplacian kernel with parameter $\mathrm{gamma=None}$.

It is important to keep in mind that the full-conformal prediction-region cannot be computed in the present setting.
Therefore, a reference region is introduced that is called the \textbf{OracleCP} region, denoted by $\widehat{C}^{\mathrm{oracle}}_{\alpha}\paren{X_{n+1}}$, \citep[following][]{ndiaye2022stable, ndiaye2023root}, and given by
\begin{align}
	\label{eq.oracle.prediction.region}
	\widehat{C}^{\mathrm{oracle}}_{\alpha}\paren{X_{n+1}}
	:= \brac{
		y \in \mathcal{Y}:
		\frac{
			1 + \sum_{i=1}^{n}
			\mathbbm{1}\brac{
				S_{D^{Y_{n+1}}}\paren{X_i, Y_i}
				\geq S_{D^{Y_{n+1}}}\paren{X_{n+1}, y}
			}
		}{
			n+1
		} > \alpha
	}.
\end{align}
Let us emphasize that this region is said an``oracle'' region since it exploits the unknown true output $Y_{n+1}$ corresponding to the input $X_{n+1}$. Unlike the full-conformal prediction-region, this one can be computed on synthetic data since it only requires training one predictor.
However in practice, it is not available since it is built upon the unknown $Y_{n+1}$.

\begin{figure}[H]
	\centering
	\includegraphics[width=0.9\textwidth]{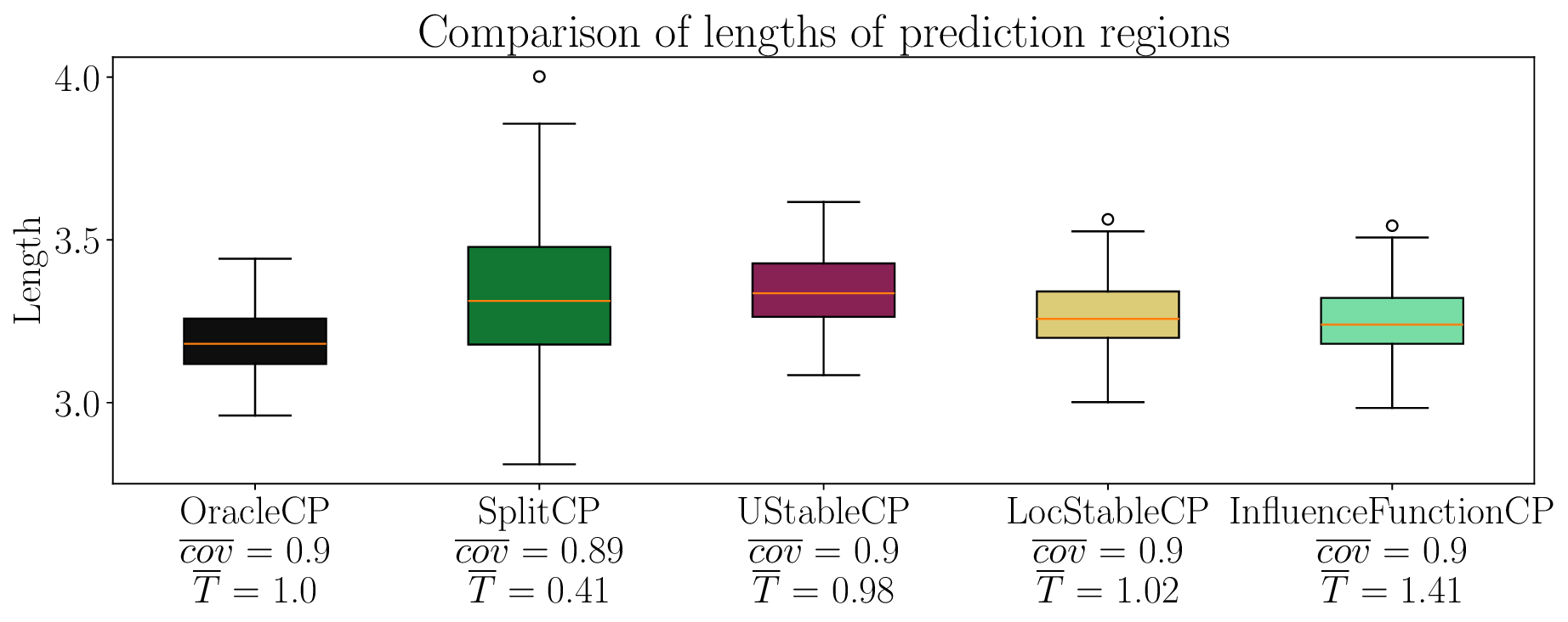}
	\caption{
		Comparison of the length of different prediction-regions
		over $N=100$ repetitions.
		$\overline{cov}$ is the proportion (over the 100 repetitions) of
		test output points contained by the prediction-regions.
		$\overline{T}$ is the average (over the 100 repetitions)
		computation time relative to the computation time of \textbf{OracleCP}.
		The confidence control-level is set at $\alpha = 0.1$.
	}
	\label{fig.comp}
\end{figure}
Figure~\ref{fig.comp} displays boxplots corresponding to the five competitors introduced just above. 
Each boxplot describes the empirical distribution of the length (Lebesgue measure) of each conformal prediction-region over $N=100$ replications.
Along th $x$-axis, $\overline{cov}$ refers to the empirical coverage-probability, while $\overline T$ denotes the average computation time required for accessing the conformal prediction-region. Let us mention that \textbf{OracleCP} serves as a reference for the computation time. For instance, $\overline{T}=0.5$ means that the method requires half as much time as \textbf{OracleCP} to output the prediction-region, while $\overline{T}=2$ means the corresponding method requires twice as much time as \textbf{OracleCP}.  

\medskip

All the competing strategies report an average coverage of about $0.9$, which is consistent with the pre-specified control-level $\alpha = 0.1$.

With the smallest computation time,
\textbf{SplitCP} produces prediction-regions
with a high average-length (among the two highest ones among the competitors). It also exhibits the highest variability in the lengths. This results the loss of information induced by data splitting.

By contrast the computation time of \textbf{UStableCP} is twice the one of \textbf{SplitCP}, while both strategies have almost the same average-length.
However the empirical variance of \textbf{UStableCP} is noticeably lower, which can be explained by the use of the full train set unlike \textbf{SplitCP}.

The new strategy \textbf{LocStableCP} based on local stability-bounds produces prediction-regions with a lower average-length compared
to the one of \textbf{UStableCP}, while the empirical variance seems very similar.
This means that on average, the conformal prediction-regions output by  \textbf{LocStableCP} are narrower (more informative) than the ones of  \textbf{UStableCP} with almost the same computation time.
Compared to \textbf{UStableCP}, this results from the tighter upper-bounds derived by exploiting the local behaviour.

A further improvement is still available for \textbf{InfluenceFunctionCP} at the price of a larger computation time ($\overline{T}=1.41$ compared to $1$).
As explained earlier (see Remark~\ref{remark.RLA}), the larger computation time of \textbf{InfluenceFunctionCP} results from inverting a Hessian matrix which can be time consuming. However using low-rank approximations from the Random Linear Algebra field can drastically reduce this time consumption when facing high-dimension problems.
In practice \textbf{InfluenceFunctionCP} exhibits the smallest average-length (that is, the most informative prediction-regions) as well as a very low empirical variance (similar to the one of all previous stability-based strategies).
Compared with \textbf{LocStableCP}, this results from exploiting the higher smoothness of the loss-function by introducing the influence function approximation. 

\medskip

In summary with a reasonable computation time,
\textbf{LocStableCP} produces prediction-regions with small length and therefore reaches a good trade-off between precision and time consumption. 
However if more computation resources are available, then \textbf{InfluenceFunctionCP} should be preferred to any other competitor among the ones considered here since it outputs the overall smallest prediction-regions with a very low variability.

\subsection{Choice of regularization parameter}

Everything that was done so far holds true for a fixed $\lambda$.
However, an important practical problem consists in making a data-driven choice of $\lambda$ (denoted by $\hat \lambda$) and plugging this  $\hat \lambda$ in the previous approximations. While doing so, it is however crucial to keep a lower bound $1 - \alpha$ on the coverage-probability. 
The present section aims at describing a possible strategy for addressing this question.

\medskip

From ideas of \citet{lei2014distribution} about full-conformal prediction
region with kernel density estimators, a data-driven choice of $\lambda$
can be made by minimizing the expected Lebesgue measure of the prediction-region. This region, denoted by $\tilde{C}_{\alpha}^{\star}\paren{X_{n+1}}$, is built to keep a lower bound $1 - \alpha$ on its coverage-probability by means of the next three steps.
\begin{enumerate}
	\item The training data set $\mathcal{D}$ is split into two disjoint parts
	$\mathcal{D}_1 \sqcup \mathcal{D}_2 = \mathcal{D}$,
	with $J_1$ and $J_2$ being the corresponding lists of indices such that
	$J_1 \sqcup J_2 = \brac{1, \ldots, n}$.

	\item The regularization parameter $\widehat{\lambda}\paren{D_1}$
	is chosen from the data set $D_1$ among the candidate values in
	$\Lambda := \brac{\lambda_1, \ldots, \lambda_V}$, in the following way
	\begin{align*}
		\widehat{\lambda}\paren{D_1}
		\in
		\argmin_{\lambda \in \Lambda}
		\frac{1}{\mathrm{Card}\paren{J_1}}
		\sum_{j \in J_{1}}
		\leb{
			\brac{
				y \in \mathcal{Y} :
				\widetilde{\pi}_{\lambda; D_1 \setminus \brac{\paren{X_j, Y_j}}}^{\up} \paren{X_j, y} > \alpha
			}
		},
	\end{align*}
	and if the minimizer is not unique then one chooses the largest value
	to break the ties. In other words, the chosen regularization parameter
	is the one minimizing the (empirical) expected Lebesgue measure
	of the upper-approximate full-conformal prediction-regions.

		\item The resulting prediction-region is built from
		$D_2$ with the chosen regularization parameter $\widehat{\lambda}\paren{D_1}$
		\begin{align*}
			\tilde{C}_{\alpha}^{\star}
			\paren{X_{n+1}}
			:= \brac{
				y \in \mathcal{Y} :
				\widetilde{\pi}_{\widehat{\lambda}\paren{D_1}; D_2}^{\up} \paren{X_{n+1}, y} > \alpha
			}.
		\end{align*}

\end{enumerate}
 \citet{yang2025selection} proved that the above strategy (incorporating the data-drive choice of $\lambda$) enjoys the desired coverage level $1-\alpha$ (see Theorem~3 in  \citet{yang2025selection}).
An alternative method for designing a stable random selection algorithm
without any splitting of the initial data set $D$ has been designed \citet{hegazy2025valid}.
However, it requires choosing additional parameter values for the stability and solving a linear optimization problem. Unfortunately their approach is not guaranteed to output a better choice of $\lambda$ since their Theorem~1 in \citet{hegazy2025valid} only proves a lower bound on the coverage-probability at a level lower than $1-\alpha$.

\newpage
\appendix
\section{General case}

This section contains the proofs of the results in Section~\ref{sec.conformal.prediction}.


\subsection{Proof of Lemma~\ref{lm.sandwiching}}
\begin{proof}
\label{proof.sandwiching}
Let us first prove the sandwiching of the full-conformal p-value function
by its lower and upper approximation.
Let $y \in \mathcal{Y}$ and $i \in \brac{1, \ldots, n+1}$.
It follows from Eq.~\eqref{eq.approx.non.conformity.scores} that, on the one hand, almost-surely
\begin{align*}
    S_{D^{y}} \paren{X_i, Y_i} - S_{D^{y}} \paren{X_{N+1}, y}
    &
    =
    \widetilde{S}_{D^{y}} \paren{X_{i}, Y_{i}}
    - \paren{\widetilde{S}_{D^{y}} \paren{X_{i}, Y_{i}} - S_{D^{y}} \paren{X_i, Y_i}}
    \\
    &
    \quad
    - \widetilde{S}_{D^{y}} \paren{X_{n+1}, y}
    + \paren{\widetilde{S}_{D^{y}} \paren{X_{n+1}, y} - S_{D^{y}} \paren{X_{N+1}, y}}
    \\
    &
    \leq
    \widetilde{S}_{D^{y}} \paren{X_{i}, Y_{i}} 
    + \abss{\widetilde{S}_{D^{y}} \paren{X_{i}, Y_{i}} - S_{D^{y}} \paren{X_i, Y_i}}
    \\
    &
    \quad
    - \widetilde{S}_{D^{y}} \paren{X_{n+1}, y}
    + \abss{\widetilde{S}_{D^{y}} \paren{X_{n+1}, y} - S_{D^{y}} \paren{X_{N+1}, y}}
    \\
    &
    \leq
    \widetilde{S}_{D^{y}} \paren{X_{i}, Y_{i}} + \widehat{\tau}_i(y)
    - \widetilde{S}_{D^{y}} \paren{X_{n+1}, y} + \widehat{\tau}_{n+1}(y),
\end{align*}
and on the other hand, almost-surely
\begin{align*}
    S_{D^{y}} \paren{X_i, Y_i} - S_{D^{y}} \paren{X_{N+1}, y}
    &
    \geq
    \widetilde{S}_{D^{y}} \paren{X_{i}, Y_{i}}
    - \abss{\widetilde{S}_{D^{y}} \paren{X_{i}, Y_{i}} - S_{D^{y}} \paren{X_i, Y_i}}
    \\
    &
    \quad
    - \widetilde{S}_{D^{y}} \paren{X_{n+1}, y}
    - \abss{\widetilde{S}_{D^{y}} \paren{X_{n+1}, y} - S_{D^{y}} \paren{X_{N+1}, y}}
    \\
    &
    \geq
    \widetilde{S}_{D^{y}} \paren{X_{i}, Y_{i}} - \widehat{\tau}_i(y)
    - \widetilde{S}_{D^{y}} \paren{X_{n+1}, y} - \widehat{\tau}_{n+1}(y).
\end{align*}
Hence, these two implications follow, almost-surely
\begin{align*}
\begin{aligned}
    &&\widetilde{S}_{D^{y}} \paren{X_{i}, Y_{i}} - \widehat{\tau}_i(y)
    &\geq \widetilde{S}_{D^{y}} \paren{X_{n+1}, y} + \widehat{\tau}_{n+1}(y)
    \\
    \Longrightarrow
    &&S_{D^{y}} \paren{X_i, Y_i} &\geq S_{D^{y}} \paren{X_{N+1}, y}
    \\
    \Longrightarrow
    &&\widetilde{S}_{D^{y}} \paren{X_{i}, Y_{i}} + \widehat{\tau}_i(y)
    &\geq \widetilde{S}_{D^{y}} \paren{X_{n+1}, y} - \widehat{\tau}_{n+1}(y).
\end{aligned}
\end{align*}
By taking the indicator, almost-surely
\begin{align*}
    &\mathbbm{1}\left\{\widetilde{S}_{D^{y}} \paren{X_{i}, Y_{i}} - \widehat{\tau}_i(y)
    \geq \widetilde{S}_{D^{y}} \paren{X_{n+1}, y} + \widehat{\tau}_{n+1}(y)\right\}
    \\
    \leq
    &
    \mathbbm{1}\left\{S_{D^{y}} \paren{X_i, Y_i} \geq S_{D^{y}} \paren{X_{N+1}, y}\right\},
    \\
    \leq
    &
    \mathbbm{1}\left\{\widetilde{S}_{D^{y}} \paren{X_{i}, Y_{i}} + \widehat{\tau}_i(y)
    \geq \widetilde{S}_{D^{y}} \paren{X_{n+1}, y} - \widehat{\tau}_{n+1}(y)\right\}.
\end{align*}
Finally, the sandwiching of the full-conformal p-value function follows
from  summing over $i \in \brac{1, \ldots, n+1}$, 
adding one, and dividing by $n+1$, almost-surely
\begin{align*}
    0 \leq \lfcpv{}{y} \leq \fcpv{y} \leq \ufcpv{}{y} \leq 1.
\end{align*}

Let us now prove the sandwiching of the full-conformal prediction-region
by its lower and upper approximation.
Let $\alpha \in (0,1)$ be a control-level 
and $y \in \mathcal{Y}$. Using the above mention sandwiching
for the conformal p-value functions, one has the following implications, almost-surely
\begin{align*}
    \widetilde{\pi}^{\mathrm{low}}_D(X_{n+1}, y) > \alpha
    \Rightarrow
    \hat{\pi}_D(X_{n+1}, y) > \alpha
    \Rightarrow
    \widetilde{\pi}_D(X_{n+1}, y) > \alpha,
\end{align*}
and by definition of the prediction-regions, almost-surely
\begin{align*}
    \lfcpr{}
    \subseteq
    \fcpr
    \subseteq
    \ufcpr{}.
\end{align*}
Therefore the thickness of the prediction-region is bounded from above, almost-surely
\begin{align*}
    \mathrm{THK}_{\alpha}(X_{n+1}) 
    &=
    \mathcal{V}\left(
        \ufcpr{} \Delta \fcpr
    \right)
    =\mathcal{V}\left(
        \ufcpr{} \setminus \fcpr
    \right)
    \\
    &
    \leq
    \mathcal{V}\left(
        \ufcpr{} \setminus \lfcpr{}
    \right)
    .
\end{align*}
\end{proof}

\subsection{Proof of Lemma~\ref{lm.coverage.approximate.region}}
\begin{proof}
\label{proof.coverage.approximate.region}    
From Lemma~\ref{lm.sandwiching}, the full-conformal prediction-region 
$\fcpr$ is contained in the approximate prediction-region
$\ufcpr{}$. Therefore, since the full-conformal prediction-region is a
confidence prediction-region by Theorem~\ref{thm.coverage}, then so is 
the approximate prediction-region. In fact, its coverage-probability
is bounded from below as
\begin{align*}
    \mathbb{P}
    \croch{
        Y_{n+1} \in \ufcpr{}
    }
    \geq
    \mathbb{P}
    \croch{
        Y_{n+1} \in \fcpr
    }
    \geq
    1 - \alpha.
\end{align*}
\end{proof}

\subsection{Proof of Lemma~\ref{lm.confidence.gap}}
\begin{lemma}[Confidence gap]
\label{lm.confidence.gap}
Assume all assumptions in Theorem~\ref{thm.coverage} hold true.
Assume further that there exists a constant $c>0$ such that 
the conditional density $p(\bullet|D, X_{n+1}) : \mathcal{Y} \to \mathbb{R}_+$, 
is bounded from above by $c$.
Then, The coverage-probability of the approximate full-conformal prediction-region $\ufcpr{}$ is bounded from above by
\begin{align*}
    \mathbb{P}\croch{
        Y_{n+1}
        \in
        \ufcpr{}
    } \leq 1 - \alpha + \frac{1}{n+1} + c \mathbb{E}_{D, X_{n+1}}\croch{\thicc{}}.
\end{align*}
\end{lemma}
\begin{proof}
\label{proof.confidence.gap}
Let us first notice that
\begin{align*}
    \mathbb{P} \croch{
        Y_{n+1} \in \ufcpr{}
    }
    = \mathbb{P} \croch{
        Y_{n+1} \in \fcpr
    } + \mathbb{P} \croch{
        Y_{n+1} \in \ufcpr{} \setminus \fcpr
    }.
\end{align*}
Furthermore, $p(\bullet | X_{n+1}, D) \leq c$, it follows that
\begin{align*}
    &
    \mathbb{P} \croch{
        Y_{n+1} \in \ufcpr{} \setminus \fcpr
    }
    \\
    = &\mathbb{E}_{D, \paren{X_{n+1}, Y_{n+1}}} \croch{
        \mathbbm{1}
        \brac{Y_{n+1} \in \ufcpr{} \setminus \fcpr}
    }
    \\
    \leq
    &
    \mathbb{E}_{D, X_{n+1}} \croch{
        \mathbb{E}_{Y_{n+1}}
        \croch{
            \mathbbm{1}
            \brac{Y_{n+1} \in \ufcpr{} \setminus \fcpr}
            \mid D, X_{n+1}
        }
    }
    \\
    \leq
    &
    \mathbb{E}_{D, X_{n+1}} \croch{
        \int_{y \in \mathcal{Y}}
        \mathbbm{1}
        \brac{y \in \ufcpr{} \setminus \fcpr}
        p(y \mid D, X_{n+1}) dy
    }
    \\
    \leq
    &
    c
    \mathbb{E}_{D, X_{n+1}} \croch{
        \thicc{}
    }.
\end{align*}
Combined with  the second inequality in Theorem~\ref{thm.coverage}
\begin{align*}
    \mathbb{P} \croch{
        Y_{n+1} \in \ufcpr{}
    }
    \leq 1 - \alpha + \frac{1}{n+1}
    + c
    \mathbb{E}_{D, X_{n+1}} \croch{
        \thicc{}
    }.
\end{align*}
\end{proof}

\section{Concerning the case of non-smooth loss-functions}

This section contains the proof of the results in Section~\ref{sec.non.smooth}.


%
%
%
{%
\subsection{Preliminary properties}
Let us first state a minimal assumption.
\begin{assumption} For every $y \in \mathcal{Y}$,
    \begin{equation}
        \label{asm.loss.lsc}
        \tag{LscL}
            \mbox{$u \in \mathcal{Y} \mapsto  \ell\paren{y, u}$ is lower-semi continuous.}    
    \end{equation}
\end{assumption}
This assumption is standard in the optimization literature
in order to proove existence of solutions (see \citet{alexanderianOptimizationInfinitedimensionalHilbert2019}).
It is omitted from the main part of the paper since
it is implied by stronger assumptions such as
Lipschitz continuity and differentiability in the main results.

\begin{lemma}
    \label{lm.wer.strongly.convex.1}
    Assume \eqref{asm.loss.convex} holds true.
    For every $v \in \mathbb{R}^{n+2}_+$,
    the weighted regularized empirical-risk $\wer{\lambda}{v}{\bullet}$ given by Eq.~\eqref{eq.Emp.Risk.weighted.rgularized} is $2 \lambda$-strongly convex.
\end{lemma}
\begin{proof}
    Let $v \in \mathbb{R}^{n+2}_+$,
    for every $i \in \brac{1, \ldots, n+2}$, $v_i \geq 0$.
    It follows from \eqref{asm.loss.convex}
    that for every $i \in \brac{1, \ldots, n+2}$,
    and every $u \in \mathcal{Y}$,
    the function
    \begin{align*}
        f \in \mathcal{H} \mapsto v_i \ell(u, f(X_i)) = v_i  \ell\paren{u, \doth{f}{K_{X_i}}}
    \end{align*}
    is convex since it is a composition of a linear form and a convex function.
    Since $f \in \mathcal{H} \mapsto \lambda \normh{f}^2$
    is $2\lambda$-strongly convex, then $\wer{\lambda}{v}{\bullet}$ is $2\lambda$-strongly convex.
\end{proof}

\begin{lemma}
    \label{lm.coercive}
    Let $\mathcal{G}$ denote a Hilbert space with its corresponding norm $\norm{\bullet}_{\mathcal{G}}$,
    and $\lambda \in \paren{0, +\infty}$.
    For every function $\mathbf{R}\paren{\bullet} : \mathcal{G} \to \mathbb{R}$,
    if it is proper and $2\lambda$-strongly convex and lower semi-continuous,
    then $\mathbf{R}\paren{\bullet}$ is coercive.
\end{lemma}
\begin{proof}
This proof is an adaptation of the one given by \citet{ZhangCoerciveStrongConvex}.
Let us first define the function $\widetilde{\mathbf{R}}\paren{\bullet} : \mathcal{G} \to \mathbb{R}$
as, for every $f \in \mathcal{G}$, $\widetilde{\mathbf{R}}\paren{f} := \mathbf{R}\paren{f} - \mathbf{R}\paren{0}$.
Then, $\widetilde{\mathbf{R}}\paren{\bullet}$ is also $2\lambda$-strongly convex and lower semi-continuous
and $\widetilde{\mathbf{R}}\paren{0} = 0$.
Moreover, for every $f, g \in \mathcal{G}$ and every $t \in \croch{0, 1}$,
\begin{align*}
    \widetilde{\mathbf{R}}\paren{
        \paren{1 - t}f + tg
    } \leq \paren{1 - t}\widetilde{\mathbf{R}}\paren{f}
    + t\widetilde{\mathbf{R}}\paren{g} - t\paren{1 - t}\lambda \norm{f - g}_{\mathcal{G}}^2.
\end{align*}
In particular for $g = 0$ and for $t = \frac{1}{2}$, this implies that
\begin{align*}
    \widetilde{\mathbf{R}}\paren{f}
    \geq 2\widetilde{\mathbf{R}}\paren{\frac{f}{2}}
    + \frac{\lambda}{2}\norm{f}_{\mathcal{G}}^2.
\end{align*}
Let assume that $\norm{f}_{\mathcal{G}} \geq 2$,
then by convexity of $\widetilde{\mathbf{R}}\paren{\bullet}$ and since $\frac{2}{\norm{f}_{\mathcal{G}}} \in \croch{0, 1}$,
\begin{align*}
    \norm{f} \widetilde{\mathbf{R}}\paren{\frac{f}{\norm{f}_{\mathcal{G}}}}
    = \norm{f} \widetilde{\mathbf{R}}\paren{\frac{f}{2} \frac{2}{\norm{f}_{\mathcal{G}}}}
    \leq \norm{f} \frac{2}{\norm{f}_{\mathcal{G}}} \widetilde{\mathbf{R}}\paren{\frac{f}{2}}
    = 2\widetilde{\mathbf{R}}\paren{\frac{f}{2}}.
\end{align*}
Conjoined with the above inequality, this implies that
\begin{align*}
    \widetilde{\mathbf{R}}\paren{f}
    \geq \norm{f} \widetilde{\mathbf{R}}\paren{\frac{f}{\norm{f}_{\mathcal{G}}}}
    + \frac{\lambda}{2}\norm{f}_{\mathcal{G}}^2
    \geq \norm{f} \inf_{\norm{g}_{\mathcal{H}} \leq 1}\widetilde{\mathbf{R}}\paren{g}
    + \frac{\lambda}{2}\norm{f}_{\mathcal{G}}^2.
\end{align*}
Since $\widetilde{\mathbf{R}}\paren{\bullet}$ is lower semi-continuous and proper,
then $\inf_{\norm{g}_{\mathcal{H}} \leq 1}\widetilde{\mathbf{R}}\paren{g} > - \infty$,
which, conjoined the above inequality, implies that $\widetilde{\mathbf{R}}\paren{\bullet}$
is coercive.
Therefore, $\mathbf{R}\paren{\bullet}$ is coercive.
\end{proof}

\medskip
\begin{definition}{(Weighted regularized empirical-risk of a candidate vector)}
\label{def.weighted.predictor_vector}
Let $z,y \in \mathcal{Y}$ be fixed.
Then for every $\lambda\geq 0$, let us define the function $\vwer{\lambda}{\bullet}{\bullet} : \mathbb{R}^{n+2} \times \mathbb{R}^{n+1}$ given by
\begin{align}
\label{eq.vec.wer}
    \paren{v, a} \mapsto \vwer{\lambda}{v}{a} 
    &:=
    \wer{\lambda}{v}{\sum_{i=1}^{n+1} a_i K_{X_i}}.
    \\
    & 
    =
    \sum_{i=1}^{n} \brac{\frac{v_i}{n+1}\ell \paren{Y_i, a^T K_{\bullet, i}}}
    +\frac{v_{n+1}}{n+1} \ell \paren{z, a^T K_{\bullet, n+1}}
    +\lambda a^T K a \notag
    \\
    &\quad
    +\frac{v_{n+2}}{n+1} \ell \paren{y, a^T K_{\bullet, n+1}} ,\notag
\end{align}
where $\wer{\lambda}{v}{\bullet}$ is defined by Eq.~\eqref{eq.Emp.Risk.weighted.rgularized}
and $K$ is the Gram matrix.
\end{definition}
\begin{notation}
\label{not.span}
Let the sub-space $\mathcal{A} \subseteq \mathcal{H}$ be the linear span of
the functions $K_{X_{1}}, \ldots, K_{X_{n+1}}$, defined as $
    \mathcal{A}~:=~\brac{
        \sum_{i = 1}^{n+1} a_i  K_{X_{i}} : a \in \mathbb{R}^{n+1}
}$. 
Let $\mathcal{R}(K) \subset \mathbb{R}^n$ denote the range of the Gram matrix $K$.
\end{notation}

\medskip
\begin{lemma}
	\label{lm.predictor.func.well.defined.1}
	Assume \eqref{asm.loss.convex} and \eqref{asm.loss.lsc} hold true.
    For every $v \in \mathbb{R}^{n+2}_+$, the predictor $\wpred{\lambda}{v}{\bullet}$ exists and is unique
    and is an element of $\mathcal{A}$.
    Moreover, there exists a unique vector
	$\hat{a}_{\lambda} \paren{v} \in \mathcal{R}(K)$ such that
	$\wpred{\lambda}{v}{\bullet} = \sum_{i=1}^{n+1} \paren{\hat{a}_{\lambda} \paren{v}}_iK_{X_i}$.
\end{lemma}
\begin{proof}
    Let $v  \in \mathbb{R}^{n+2}_+$.

    \notparagraph{Existence and uniqueness}
	By Lemma~\ref{lm.wer.strongly.convex.1},
	under \eqref{asm.loss.convex},
	$\wer{\lambda}{v}{\bullet}$ is $2\lambda$-strongly convex.

    Let $\paren{x, y} \in D$.
    By the reproducing property of the hypothesis space $\mathcal{H}$,
    for every $x \in \mathcal{X}$,
    the evaluation linear form
    $L_x : \mathcal{H} \to \mathcal{Y},$
    $f \mapsto L_x\paren{f} = f\paren{x}$
    is a bounded, thus continuous.
    It follows that under \eqref{asm.loss.lsc},
    the following function
    \begin{align*}
        f \in \mathcal{H} \mapsto \ell\paren{y, f\paren{x}}
        = \croch{\ell\paren{y, \bullet}}\paren{f\paren{x}}
        = \croch{\ell\paren{y, \bullet} \circ L_x} \paren{f} \in \mathbb{R}
    \end{align*}
    is a lower semi-continuous form,
    since it is a composition of
    a lower semi-continuous function and
    a continuous form.
    Therefore,
    the weighted regularized empirical-risk function
    is a lower semi-continuous since the empirical-risk term
    is a linear combination of lower semi-continuous
    forms, and the regularization term is the norm squared
    which is itself continuous.

    The above two properties, conjoined with \eqref{asm.loss.convex}
    and Lemma~\ref{lm.coercive} imply that $\wer{\lambda}{v}{\bullet}$ is coercive.

	Therefore, by Corollary 5.6 of \citet{alexanderianOptimizationInfinitedimensionalHilbert2019},
    the weighted regularized empirical-risk function admits
    a unique minimizer in $\mathcal{H}$.
    Let us note said minimizer $\wpred{\lambda}{v}{\bullet}$.

{

\bigskip    
\notparagraph{Representor}
Since $\mathcal{H}$ is an Hilbert space
	and $\mathcal{A}$ is the span of a finite number vector of $\mathcal{H}$,
	$\mathcal{H} = \mathcal{A} \oplus  \mathcal{A}^{\perp}$.
	It follows that there exist a unique functions
	$f_{\mathcal{A}} \in \mathcal{A}$ and
	$f_{\mathcal{A}^{\perp}} \in \mathcal{A}^{\perp}$
	such that $\wpred{\lambda}{v}{\bullet}
	= f_{\mathcal{A}} + f_{\mathcal{A}^{\perp}}$.
	Since $f_{\mathcal{A}^{\perp}} \perp K_{X_i}$,
	for every $i \in \brac{1, \ldots, n+1}$,
	\begin{align*}
		\wpred{\lambda}{v}{X_i} 
		=
		\doth{\wpred{\lambda}{v}{\bullet}}{K_{X_i}}
		=
		\doth{f_{\mathcal{A}} + f_{\mathcal{A}^{\perp}}}{K_{X_i}}
		= 
		\doth{f_{\mathcal{A}}}{K_{X_i}} 
		= f_{\mathcal{A}}\paren{X_i}. 
	\end{align*}
	It follows that
	\begin{align*}
		\wer{\lambda}{v}{\wpred{\lambda}{v}{\bullet}}
		&= \wer{0}{v}{\wpred{\lambda}{v}{\bullet}}
		+ \|\wpred{\lambda}{v}{\bullet}\|_{\mathcal{H}}
		\\
		&=
		\wer{0}{v}{f_{\mathcal{A}}}
		+ \|f_{\mathcal{A}} 
		+ f_{\mathcal{A}^{\perp}}\|_{\mathcal{H}}
		= \wer{0}{v}{f_{\mathcal{A}}}
		+ \|f_{\mathcal{A}} \|_{\mathcal{H}}
		+ \|f_{\mathcal{A}^{\perp}}\|_{\mathcal{H}}
		\\
		&
		\geq
		\wer{0}{v}{f_{\mathcal{A}}}
		+ \|f_{\mathcal{A}} \|_{\mathcal{H}}
		= \wer{\lambda}{v}{f_{\mathcal{A}}}.
	\end{align*}
	Since $\wpred{\lambda}{v}{\bullet}$ is the unique minimizer
	of $\wer{\lambda}{v}{\bullet}$ then
	$\wpred{\lambda}{v}{\bullet} = f_{\mathcal{A}}$.
	It follows that minimizing $\wer{\lambda}{v}{\bullet}$ over
	$\mathcal{H}$ is equivalent (yields the same result) as minimizing
	it over $\mathcal{A}$. Moreover, since $\wpred{\lambda}{v}{\bullet}$ is unique,
	for every minimizers $\tilde a_{\lambda} \paren{v}$
	of $\vwer{\lambda}{v}{\bullet}$ over $\mathbb{R}^{n+2}$,
	$\wpred{\lambda}{v}{\bullet} = \sum_{i=1}^{n+1} \paren{\tilde a_{\lambda} \paren{v}}_i K_{X_i}$.
	Let us note $\hat{a}_{\lambda} \paren{v}$ the projection $\tilde a_{\lambda} \paren{v}$ over the range of the matrix $K$,
	\begin{align*}
		\vwer{\lambda}{v}{\tilde{a}\paren{v}}
		&
		=
		\wer{\lambda}{v}{\wpred{\lambda}{v}{\bullet}}
		= \wer{0}{v}{\wpred{\lambda}{v}{\bullet}}
		+ \normh{\wpred{\lambda}{v}{\bullet}}
		\\
		&
		= \frac{1}{n+1}
		\sum_{i=1}^{n}
		v_i \ell\paren{Y_i, \tilde a_{\lambda} \paren{v}^T K_{\bullet, i}}
		\\
		&
		\quad
		+ \frac{v_{n+1}}{n+1} \ell\paren{z, \tilde a_{\lambda} \paren{v}^T K_{\bullet, n+1}}
		+ \frac{v_{n+2}}{n+1} \ell\paren{y, \tilde a_{\lambda} \paren{v}^T K_{\bullet, n+1}}
		+ \tilde a_{\lambda} \paren{v}^T K \tilde a_{\lambda} \paren{v}
		\\
		&
		= \frac{1}{n+1}
		\sum_{i=1}^{n}
		v_i \ell\paren{Y_i, \hat{a}_{\lambda} \paren{v}^T K_{\bullet, i}}
		\\
		&
		\quad
		+ \frac{v_{n+1}}{n+1} \ell\paren{z, \hat{a}_{\lambda} \paren{v}^T K_{\bullet, n+1}}
		+ \frac{v_{n+2}}{n+1} \ell\paren{y, \hat{a}_{\lambda} \paren{v}^T K_{\bullet, n+1}}
		+ \hat{a}_{\lambda} \paren{v}^T K \hat{a}_{\lambda} \paren{v}
		\\
		& = \vwer{\lambda}{v}{\hat{a}_{\lambda} \paren{v}}.
	\end{align*}
	This means the range of $K$ contains minimizers of  $\vwer{\lambda}{v}{\bullet}$.
    Since all the coordinates of the weight vector are non-negative,
    then by Eq.~\eqref{eq.vec.wer} and \eqref{asm.loss.convex},
    $\vwer{\lambda}{v}{\bullet}$
	is strongly convex over the range of $K$.
    Therefore, coupled with \eqref{asm.loss.lsc},
    there exists a unique vector
	$\hat{a}_{\lambda} \paren{v} \in \mathcal{R}(K)$ such that
	$\wpred{\lambda}{v}{\bullet} = \sum_{i=1}^{n+1} \paren{\hat{a}_{\lambda} \paren{v}}_iK_{X_i}$.
}
\end{proof}
}

\subsection{Proof of Theorem~\ref{thm.score.bound.non.smooth}}
\begin{proof}
\label{proof.score.bound.non.smooth}
The following proof is an adaptation of Lemmas~20 and~22 from \citet{bousquet2002stability}.

\bigskip
\noindent{\emph{Approximation of the predictor.}}
Let $z , y\in \mathcal{Y}$ be fixed.
Let us first provide an upper-bound for the error of approximation of the predictor
$\pred{\lambda;}{y}(\bullet)$ by
$\tilde{f}_{\lambda; D^y}(\bullet) := \pred{\lambda;}{z}(\bullet)$.
This error is defined as
the $\mathcal{H}$-norm of the difference
$\Delta \hat{f}_{\lambda} := \tilde{f}_{\lambda; D^y}(\bullet) - \pred{\lambda;}{y}(\bullet)$
between these two predictors.

To do so, define the weight vectors $\uu, \uw \in \mathbb{R}^{n+2}$ such that
$\uu = (1, \ldots, 1, 1, 0) \in \mathbb{R}^{n+2}$
and $\uw = (1, \ldots, 1, 0, 1) \in \mathbb{R}^{n+2}$. Define the predictor
$\wpred{\lambda}{\uu}{\bullet} \in \mathcal{H}$ as the minimizer of 
the $\lambda$-regularized $\uu$-weighted empirical-risk 
$\wer{\lambda}{\uu}{\bullet}$ formulated in Definition~\ref{def.weighted.predictor} and
$\wpred{\lambda}{\uw}{\bullet} \in \mathcal{H}$ as that of $\wer{\lambda}{\uw}{\bullet}$.
By Lemma~\ref{lm.predictor.func.well.defined.1},
$\wpred{\lambda}{\uu}{\bullet}$ and $\wpred{\lambda}{\uw}{\bullet}$
exist and are unique.

Notice that $\wpred{\lambda}{\uw}{\bullet} = \pred{\lambda;}{y}(\bullet)$
and $\wpred{\lambda}{\uu}{\bullet} = \pred{\lambda;}{z}(\bullet) =: \tilde{f}_{\lambda; D^y}(\bullet)$.
Therefore, $\Delta \hat{f}_{\lambda}$ can rewritten as
$\Delta \hat{f}_{\lambda} = \wpred{\lambda}{\uu}{\bullet} - \wpred{\lambda}{\uw}{\bullet}$.

Under \eqref{asm.loss.convex},
for every $i \in \brac{1, \ldots, n+1}$
and every $u \in \mathcal{Y}$,
the functional
$f \in \mathcal{H} \mapsto \ell(u, \langle f, K_{X_i}\rangle_{\mathcal{H}}) \in \mathbb{R}$
is convex, since a composition of a linear form with a convex function.
Since $\wer{0}{\uu}{\bullet}$
and $\wer{0}{\uw}{\bullet}$ are linear combinations,
with non-negative weights, of these convex functionals,
they are themselves convex.

Let $t \in (0, 1)$.
Since  $\wer{0}{\uu}{\bullet}$ is convex, it follows that
\begin{align*}
    \wer{0}{\uu}{
        \wpred{\lambda}{\uw}{\bullet} + t \Delta \hat{f}_{\lambda}
    }
    -
    \wer{0}{\uu}{
        \wpred{\lambda}{\uw}{\bullet}
    }
    &\leq
    t \left[
        \wer{0}{\uu}{\wpred{\lambda}{\uu}{\bullet}} 
        - \wer{0}{\uu}{\wpred{\lambda}{\uw}{\bullet}}
    \right],
    \\
    \wer{0}{\uu}{\wpred{\lambda}{\uu}{\bullet} - t \Delta \hat{f}_{\lambda}} 
    -
    \wer{0}{\uu}{\wpred{\lambda}{\uu}{\bullet}}
    &\leq
    t \left[
        \wer{0}{\uu}{\wpred{\lambda}{\uw}{\bullet}}
        - \wer{0}{\uu}{\wpred{\lambda}{\uu}{\bullet}}
    \right].
\end{align*}
Taking their sum
\begin{align*}
    \wer{0}{\uu}{
        \wpred{\lambda}{\uw}{\bullet} + t \Delta \hat{f}_{\lambda}
    }
    -
    \wer{0}{\uu}{
        \wpred{\lambda}{\uw}{\bullet}
    }
    +
    \wer{0}{\uu}{
        \wpred{\lambda}{\uu}{\bullet} - t \Delta \hat{f}_{\lambda}
    }
    -
    \wer{0}{\uu}{\wpred{\lambda}{\uu}{\bullet}}
    \leq 0.
\end{align*}
Since  $\wpred{\lambda}{\uu}{\bullet}$ and $\wpred{\lambda}{\uw}{\bullet}$ are
the minimizers of
$\wer{\lambda}{\uu}{\bullet}$ and $\wer{\lambda}{\uw}{\bullet}$ respectively,
it follows that
\begin{align*}
    \wer{\lambda}{\uu}{\wpred{\lambda}{\uu}{\bullet}}
    - \wer{\lambda}{\uu}{\wpred{\lambda}{\uu}{\bullet} - t \Delta \hat{f}_{\lambda}}
    \leq
    & 0,
    \\
    \wer{\lambda}{\uw}{\wpred{\lambda}{\uw}{\bullet}}
    - \wer{\lambda}{\uw}{\wpred{\lambda}{\uw}{\bullet} + t \Delta \hat{f}_{\lambda}}
    \leq
    & 0.
\end{align*}
By summing up the three previous inequalities, it follows that
\begin{align*}
    &\frac{1}{n+1}
    \ell\left(y, \hat{f}_{\lambda}(\uw; X_{n+1})\right)
    -
    \frac{1}{n+1}
    \ell\left(y, \hat{f}_{\lambda}(\uw; X_{n+1}) + t \Delta \hat{f}_{\lambda}(X_{n+1})\right)
    \\
    -&
    \frac{1}{n+1}
    \ell\left(z, \hat{f}_{\lambda}(\uw; X_{n+1})\right)
    +
    \frac{1}{n+1}
    \ell\left(z, \hat{f}_{\lambda}(\uw; X_{n+1}) + t \Delta \hat{f}_{\lambda}(X_{n+1})\right)
    \\
    +& \lambda 
    \left(
        \left\|\wpred{\lambda}{\uu}{\bullet}\right\|^2_{\mathcal{H}}
        - \left\|\wpred{\lambda}{\uu}{\bullet} - t \Delta \hat{f}_{\lambda}\right\|^2_{\mathcal{H}}
        + \left\|\wpred{\lambda}{\uw}{\bullet}\right\|^2_{\mathcal{H}}
        - \left\|\wpred{\lambda}{\uw}{\bullet} + t \Delta \hat{f}_{\lambda}\right\|^2_{\mathcal{H}}
    \right)
    \leq 0.
\end{align*}
Under \eqref{asm.loss.lipschitz}, 
all the terms involving the loss-function $\ell\paren{\bullet, \bullet}$
are bounded from above. By putting all such terms on the right of the inequality and using say bound,
it follows that
\begin{align*}
    \left\|\wpred{\lambda}{\uu}{\bullet}\right\|^2_{\mathcal{H}}
    - \left\|\wpred{\lambda}{\uu}{\bullet} - t \Delta \hat{f}_{\lambda}\right\|^2_{\mathcal{H}}
    + \left\|\wpred{\lambda}{\uw}{\bullet}\right\|^2_{\mathcal{H}}
    - \left\|\wpred{\lambda}{\uw}{\bullet} + t \Delta \hat{f}_{\lambda}\right\|^2_{\mathcal{H}}
    \leq t |\Delta \hat{f}_{\lambda}(X_{n+1})| \frac{2 \rho}{\lambda (n+1)}.
\end{align*}
Expanding the terms in the left hand side and dividing both sides by $t$,
it follows that
\begin{align*}
    2(1 - t)\left\|\Delta \hat{f}_{\lambda}\right\|^2_{\mathcal{H}}
    \leq |\Delta \hat{f}_{\lambda}(X_{n+1})| \frac{2 \rho}{\lambda (n+1)}.
\end{align*}
Taking the sup of $(1 - t)$ for $t \in (0,1)$ on the left hand side
and applying Cauchy Schwarz on the right-hand side
\begin{align*}
    2 \left\|\Delta \hat{f}_{\lambda} \right\|^2_{\mathcal{H}}
    \leq
    &
    \abss{\Delta \hat{f}_{\lambda}(X_{n+1})}
    \frac{2 \rho}{\lambda (n+1)}
    \leq
    \normh{\Delta \hat{f}_{\lambda}}
    \normh{K_{X_{n+1}}}
    \frac{2 \rho}{\lambda (n+1)}.
    \\
    \leq
    &
    \normh{\Delta \hat{f}_{\lambda}}
    \sqrt{ K_{n+1, n+1 }} \frac{2 \rho}{\lambda (n+1)}.
\end{align*}
In the case where $\wpred{\lambda}{\uw}{\bullet} = \wpred{\lambda}{\uu}{\bullet}$,
then $\left\|\Delta \hat{f}_{\lambda} \right\|^2_{\mathcal{H}} = 0$.
Otherwise the error is non zero and dividing the inequality above by
$\left\|\Delta \hat{f}_{\lambda} \right\|^2_{\mathcal{H}}$,
the predictor approximation quality is bounded from above
\begin{align*}
    \left\|\Delta \hat{f}_{\lambda} \right\|_{\mathcal{H}}
    \leq
    \sqrt{ K_{n+1, n+1}} \frac{\rho}{\lambda (n+1)}
    \quad\mbox{a.s.}
\end{align*}

\bigskip{}
\noindent{\emph{Approximation of \emph{non-conformity scores}.}}
For every $i \in \brac{1, \ldots, n}$,
from the previous bound and under \eqref{asm.score.lipschitz},
the non-conformity score approximation quality is bounded from above
\begin{align*}
    \left|
        S_{\lambda;D^{y}} \paren{X_i, Y_i}
        - \widetilde{S}_{\lambda;D^{y}} \paren{X_{i}, Y_{i}}
    \right|
    &
    =
    \left| 
        s\paren{Y_i, \pred{\lambda;}{y}(X_i)}
        - s\paren{Y_i, \pred{\lambda;}{z}(X_i)}
    \right|
    \\
    &
    \leq
    \gamma
    \left|
        \pred{\lambda;}{y}(X_i)
        - \pred{\lambda;}{z}(X_i)
    \right|
    \leq
    \gamma
    \abss{
        \doth{\pred{\lambda;}{y} - \pred{\lambda;}{z}}{K_{X_i}}
    }
    \\
    &
    \leq
    \gamma
    \normh{K_{X_i}}
    \normh{\pred{\lambda;}{y} - \pred{\lambda;}{z}}
    \leq
    \gamma
    \sqrt{ K_{i, i} }
    \normh{\pred{\lambda;}{y} - \pred{\lambda;}{z}}
    \\
    &
    \leq
    \gamma
    \sqrt{ K_{i, i }}
    \normh{\Delta \hat{f}_{\lambda}}
    \leq
    \sqrt{ K_{i, i }}
    \sqrt{ K_{n+1, n+1}}
    \frac{\gamma \rho}{\lambda (n+1)} 
    \quad\mbox{a.s.}
\end{align*}
Likewise, for $i = n+1$    
\begin{align*}
    \left|
        S_{\lambda;D^{y}} \paren{X_i, y}
        - \widetilde{S}_{\lambda;D^{y}} \paren{X_{i}, y}
    \right|
    \leq
    \sqrt{ K_{i, i}}
    \sqrt{ K_{n+1, n+1}}
    \frac{\ \rho}{\lambda (n+1)}
    \quad\mbox{a.s.}
\end{align*}
\end{proof}

\subsection{Concerning the empirical upper-bound on the thickness}
\begin{lemma}
\label{lm.some.area}
For every $0 \leq a < \infty$ and $0 \leq \tau < \infty$
\begin{align*}
    \leb{
        y \in \mathcal{Y}:
        \abss{\abss{y} - a} < \tau
    } \leq 4 \tau
\end{align*}
\end{lemma}
\begin{proof}
Let $0 \leq a < \infty$ and $0 \leq \tau < \infty$
\begin{align*}
    \leb{
        y \in \mathcal{Y}:
        \abss{\abss{y} - a} < \tau
    }
    &= \int_{-\infty}^{+\infty}
    \mathbbm{1}
    \brac{
        \abss{\abss{y} - a} < \tau
    } dy
    \\
    &
    = \int_{0}^{+\infty}
    \mathbbm{1}
    \brac{
        \abss{y - a} < \tau
    } dy
    +
    \int_{-\infty}^{0}
    \mathbbm{1}
    \brac{
        \abss{-y - a} < \tau
    } dy
    \\
    &
    = \int_{0}^{a}
    \mathbbm{1}
    \brac{
        a - y < \tau
    } dy
    +
    \int_{a}^{+\infty}
    \mathbbm{1}
    \brac{
        y - a < \tau
    } dy
    \\
    &
    \quad
    +
    \int_{-\infty}^{-a}
    \mathbbm{1}
    \brac{
        -y - a < \tau
    } dy
    +
    \int_{-a}^{0}
    \mathbbm{1}
    \brac{
        y + a < \tau
    } dy
    \\
    &
    = \int_{0}^{a}
    \mathbbm{1}
    \brac{
        a - \tau < y
    } dy
    +
    \int_{a}^{+\infty}
    \mathbbm{1}
    \brac{
        y  < \tau + a
    } dy
    \\
    &
    \quad
    +
    \int_{-\infty}^{-a}
    \mathbbm{1}
    \brac{
        -\tau - a < y
    } dy
    +
    \int_{-a}^{0}
    \mathbbm{1}
    \brac{
        y  < \tau - a
    } dy
    \\
    &
    \leq \int_{a - \tau}^{a}
    dy
    +
    \int_{a}^{\tau + a}
    dy
    +
    \int_{-\tau - a}^{-a}
    dy
    +
    \int_{-a}^{\tau - a}
    dy
    \leq 4 \tau.
\end{align*}
\end{proof}
\begin{lemma}[First bound on thickness]
\label{lm.bound.conf.region.gap}
    Assume that there exists an almost-surely finite random variable
    $\Tau$ such that, for every $y \in \mathcal{Y}$ and every $i \in \brac{1, \ldots, n+1}$,
    $\widehat{\tau}_i(y) \leq \Tau$ almost-surely.    
Let us further assume that there exist $y$-independent random variables
    $\tilde{Y}_1, \ldots, \tilde{Y}_{n+1}$ such that,
    for every $y \in \mathcal{Y}$ and
    every $i \in \brac{1, \ldots, n+1}$, $\apred{}{y}{}\paren{X_i} = \tilde{Y_i}$ almost-surely.

    Then it follows from Lemma~\ref{lm.sandwiching} that the thickness is bounded from above as
    \begin{align}
    \label{eq.confidence.region.gap.bound.1}
        \thicc{} \leq
            \mathcal{V}\left(
                \ufcpr{} 
                \setminus
                \lfcpr{}
            \right) \leq 8 \Tau
            \quad\mbox{a.s.}
    \end{align}
\end{lemma}
\begin{proof}
\label{proof.bound.conf.region.gap}
For every $i \in \brac{1, \ldots, n}$  and $y \in \mathcal{Y}$, almost-surely
\begin{align*}
\begin{aligned}
    &&\widetilde{S}_{D^{y}} \paren{X_i, Y_i}
    + \widehat{\tau}_{i}(y)
    &\geq
    \widetilde{S}_{D^{y}} \paren{X_{n+1}, y}
    - \widehat{\tau}_{n+1}(y)
    \\
    \implies
    &&
    \widetilde{S}_{D^{y}} \paren{X_i, Y_i}
    + \Tau
    &\geq
    \widetilde{S}_{D^{y}} \paren{X_{n+1}, y}
    - \Tau
    \\
    \implies
    &&
    \abss{
        Y_i - \apred{}{y}{}\paren{X_i}
    }
    + \Tau
    &\geq
    \abss{
        y - \apred{}{y}{}\paren{X_{n+1}}
    }
    - \Tau
    \\
    \implies
    &&
    \abss{
        Y_i - \tilde{Y}_i
    }
    + \Tau
    &\geq
    \abss{
        y - \tilde Y_{n+1}
    }
    - \Tau
    \\
    \implies
    &&
    \abss{
        Y_i - \tilde{Y}_i
    }
    &\geq
    \abss{
        y - \tilde Y_{n+1}
    }- 2 \Tau
\end{aligned}
\end{align*}
By taking the indicator, summing over $i$, adding one and dividing by $n+1$,
the upper approximation p-value function is bounded from above, for every $y \in \mathcal{Y}$,
almost-surely
\begin{align*}
    \ufcpv{}{y} \leq \frac{1
        + \sum_{i=1}^{n}
        \mathbbm{1}
        \brac{
            \abss{
            Y_i - \tilde{Y}_i
            }
            \geq
            \abss{
                y - \tilde Y_{n+1}
            }
            - 2\Tau
        }
    }{n+1}.
\end{align*}
It follows that the upper-approximate prediction-region is contained in the following way,
almost-surely
\begin{align*}
    \ufcpr{}
    &:=
    \brac{
        y \in \mathcal{Y} : 
        \ufcpv{}{y}
        > \alpha
    }
    \\
    &\subseteq
    \brac{
        y \in \mathcal{Y} : 
        \frac{1
            + \sum_{i=1}^{n}
            \mathbbm{1}
            \brac{
                \abss{
                Y_i - \tilde{Y}_i
                }
                \geq
                \abss{
                    y - \tilde Y_{n+1}
                }
                - 2\Tau
            }
        }{n+1} > \alpha
    }
    \\
    &\subseteq
    \brac{
        y \in \mathcal{Y} : 
        \frac{1}{n}
        \sum_{i=1}^{n}
            \mathbbm{1}
            \brac{
                \abss{
                Y_i - \tilde{Y}_i
                }
                <
                \abss{
                    y - \tilde Y_{n+1}
                }
                - 2\Tau
            }
        < \frac{n+1}{n} \paren{1 - \alpha}
    }.
\end{align*}
Let us note $\abss{Y_{(1)} - \tilde{Y}_{(1)}} \leq \ldots \leq \abss{Y_{(n)} - \tilde{Y}_{(n)}}$
be the values among $\abss{Y_{1} - \tilde{Y}_{1}}, \ldots, \abss{Y_{n} - \tilde{Y}_{n}}$
sorted in increasing order, and the index $i_{n, \alpha} = \ceil{
        (n+1)\paren{1-\alpha}
}$. It follows that, almost-surely
\begin{align*}
    \ufcpr{}
    &
    \subseteq
    \brac{
        y \in \mathcal{Y} :
        \abss{
            y - \tilde Y_{n+1}
        }
        - 2\Tau
        \leq \abss{Y_{(i_{n, \alpha})} - \tilde{Y}_{(i_{n, \alpha})}}
    },
    \\
    &
    \subseteq
    \brac{
        y \in \mathcal{Y} :
        \abss{
            y - \tilde Y_{n+1}
        }
        \leq \abss{Y_{(i_{n, \alpha})} - \tilde{Y}_{(i_{n, \alpha})}}
        + 2\Tau
    }.
\end{align*}
Following the same logic, 
the lower-approximate prediction-region contains the following region, almost-surely
\begin{align*}
    \lfcpr{}
    \supseteq
    \brac{
        y \in \mathcal{Y} :
        \abss{
            y - \tilde Y_{n+1}
        }
        \leq \abss{Y_{(i_{n, \alpha})} - \tilde{Y}_{(i_{n, \alpha})}}
        - 2\Tau
    }.
\end{align*}
Using the inequality in Eq.~\eqref{eq.confidence.region.gap.bound},
and two upper and lower (in the sense of inclusion) regions above,
the thickness is almost-surely bounded from above as
\begin{align*}
    \thicc{}
    &\leq
    \leb{\ufcpr{} \setminus \lfcpr{}}
    \\
    &
    \leq
    \leb{
        \brac{
        y \in \mathbb{R} :
        \abss{
            \abss{
                y - \tilde Y_{n+1}
            } - \abss{Y_{(i_{n, \alpha})} - \tilde{Y}_{(i_{n, \alpha})}}
        }
        < 2\Tau
    }
    }
    \\
    &
    \leq
    \leb{
        \brac{
        y \in \mathbb{R} :
        \abss{
            \abss{
                y
            } - \abss{Y_{(i_{n, \alpha})} - \tilde{Y}_{(i_{n, \alpha})}}
        }
        < 2\Tau
    }
    }
    \\
    &
    \leq
    8 \Tau.
\end{align*}
where the third inequality follows from
the translation invariance of the Lebesgue measure, and
the last inequality follows from Lemma~\ref{lm.some.area}.
\end{proof}
\subsection{Proof of Theorem~\ref{thm.non.smooth}}
\begin{proof}
\label{proof.non.smooth}
\noindent{\emph{Coverage.}}
The coverage guarantee holds true by application of the approximation scheme
described in Definition~\ref{def.approx.prediction.region}
to the non-conformity scores approximation defined in Eq.~\eqref{eq.approximate.score.0}
and the upper-bounds given in Theorem~\ref{thm.score.bound.non.smooth}.

\bigskip{}
\noindent{\emph{upper-bound for the thickness.}}
The upper-bound on thickness follows from Eq.~\eqref{eq.confidence.region.gap.bound.1},
by taking for every $i \in \brac{1, \ldots, n+1}$,
$\tilde{Y}_i = \pred{\lambda;}{z}\paren{X_i} \paren{= \apred{\lambda;}{y}{}\paren{X_i}}$.
In this case,
for every $i \in \brac{1, \ldots, n}$,
almost-surely
\begin{align*}
    \widehat{\tau}^{(0)}_{\lambda; i}
    = 
    \sqrt{ K_{i, i}}
    \sqrt{ K_{n+1, n+1}}
    \frac{\rho}{\lambda (n+1)}
    \leq
    \frac{\rho}{\lambda (n+1)}
    \max_{i \in \brac{1, \ldots, n+1}}
    \kappa_{\mathcal{H}} \paren{X_i, X_i}
    =: \widehat{\tau}^{(0)}.
\end{align*}
Therefore, the thickness is bounded from above
\begin{align*}
    \thick{(0)}{\lambda; \alpha}
    \leq
    8 \widehat{\tau}^{(0)}
    \leq
    8 \frac{\rho}{\lambda (n+1)}
    \max_{i \in \brac{1, \ldots, n+1}}
    \kappa_{\mathcal{H}} \paren{X_i, X_i}
    \quad \mbox{a.s.}
\end{align*}

\end{proof}

\subsection{Proof of Corollary~\ref{cor.non.smooth}}
\label{proof.cor.non.smooth}
\begin{proof}
    Under the assumptions mentionned in Theorem~\ref{thm.non.smooth},
    the prediction-region $\tilde{C}_{\lambda; \alpha}^{\mathrm{up, (0)}}\paren{X_{n+1}}$
    is an approximate full-conformal prediction-region.

    Since it is assumed that
    the exists a constant $c \in \paren{0, +\infty}$
    such that the conditional probability density function
    of the output $p\paren{\bullet | D, X_{n+1}} : \mathcal{Y} \mapsto \mathbb{R}_+$
    is bounded from above by the constant $c$, and
    that the non-conformity scores are almost-surely distinct, 
    it follows from Lemma~\ref{lm.confidence.gap} that
    \begin{align*}
        \mathbb{P}\croch{
            Y_{n+1}\in \tilde{C}_{\alpha}^{\mathrm{up, (0)}}\paren{X_{n+1}}
        }
        &
        \leq
        1 - \alpha + c
        \mathbb{E}_{D, X_{n+1}}\croch{
            \mathrm{THK}_{\lambda; \alpha}^{\mathrm{(0)}}
        }
        \\
        &
        \leq
        1 - \alpha
        + \frac{8 c\rho}{\lambda (n+1)}
        \mathbb{E}_{D, X_{n+1}}\croch{
            \max_{i \in \brac{1, \ldots, n+1}}
            \kappa_{\mathcal{H}} \paren{X_i, X_i}
        },
    \end{align*}
    where the second inequality
    follows from the upper-bound
    on the thickness provided in Eq.~\eqref{eq.bound.thickness.non.smooth}
    which holds under the assumptions mentionned in Theorem~\ref{thm.non.smooth}.
\end{proof}

\section{Concerning the case of smooth loss-functions}

This section provides the proofs for the results in Section~\ref{sec.smooth}.



\subsection{Peliminary properties}
\begin{lemma}
    \label{lm.first.diff.rer}
    Assume \eqref{asm.loss.c1} holds true.
    For every $\lambda \in \paren{0, \infty}$, and
    for every $v \in \mathbb{R}^{n+2}$, and
    for every $f \in \mathcal{H}$,
    the Fr{\'e}chet differential
    $\croch{
        \mathcal{D}\hat{\mathbf{R}}_{\lambda}\paren{v ;\bullet}
    }\paren{f} : \mathcal{H} \to \mathbb{R}$ of
    the weighted regularized empirical-risk functional
    $\hat{\mathbf{R}}_{\lambda}\paren{v; \bullet} : \mathcal{H} \mapsto \mathbb{R}$
    is well-defined and, for every $h \in \mathcal{H}$,
    \begin{align*}
        \croch{
            \mathcal{D}\hat{\mathbf{R}}_{\lambda}\paren{v ;\bullet}
        }\paren{f}\paren{h}
        = \dotprod{
            \partial_2\hat{\mathbf{R}}_{\lambda}\paren{v; f}
        }{
            h
        }{\mathcal{H}}
    \end{align*} 
    where the function
    $
    \partial_2\hat{\mathbf{R}}_{\lambda}\paren{v; f} \in \mathcal{H}
    $ is defined as
    \begin{align}
        \label{eq.risk.d1}
        \partial_2\hat{\mathbf{R}}_{\lambda}\paren{v; f}
        &
        :=
        \frac{1}{n+1} \sum_{i=1}^{n}v_i \partial_{2}\ell\paren{Y_i, f\paren{X_i}} K_{X_i}
        \\\notag
        &
        \quad
        +
        \frac{v_{n+1}}{n+1}
        \partial_{2}\ell\paren{z, f\paren{X_{n+1}}} K_{X_{n+1}}
        \\\notag
        &
        \quad
        +
        \frac{v_{n+1}}{n+1}
        \partial_{2}\ell\paren{y, f\paren{X_{n+1}}} K_{X_{n+1}}
        +
        2\lambda f.
    \end{align}
\end{lemma}
\begin{proof}
    Let $\lambda \in\paren{0, +\infty}$, and
    $v \in \mathbb{R}^{n+2}$, and
    $f \in \mathcal{H}$.

    Let $\paren{x, u} \in \mathcal{X} \times \mathcal{Y}$.
    The Fr{\'e}chet differential $\croch{\mathcal{D} L_{x}}\paren{f} : \mathcal{H} \mapsto \mathbb{R}$
    of the functional $L_{x} : \mathcal{H} \mapsto \mathbb{R}$, $f \mapsto f\paren{x}$
    at $f$,
    is $L_{x}$ itself.
    In fact, for every $h \in \mathcal{H}$
    \begin{align*}
        L_{x}\paren{f + h}
        = L_{x}\paren{f} + L_{x}\paren{h},
    \end{align*}
    where the second term on the right hand side is,
    by the hypothesis space $\mathcal{H}$ being an RKHS,
    a bounded linear operator $L_{x}$ applied on $h$.
    One can then compute the Fr{\'e}chet differential of
    the function $M_{x, u} : \mathcal{H} \to \mathcal{H}$,
    $f \mapsto \ell\paren{u, f\paren{x}}$ at $f$
    by applying the chain rule as follows,
    for every $h \in \mathcal{H}$
    \begin{align*}
        \croch{
            \mathcal{D} M_{x, u}
        }(f)(h)
        &
        = 
        \croch{
            \mathcal{D}
            \croch{
                \ell\paren{u, \bullet}
                \circ L_x
            }
        }(f)(h)
        \\
        &
        = 
        \croch{
            \mathcal{D}
            \ell\paren{u, \bullet}
        }\paren{
            L_x\paren{f}
        }\paren{
            \croch{\mathcal{D}L_{x}}\paren{f}\paren{h}
        }
        \\
        &
        = 
        \croch{
            \mathcal{D}
            \ell\paren{u, \bullet}
        }\paren{
            f(x)
        }\paren{
            L_x\paren{h}
        }
        \\
        &
        =
        \partial_2 \ell\paren{u, f\paren{x}}
        L_x\paren{h},
    \end{align*}
    where the second equality follows from the chain rule, and
    the third inequality follows the above definition of the
    Fr{\'e}chet differential of $L_x$ at $f$, and
    the last equality follows from the coincidence
    between the Fr{\'e}chet differential and regular derivative
    of function in $\mathcal{F}\paren{\mathbb{R}, \mathbb{R}}$, which exists by \eqref{asm.loss.c1}.

    The Fr{\'e}chet differential
    $\croch{\mathcal{D}\norm{\bullet}_{\mathcal{H}}^2}(f) : \mathcal{H} \mapsto \mathbb{R}$
    of the norm squared at $f$ is
    $h \mapsto 2 \doth{f}{h}$ which can be represented by $2f \in \mathcal{H}$.
    In fact, for every $h \in \mathcal{H}$
    \begin{align*}
        \normh{f + h}^2
        = \normh{f}^2
        + 2 \doth{f}{h}
        + \normh{h}^2,
    \end{align*}
    where $h \mapsto 2 \doth{f}{h}$ is a linear operator
    which is bounded by Cauchy Schwarz and $f$ having a finite norm on $\mathcal{H}$.

    By linearity of the Fr{\'e}chet differential,
    the Fr{\'e}chet differential
    $\croch{
        \mathcal{D}\hat{\mathbf{R}}_{\lambda}\paren{v ;\bullet}
    }\paren{f} : \mathcal{H} \to \mathbb{R}$ of
    the regularized empirical-risk functional
    $\hat{\mathbf{R}}_{\lambda}\paren{v; \bullet} : \mathcal{H} \mapsto \mathbb{R}$
    is then given by, for every $h \in \mathcal{H}$
    \begin{align*}
        \croch{
            \mathcal{D}\hat{\mathbf{R}}_{\lambda}\paren{v ;\bullet}
        }\paren{f}\paren{h}
        &
        =
        \frac{1}{n+1} \sum_{i=1}^{n}v_i \partial_{2}\ell\paren{Y_i, f\paren{X_i}}
        L_{X_i} \paren{h}
        \\
        &
        \quad
        +
        \frac{v_{n+1}}{n+1}
        \partial_{2}\ell\paren{z, f\paren{X_{n+1}}}
        L_{X_{n+1}} \paren{h}
        \\
        &
        \quad
        +
        \frac{v_{n+2}}{n+1}
        \partial_{2}\ell\paren{y, f\paren{X_{n+1}}} L_{X_{n+1}} \paren{h}
        +
        2\lambda \doth{f}{h},
        \\
        &
        =
        \frac{1}{n+1} \sum_{i=1}^{n}v_i \partial_{2}\ell\paren{Y_i, f\paren{X_i}}
        \dotprod{K_{X_i}}{h}{\mathcal{H}}
        \\
        &
        \quad
        +
        \frac{v_{n+1}}{n+1}
        \partial_{2}\ell\paren{z, f\paren{X_{n+1}}}
        \dotprod{K_{X_{n+1}}}{h}{\mathcal{H}}
        \\
        &
        \quad
        +
        \frac{v_{n+2}}{n+1}
        \partial_{2}\ell\paren{y, f\paren{X_{n+1}}} \dotprod{K_{X_{n+1}}}{h}{\mathcal{H}}
        +
        2\lambda \doth{f}{h},
        \\
        &
        = \dotprod{
            \partial_2\hat{\mathbf{R}}_{\lambda}\paren{v; f}
        }{
            h
        }{\mathcal{H}},
    \end{align*}
    where $\partial_2\hat{\mathbf{R}}_{\lambda}\paren{v; f} \in \mathcal{H}$ is given by Eq.~\eqref{eq.risk.d1}.
\end{proof}

\medskip
{\begin{lemma}
\label{lm.local.predictor.stability}
Assume \eqref{asm.loss.convex} and \eqref{asm.loss.c1} hold true.
Then for every weight vector $v \in \croch{\uu, \uw}$,
\begin{align*}
    \norm{\hat{f}_{\lambda}\paren{v; \bullet} - \pred{\lambda;}{z}}_{\mathcal{H}}
    \leq \sqrt{K_{n+1, n+1}} \frac{\rho^{(1)}_{\lambda}(y)}{
    \lambda (n+1) },
\end{align*}
where $\rho^{(1)}_{\lambda}(y)$ is given by Eq.~\eqref{eq.predictor.approx.bound.smooth}.
\end{lemma}
\begin{proof}
Let $v \in \croch{\uu, \uw}$. By design, $v \in \mathbb{R}_{+}^{n+2}$.
Since $\eqref{asm.loss.c1}$ implies \eqref{asm.loss.lsc},
then conjoined with \eqref{asm.loss.convex} and Lemma~\ref{lm.predictor.func.well.defined.1},
$\hat{f}_{\lambda}\paren{v; \bullet} \in \mathcal{H}$ and $\hat{f}_{\lambda}\paren{\uu; \bullet} \in \mathcal{H}$
exist and are unique minimizers of $\hat{\mathbf{R}}_{\lambda}\paren{v; \bullet}$,
and $\hat{\mathbf{R}}_{\lambda}\paren{\uu; \bullet}$ respectively.
Moreover, by Lemma~\ref{lm.wer.strongly.convex.1}, $\hat{\mathbf{R}}_{\lambda}\paren{v; \bullet}$,
and $\hat{\mathbf{R}}_{\lambda}\paren{\uu; \bullet}$ are $2\lambda$-strongly convex functions.
Under \eqref{asm.loss.c1} and by optimality for $\hat{f}_{\lambda}\paren{v; \bullet}$,
\begin{align*}
    &
    2 \lambda \norm{
       \hat{f}_{\lambda}\paren{v; \bullet}
        - \hat{f}_{\lambda}\paren{\uu; \bullet}
    }_{\mathcal{H}}^2
    \\
    &
    \leq
    \paren{
        \croch{
            \mathcal{D}\hat{\mathbf{R}}_{\lambda}\paren{\bullet, v}
        }\paren{
            \hat{f}_{\lambda}\paren{v; \bullet}
        }
        - \croch{
            \mathcal{D}\hat{\mathbf{R}}_{\lambda}\paren{\bullet, v}
        }\paren{
            \hat{f}_{\lambda}\paren{\uu; \bullet}
        } 
    }
    \paren{
        \hat{f}_{\lambda}\paren{v; \bullet}
        - \hat{f}_{\lambda}\paren{\uu; \bullet}
    }
    \\
    &
    \leq
    \paren{
        \croch{
            \mathcal{D}\hat{\mathbf{R}}_{\lambda}\paren{\bullet, \uu}
        }\paren{
            \hat{f}_{\lambda}\paren{\uu; \bullet}
        }
        - \croch{
            \mathcal{D}\hat{\mathbf{R}}_{\lambda}\paren{\bullet, v}
        }\paren{
            \hat{f}_{\lambda}\paren{\uu; \bullet}
        } 
    }
    \paren{
        \hat{f}_{\lambda}\paren{v; \bullet}
        - \hat{f}_{\lambda}\paren{\uu; \bullet}
    }
    \\
    &
    \leq
    \dotprod{
        \partial_2\hat{\mathbf{R}}_{\lambda}\paren{\uu; \hat{f}_{\lambda}\paren{\uu; \bullet}}
        - \partial_2\hat{\mathbf{R}}_{\lambda}\paren{v; \hat{f}_{\lambda}\paren{\uu; \bullet}}
    }{
        \hat{f}_{\lambda}\paren{v; \bullet}
        - \hat{f}_{\lambda}\paren{\uu; \bullet}
    }{\mathcal{H}}
    \\
    &
    \leq
    \norm{
        \partial_2\hat{\mathbf{R}}_{\lambda}\paren{\uu - v; \hat{f}_{\lambda}\paren{\uu; \bullet}}
    }_{\mathcal{H}}
    \norm{
        \hat{f}_{\lambda}\paren{v; \bullet}
        - \hat{f}_{\lambda}\paren{\uu; \bullet}
    }_{\mathcal{H}},
\end{align*}
where the first inequality follows from
the strong convexity of $\hat{\mathbf{R}}_{\lambda}\paren{v; \bullet}$, and
the second inequality follows from the optimality of $\hat{f}_{\lambda}\paren{\uu; \bullet}$, and
the third inequality follows from Lemma~\ref{lm.first.diff.rer},
and the last inequality follows form the Cauchy-Schwarz inequality
and by linearity of $\partial_2\hat{\mathbf{R}}_{\lambda}\paren{\bullet; f}$ for every $f \in \mathcal{H}$.

Since $v \in \croch{\uu, \uw}$ then there exists $t \in \croch{0, 1}$,
such that, $v = \uu + t\paren{\uw - \uu}$. Then,
$\uu - v = t \paren{\uu - \uw}$
\begin{align*}
    \partial_2\hat{\mathbf{R}}_{\lambda}\paren{\uu - v; \hat{f}_{\lambda}\paren{\uu; \bullet}}
    &
    = \partial_2\hat{\mathbf{R}}_{\lambda}\paren{t \paren{\uu - \uw}; \hat{f}_{\lambda}\paren{\uu; \bullet}}
    \\
    &
    = t  \partial_2\hat{\mathbf{R}}_{\lambda}\paren{\paren{\uu - \uw}; \hat{f}_{\lambda}\paren{\uu; \bullet}}
    \\
    &
    = \frac{t}{n+1} \paren{
        \partial_2 \ell\paren{z, \hat{f}_{\lambda}\paren{\uu; X_{n+1}}}
        - \partial_2 \ell\paren{y, \hat{f}_{\lambda}\paren{\uu; X_{n+1}}}
    } K_{X_{n+1}},
\end{align*}
since $\uu - \uw = \paren{0, \ldots, 0, 1, -1}$.
Conjoined with the above inequality and by dividing both sides by $2 \lambda \norm{
\hat{f}_{\lambda}\paren{v; \bullet}
- \hat{f}_{\lambda}\paren{\uu; \bullet}
}_{\mathcal{H}}$,
\begin{align*}
    \norm{
        \hat{f}_{\lambda}\paren{v; \bullet}
        - \hat{f}_{\lambda}\paren{\uu; \bullet}
    }_{\mathcal{H}}
    \leq \norm{K_{X_{n+1}}}_{\mathcal{H}} \frac{t\abss{
        \partial_2 \ell\paren{z, \hat{f}_{\lambda}\paren{\uu; X_{n+1}}}
        - \partial_2 \ell\paren{y, \hat{f}_{\lambda}\paren{\uu; X_{n+1}}}
    }}{2\lambda\paren{n+1}}.
\end{align*}
One concludes by recalling that $\norm{K_{X_{n+1}}}_{\mathcal{H}} =  \sqrt{K_{n+1, n+1}}$ and
that $\wpred{\lambda}{\uu}{\bullet} = \pred{\lambda;}{z}$.
\end{proof}
}

\subsection{Proof of Proposition~\ref{prop.predictor.stability.smooth} and Theorem~\ref{thm.score.bound.smooth}}
\label{proof.score.bound.smooth}

\begin{proof}(Proof of Proposition~\ref{prop.predictor.stability.smooth})
    {Direct application of Lemma~\ref{lm.local.predictor.stability}
    for $v = \uw$ and the fact that $\wpred{\lambda}{\uw}{\bullet} = \pred{\lambda;}{y}$. 
    }
\end{proof}

\medskip

\begin{proof}(Proof of Theorem~\ref{thm.score.bound.smooth})
For every $i \in \brac{1, \ldots, n}$, combining \eqref{asm.score.lipschitz}
and Proposition~\ref{prop.predictor.stability.smooth}, the non-conformity score approximation quality is immediately bounded from above by
\begin{align*}
	\left|
	S_{\lambda; D^{y}} \paren{X_i, Y_i}
	- \widetilde{S}_{\lambda; D^{y}} \paren{X_{i}, Y_{i}}
	\right|
	&
	=
	\left| 
	s\paren{Y_i, \pred{\lambda;}{y}(X_i)}
	- s\paren{Y_i, \pred{\lambda;}{z}(X_i)}
	\right|
	\\
	&
	\leq
	\gamma
	\left|
	\pred{\lambda;}{y}(X_i)
	- \pred{\lambda;}{z}(X_i) 
	\right| \\
	&	\leq
	\gamma
	\abss{
		\doth{\pred{\lambda;}{y} - \pred{\lambda;}{z}}{K_{X_i}}
	}
	\\
	&
	\leq
	\gamma \normh{K_{X_i}}
	\normh{\pred{\lambda;}{y} - \pred{\lambda;}{z}} \\
	&
	\leq
	\gamma \sqrt{K_{i, i}}
	\normh{\pred{\lambda;}{y} - \pred{\lambda;}{z}}
	\\
	&
	\leq \gamma
	\sqrt{K_{i, i}}
	\sqrt{ K_{n+1, n+1}}
	\frac{\ \rho^{(1)}_{\lambda}(y)}{\lambda (n+1)}.
\end{align*}
Following  a similar reasoning for $i = n+1$,
\begin{align*}
	\left|
	S_{\lambda;D^{y}} \paren{X_i, y} - \widetilde{S}_{\lambda;D^{y}} \paren{X_{i}, y}
	\right|
	\leq  \gamma
	\sqrt{ K_{i, i}}
	\sqrt{ K_{n+1, n+1}}
	\frac{\ \rho^{(1)}_{\lambda}(y)}{\lambda (n+1)},
\end{align*}
which provides the desired conclusion.
\end{proof}

\subsection{Proof of Corollary~\ref{cor.smooth}}
\begin{proof}
\label{proof.smooth}
\noindent{\emph{Coverage.}}
Following a similar reasoning as in Theorem~\ref{thm.non.smooth},
the new approximate region is a confidence-region.

\bigskip{}
\noindent{\emph{Comparison with $\ufcprr{,(0)}{\lambda; \alpha}$.}}
Under \eqref{asm.loss.lipschitz},
for every $y \in \mathcal{Y}$, almost-surely
\begin{align*}
    \rho^{(1)}_{\lambda}(y)
    =
    \frac{1}{2}
    \left|
        - \partial_2 \ell(z, \pred{\lambda;}{z}(X_{n+1})) 
        + \partial_2 \ell(y, \pred{\lambda;}{z}(X_{n+1}))
    \right|
    \leq
    \frac{1}{2} 2 \rho = \rho.
\end{align*}
It follows that for every $i \in \brac{1, \ldots, n+1}$, and every $y \in \mathcal{Y}$,
$\widehat{\tau}_{\lambda; i}^{(1)}(y) \leq \widehat{\tau}_{\lambda; i}^{(0)}$ and almost-surely
\begin{align*}
    \mathbbm{1}
    \brac{
        \widetilde{S}_{\lambda; D^{y}}\paren{X_i} + \widehat{\tau}_{\lambda;i}^{(1)}(y)
        \leq \widetilde{S}_{\lambda; D^{y}}\paren{X_{n+1}} - \widehat{\tau}_{\lambda;n+1}^{(1)}(y)
    }
    &\leq
    \mathbbm{1}
    \brac{
        \widetilde{S}_{\lambda; D^{y}}\paren{X_i} + \widehat{\tau}_{\lambda;i}^{(0)}
        \leq \widetilde{S}_{\lambda; D^{y}}\paren{X_{n+1}} - \widehat{\tau}_{\lambda;n+1}^{(0)}
    },\\
    \mathbbm{1}
    \brac{
        \widetilde{S}_{\lambda; D^{y}}\paren{X_i} - \widehat{\tau}_{\lambda;i}^{(1)}(y)
        \leq \widetilde{S}_{\lambda; D^{y}}\paren{X_{n+1}} + \widehat{\tau}_{\lambda;n+1}^{(1)}(y)
    }
    &\geq
    \mathbbm{1}
    \brac{
        \widetilde{S}_{\lambda; D^{y}}\paren{X_i} - \widehat{\tau}_{\lambda;i}^{(0)}
        \leq \widetilde{S}_{\lambda; D^{y}}\paren{X_{n+1}} + \widehat{\tau}_{\lambda;n+1}^{(0)}
    }.
\end{align*}
After summing over $i$, adding $1$ and dividing by $n+1$, for every $y \in \mathcal{Y}$,
almost-surely
\begin{align*}
    \lfcpvv{,(0)}{\lambda; D}{y}
    \leq \lfcpvv{,(1)}{\lambda; D}{y}
    \leq \fcpvv{\lambda; D}{y}
    \leq \ufcpvv{,(1)}{\lambda; D}{y}
    \leq \ufcpvv{,(0)}{\lambda; D}{y}.
\end{align*}
This implies that the new approximate p-value function is closer
to the full-conformal one compared to the previous. Furthermore, in terms of regions,
almost-surely
\begin{align*}
    \lfcprr{,(0)}{\lambda; \alpha}
    \subseteq
    \lfcprr{,(1)}{\lambda; \alpha}
    \subseteq
    \fcprr
    \subseteq
    \ufcprr{,(1)}{\lambda; \alpha}
    \subseteq
    \ufcprr{,(0)}{\lambda; \alpha}.
\end{align*}
This implies that the new prediction-region is closer
to the full-conformal one compared to the previous,
\begin{align*}
    \thick{(1)}{\lambda; \alpha}
    \leq
    \thick{(0)}{\lambda; \alpha} \quad \mbox{a.s.}
\end{align*}
\end{proof}

\section{Concerning the case of very smooth loss-functions}
This section lists the proofs of the results in Section~\ref{sec.very.smooth}.

The following definition describes the expression of
some open set over which the predictor function of weights $v \to \wpred{\lambda}{v}{\bullet}$
needs to be at least twice continuously differentiable
for the new predictor approximation via influence functions and
the new upper-bounds to hold.

\begin{definition}[$\epsilon$-neighborhood of $\croch{\uu, \uw}$]
	\label{def.epsilon.neighborhood}
	For $\epsilon>0$
	\begin{align*}
		\Omega_{\epsilon} := 
		\left\{
		v \in \mathbb{R}^{n+2}
		: \exists a \in \croch{\uu, \uw},
		s.t. \left\|v - a\right\|_{2} < \epsilon
		\right\}.
	\end{align*}
\end{definition}

However, this $\epsilon$-neighborhood contains
weight vectors with some negative coordinates.
To ensure $\wpred{\lambda}{v}{\bullet}$ is well-defined
for every $v\in \Omega_{\epsilon}$,
$\epsilon$ must be small enough to keep
$\wer{\lambda}{v}{\bullet}$ strongly convex.
This is done by balancing two conflicting terms: the non-convex terms
induced by the negative coordinates and,
the strongly convex regularization term.

Along with the constant $\beta_{\ell; 2}$, the values $\kappa_{\mathcal{H}}\paren{X_{1}, X_{1}}$,
$\ldots$, $\kappa_{\mathcal{H}}\paren{X_{n+1}, X_{n+1}}$ control the amplitude of the potentially non-convex terms,
which is why they are required to be finite.
It follows that there is a non-empty range of values for $\epsilon$,
such that for every $v\in \Omega_{\epsilon}$, $\wpred{\lambda}{v}{\bullet}$ is well-defined.

\subsection{Preliminary properties}
\label{sec.preliminary.prop.2}

\begin{lemma}
	\label{lm.second.diff.rer}
	Assume \eqref{asm.loss.c2} holds true.
	For every $\lambda \in \paren{0, \infty}$, and
	for every $v \in \mathbb{R}^{n+2}$, and
	for every $f \in \mathcal{H}$,
	the second-order Fr{\'e}chet differential
	$\croch{
		\mathcal{D}^2\hat{\mathbf{R}}_{\lambda}\paren{v ;\bullet}
	}\paren{f} : \mathcal{H} \times \mathcal{H} \to \mathbb{R}$ of
	the regularized empirical-risk functional
	$\hat{\mathbf{R}}_{\lambda}\paren{v; \bullet} : \mathcal{H} \mapsto \mathbb{R}$
	is well-defined and is given by, for every $\paren{h_1, h_2} \in \mathcal{H} \times \mathcal{H}$,
    \begin{align*}
        \croch{
			\mathcal{D}^2\hat{\mathbf{R}}_{\lambda}\paren{v ;\bullet}
		}\paren{f}\paren{h_1, h_2}
        = \dotprod{\partial_2^2\hat{\mathbf{R}}_{\lambda}\paren{v; f} h_1}{h_2}{\mathcal{H}},
    \end{align*}
    where the linear operator
    $\partial_2^2\hat{\mathbf{R}}_{\lambda}\paren{v; f} : \mathcal{H} \mapsto  \mathcal{H}$
    is given by,
	\begin{align}
        \label{eq.risk.d2}
		\partial_2^2\hat{\mathbf{R}}_{\lambda}\paren{v; f}
		&
		:=
		\frac{1}{n+1} \sum_{i=1}^{n}
		v_i \partial_{2}^2\ell\paren{Y_i, f\paren{X_i}}
		K_{X_i} \otimes K_{X_i}
		\\ \notag
		&
		\quad
		+
		\frac{v_{n+1}}{n+1}
		\partial_{2}^2\ell\paren{z, f\paren{X_{n+1}}}
		K_{X_{n+1}} \otimes K_{X_{n+1}}
		\\ \notag
		&
		\quad
		+
		\frac{v_{n+2}}{n+1}
		\partial_{2}^2\ell\paren{y, f\paren{X_{n+1}}}
		K_{X_{n+1}} \otimes K_{X_{n+1}}
		+
		2\lambda\mathrm{Id}.
	\end{align}
\end{lemma}
\begin{proof}
	Let $\lambda \in\paren{0, +\infty}$, and
	$v \in \mathbb{R}^{n+2}$, and
	$f \in \mathcal{H}$.
	Let $\paren{x, u} \in \mathcal{X} \times \mathcal{Y}$.
	Let us recall from the proof of Lemma~\ref{lm.first.diff.rer}
	that, for every $g \in \mathcal{H}$, the second-order Fr{\'e}chet differential
	$\croch{\mathcal{D}^2 M_{x, u}}\paren{f} : \mathcal{H} \times \mathcal{H} \to \mathbb{R}$
	of the functional $M_{x, u}: \mathcal{H} \to \mathbb{R}$,  $f \mapsto \ell\paren{u, f\paren{x}}$
	at $g$ is given by, for every $h_1 \in \mathcal{H}$
	\begin{align*}
		\croch{
			\mathcal{D}
			M_{x, u}
		}\paren{g} \paren{h_1}
		&
		= \partial_{2}\ell\paren{u, L_{x}\paren{g}}
		L_{x}\paren{h_1}
		\\
		&
		=
		\paren{
			L_{x}\paren{h_1}
			\partial_{2}\ell\paren{u, \bullet}
			\circ L_x
		}\paren{g},
	\end{align*}
	where $\partial_{2}\ell\paren{u, \bullet}$ exists thanks to \eqref{asm.loss.c2}.
	
	The second-order Fr{\'e}chet differential
	$\croch{\mathcal{D}^2 M_{x, u}}\paren{f} : \mathcal{H} \times \mathcal{H} \to \mathbb{R}$
	of the functional $M_{x, u}: \mathcal{H} \to \mathbb{R}$,  $f \mapsto \ell\paren{u, f\paren{x}}$
	at $f$ is given by,
	for every $\paren{h_1, h_2} \in \mathcal{H} \times \mathcal{H}$
	\begin{align*}
		\croch{\mathcal{D}^2 M_{x, u}}\paren{f}\paren{h_1, h_2}
		&
		= \croch{
			\mathcal{D}
			\paren{
				\croch{
					\mathcal{D}
					M_{x, u}
				}\paren{\bullet} \paren{h_1}
			}
		} \paren{f}\paren{h_2}
		\\
		&
		=
		\croch{
			\mathcal{D}
			\paren{
				L_x\paren{h_1}
				\partial_2 \ell\paren{
					u, \bullet
				}
				\circ
				L_x
			}
		} \paren{f}\paren{h_2}
		\\
		&
		=
		L_x\paren{h_1}
		\croch{
			\mathcal{D}
			\paren{
				\partial_2 \ell\paren{
					u, \bullet
				}
				\circ
				L_x
			}
		} \paren{f}\paren{h_2}
		\\
		&
		=
		L_x\paren{h_1}
		\croch{
			\mathcal{D}
			\partial_2 \ell\paren{
				u, \bullet
			}
		}
		\paren{
			L_x\paren{f}
		}
		\paren{
			\mathcal{D}
			L_x \paren{f}\paren{h_2}
		}
		\\
		&
		=
		L_x\paren{h_1}
		\partial_2^2\ell\paren{
			u, L_x\paren{f}
		}
		L_x\paren{h_2}
		\\
		&
		=
		\partial_2^2\ell\paren{
			u, f\paren{x}
		}
		L_x\paren{h_1}
		L_x\paren{h_2},
	\end{align*}
	where the first equality is by definition, and
	the second equality follows from the above reminder, and
	the third equality follows from the linearity of the differential, and
	the fourth equality follows from the chain rule, and
	the fith equality follows from the coincidence of the Fr{\'e}chet differential
	and the regular derivative for functions in $\mathbb{R}^{\mathbb{R}}$
	and the existence of $\partial_2^2\ell\paren{
		u, \bullet
	}$ under \eqref{asm.loss.c2}, and
	the last equality is just reordering the terms.

	The second-order Fr{\'e}chet differential
	$\croch{\mathcal{D}^2 \normh{\bullet}^2}\paren{f} : \mathcal{H} \times \mathcal{H} \to \mathbb{R}$
	of the norm squared at $f$ is given by,
	for every $\paren{h_1, h_2} \in \mathcal{H} \times \mathcal{H}$
	\begin{align*}
		\croch{\mathcal{D}^2 \normh{\bullet}^2}\paren{f}\paren{h_1, h_2}
		&
		= \croch{
			\mathcal{D}
			\paren{
				\croch{
					\mathcal{D}
					\normh{\bullet}^2
				}\paren{\bullet} \paren{h_1}
			}
		} \paren{f}\paren{h_2}
		\\
		&
		=2 \doth{h_2}{h_1}.
	\end{align*}
	In fact,
	\begin{align*}
		\croch{
			\mathcal{D}
			\normh{\bullet}^2
		}\paren{f + h_2} \paren{h_1}
		= 2 \doth{f+h_2}{h_1}
		&
		= 2 \doth{f}{h_1}
		+ 2 \doth{h_2}{h_1}
		\\
		&
		=
		\croch{
			\mathcal{D}
			\normh{\bullet}^2
		}\paren{f} \paren{h_1}
		+ 2 \doth{h_2}{h_1},
	\end{align*}
	where the first equality follows from a result
	the proof of Lemma~\ref{lm.first.diff.rer}, and
	the right most term on the second line is
	a bounded bilinear form on $\mathcal{H}$.
	
	By linearity of the Fr{\'e}chet differential,
	the second-order Fr{\'e}chet differential
	$\croch{
		\mathcal{D}^2\hat{\mathbf{R}}_{\lambda}\paren{v ;\bullet}
	}\paren{f} : \mathcal{H} \times \mathcal{H} \to \mathbb{R}$ of
	the regularized empirical-risk functional
	$\hat{\mathbf{R}}_{\lambda}\paren{v; \bullet} : \mathcal{H} \mapsto \mathbb{R}$
	is then given by, for every $\paren{h_1, h_2} \in \mathcal{H} \times \mathcal{H}$
	\begin{align*}
		&
        \croch{
			\mathcal{D}^2\hat{\mathbf{R}}_{\lambda}\paren{v ;\bullet}
		}\paren{f}\paren{h_1, h_2}
		\\
        &
		=
		\frac{1}{n+1} \sum_{i=1}^{n}
		v_i \partial_{2}^2\ell\paren{Y_i, f\paren{X_i}}
		L_{X_i} \paren{h_1} L_{X_i} \paren{h_2}
		\\
		&
		\quad
		+
		\frac{v_{n+1}}{n+1}
		\partial_{2}^2\ell\paren{z, f\paren{X_{n+1}}}
		L_{X_{n+1}} \paren{h_1} L_{X_{n+1}} \paren{h_2}
		\\
		&
		\quad
		+
		\frac{v_{n+2}}{n+1}
		\partial_{2}^2\ell\paren{y, f\paren{X_{n+1}}}
		L_{X_{n+1}} \paren{h_1} L_{X_{n+1}} \paren{h_2}
		+
		2\lambda \doth{h_1}{h_2}
        \\
        &
        =
		\frac{1}{n+1} \sum_{i=1}^{n}
		v_i \partial_{2}^2\ell\paren{Y_i, f\paren{X_i}}
		\dotprod{K_{X_{i}}}{h_1}{\mathcal{H}}
        \dotprod{K_{X_{i}}}{h_2}{\mathcal{H}}
		\\
		&
		\quad
		+
		\frac{v_{n+1}}{n+1}
		\partial_{2}^2\ell\paren{z, f\paren{X_{n+1}}}
		\dotprod{K_{X_{n+1}}}{h_1}{\mathcal{H}}
        \dotprod{K_{X_{n+1}}}{h_2}{\mathcal{H}}
		\\
		&
		\quad
		+
		\frac{v_{n+2}}{n+1}
		\partial_{2}^2\ell\paren{y, f\paren{X_{n+1}}}
		\dotprod{K_{X_{n+1}}}{h_1}{\mathcal{H}}
        \dotprod{K_{X_{n+1}}}{h_2}{\mathcal{H}}
		+
		2\lambda \doth{h_1}{h_2},
        \\
        &
        = \dotprod{\partial_2^2\hat{\mathbf{R}}_{\lambda}\paren{v; f} h_1}{h_2}{\mathcal{H}},
	\end{align*}
    where the linear operator
    $\partial_2^2\hat{\mathbf{R}}_{\lambda}\paren{v; f} : \mathcal{H} \mapsto  \mathcal{H}$
    is given by Eq.~\eqref{eq.risk.d2}.
\end{proof}

\medskip

\begin{lemma}
    \label{lm.wer.strongly.convex.2}
        Assume \eqref{asm.loss.convex},
        \eqref{asm.loss.c2},
        \eqref{asm.loss.smooth}, and
        \eqref{asm.finite.kern} hold true. 
        Then, for every $\epsilon>0$ with
    \begin{align}
    \label{eq.epsilon.small}
        \tag{SmallEps}
        \epsilon 
        < 
        \min
        \left(
            \frac{ \lambda (n+1)}{ \beta_{\ell; 2} K_{n+1, n+1} }
            , 1/2
        \right)
    \end{align}
    almost-surely,
    for every $v \in \Omega_{\epsilon}$ (from Definition~\ref{def.epsilon.neighborhood}), the weighted regularized empirical-risk $\wer{\lambda}{v}{\bullet}$ (see Eq.~\ref{eq.Emp.Risk.weighted.rgularized}) is
   $ \lambda$-strongly convex. 
\end{lemma}
Notice that the previous lemma differs from Lemma~\ref{lm.wer.strongly.convex.1} by the requirement that the vectors $v \in\Omega_{\epsilon} $ are no longer constrained to belong to $\mathbb{R}_+^{n+2} $.%

\begin{proof}
    Let $\epsilon>0$.
    %
    %
    Let us introduce $v \in \Omega_{\epsilon}$ and recall that $\uu := (1, \ldots, 1, 1, 0)$ and $\uw := (1, \ldots, 1, 0, 1)$.
    Then, there exists $a  \in [\uu, \uw]$ such that $\|v - a\|_2<\epsilon$.
    It follows that there exists $t_1, t_2 \in [0, 1]$ with $t_1+t_2 = 1$,
    and $\paren{\delta_1, \ldots, \delta_{n+2}}$ with $|\delta_i| < \epsilon$,
    such that 
    $
    v_{n+1}
    =
    t_1 - \delta_{n+1}$,
    $
    v_{n+2}
    =
    t_2 - \delta_{n+2}
    $, and 
    for every $i \in \brac{1, \ldots, n}$, 
    $v_i = 1 - \delta_i$.

    If $\epsilon <1/2$ then for every $i \in \brac{1, \ldots, n}$, $v_i > 0$, and
    \begin{enumerate}
        \item if $v_{n+1} = t_1 - \delta_{n+1} \leq 0$, then $v_{n+2} = t_2 - \delta_{n+2} \geq 0$, and
        \item if $v_{n+2} = t_2 - \delta_{n+2} \leq 0$, then $v_{n+1} = t_1 - \delta_{n+1} \geq 0$.
    \end{enumerate}
    Let us consider the second case,
    $- \epsilon < v_{n+2} = t_2 - \delta_{n+2} \leq 0$ and $v_{n+1} \geq 0$.
    Let $f \in \mathcal{H}$.
    It follows that under \eqref{asm.loss.smooth},
    for every $u \in \mathcal{Y}$,
    $v_{n+2}\ddtwo \ell(u, f\paren{X_{n+1}}) \geq -\epsilon \beta_{\ell; 2}$,
    and under \eqref{asm.loss.convex},
    for every $i \in \brac{1, \ldots, n+1}$,
    $v_i \ddtwo \ell(u, f\paren{X_{i}}) \geq 0$.
    Let $h \in \mathcal{H}$.
    Applying $\croch{\mathcal{D}^2\wer{\lambda}{v}{\bullet}}\paren{f}$ to $(h, h)$,
    and using the expression given by Lemma~\ref{lm.second.diff.rer} which holds true under \eqref{asm.loss.c2},
    \begin{align*}
        \croch{\mathcal{D}^2\wer{\lambda}{v}{\bullet}}\paren{f} \paren{h, h}
        &
        = \dotprod{
            \partial_2^2 \hat{\mathbf{R}}_{\lambda}\paren{v, f} h
        }{
            h
        }{\mathcal{H}}
        \\
        &
        =\frac{1}{n+1}
        \sum_{i=1}^n
        v_i \partial_2^2 \ell\paren{Y_i, f\paren{X_i}} h(X_i)^2
        \\
        &
        \quad
        + \frac{1}{n+1}
        v_{n+1} \partial_2^2 \ell\paren{z, f\paren{X_{n+1}}}
        h(X_{n+1})^2
        \\
        &
        \quad
        + \frac{1}{n+1}
        v_{n+2} \partial_2^2 \ell\paren{y, f\paren{X_{n+1}}}
        h(X_{n+1})^2
        + 2 \lambda \|h\|^2_{\mathcal{H}}
        \\
        &
        \geq
        - \epsilon
        \frac{1}{n+1}
        \beta_{\ell; 2}
        \|K_{X_{n+1}}\|_{\mathcal{H}}^2
        \|h\|_{\mathcal{H}}^2
        + 2 \lambda \|h\|^2_{\mathcal{H}}
        \\
        &
        \geq
        \left(
            - \epsilon
            \beta_{\ell; 2}
            \frac{K_{n+1, n+1}}{n+1}
            + 2 \lambda
        \right)
        \|h\|_{\mathcal{H}}^2
        \geq
        \lambda \|h\|_{\mathcal{H}}^2,
    \end{align*}
    where $\partial_2^2 \hat{\mathbf{R}}_{\lambda}\paren{v, f} : \mathcal{H} \to \mathcal{H}$
    is given by Eq.~\eqref{eq.risk.d2} and the last inequality hold trues as soon as
    \begin{align*}
        \epsilon 
        < 
        \min
        \left(
            \frac{ \lambda (n+1)}{ \beta_{\ell; 2} K_{n+1, n+1}}
            , 1/2
        \right).
    \end{align*}

   Let us emphasize that the term in the right-hand side is almost-surely non-zero under \eqref{asm.loss.smooth}, \eqref{asm.finite.kern} since
    $\lambda >0$.
\end{proof}

\medskip

\begin{lemma}
\label{lm.predictor.func.well.defined.2}
Assume \eqref{asm.loss.convex},
    \eqref{asm.loss.c2},
    \eqref{asm.loss.smooth}, and
    \eqref{asm.finite.kern} holds true. 
 Then for every $\epsilon>0$ fulfilling Eq.~\eqref{eq.epsilon.small} a.s., it comes that
for every $v \in \Omega_{\epsilon}$ (see Definition~\ref{def.epsilon.neighborhood}), the predictor $\wpred{\lambda}{v}{\bullet}$ from Definition~\ref{def.weighted.predictor} does exist, is unique, and belongs to
the sub-space $\mathcal{A} \subseteq \mathcal{H}$ from Notation~\ref{not.span}.
\end{lemma}
\begin{proof}
Let $\epsilon >0$ satisfying Eq.~\eqref{eq.epsilon.small} and take $v  \in \Omega_{\epsilon}$.
By Lemma~\ref{lm.wer.strongly.convex.2} under
\eqref{asm.loss.convex},
\eqref{asm.loss.c2},
\eqref{asm.loss.smooth},
$\wer{\lambda}{v}{\bullet}$ is $\lambda$-strongly convex.
Since \eqref{asm.loss.c2} implies \eqref{asm.loss.lsc},
it follows that $\wer{\lambda}{v}{\bullet}$ is lower semi-continuous.
Conjoining the two with \eqref{asm.loss.convex} and Lemma~\eqref{lm.coercive};
$\wer{\lambda}{v}{\bullet}$ is coercive.
Therefore Corollary 5.6 of \citet{alexanderianOptimizationInfinitedimensionalHilbert2019} implies that
a minimizer $\wpred{\lambda}{v}{\bullet}$
does exist and is unique.
{
Since $\mathcal{H}$ is an Hilbert space
and $\mathcal{A}$ is the span of a finite number vector of $\mathcal{H}$,
$\mathcal{H} = \mathcal{A} \oplus  \mathcal{A}^{\perp}$.
It follows that there exist a unique functions
$f_{\mathcal{A}} \in \mathcal{A}$ and
$f_{\mathcal{A}^{\perp}} \in \mathcal{A}^{\perp}$
such that $\wpred{\lambda}{v}{\bullet}
= f_{\mathcal{A}} + f_{\mathcal{A}^{\perp}}$.
Since $f_{\mathcal{A}^{\perp}} \perp K_{X_i}$,
for every $i \in \brac{1, \ldots, n+1}$,
\begin{align*}
    \wpred{\lambda}{v}{X_i} 
    =
    \doth{\wpred{\lambda}{v}{\bullet}}{K_{X_i}}
    =
    \doth{f_{\mathcal{A}} + f_{\mathcal{A}^{\perp}}}{K_{X_i}}
    = 
    \doth{f_{\mathcal{A}}}{K_{X_i}} 
    = f_{\mathcal{A}}\paren{X_i}. 
\end{align*}
It follows that
\begin{align*}
    \wer{\lambda}{v}{\wpred{\lambda}{v}{\bullet}}
    &= \wer{0}{v}{\wpred{\lambda}{v}{\bullet}}
    + \|\wpred{\lambda}{v}{\bullet}\|_{\mathcal{H}}
    \\
    &=
    \wer{0}{v}{f_{\mathcal{A}}}
    + \|f_{\mathcal{A}} 
    + f_{\mathcal{A}^{\perp}}\|_{\mathcal{H}}
    = \wer{0}{v}{f_{\mathcal{A}}}
    + \|f_{\mathcal{A}} \|_{\mathcal{H}}
    + \|f_{\mathcal{A}^{\perp}}\|_{\mathcal{H}}
    \\
    &
    \geq
    \wer{0}{v}{f_{\mathcal{A}}}
    + \|f_{\mathcal{A}} \|_{\mathcal{H}}
    = \wer{\lambda}{v}{f_{\mathcal{A}}}.
\end{align*}
Since $\wpred{\lambda}{v}{\bullet}$ is the unique minimizer
of $\wer{\lambda}{v}{\bullet}$ then
$\wpred{\lambda}{v}{\bullet} = f_{\mathcal{A}}$.
}
\end{proof}

\medskip
\begin{lemma}
    \label{lm.vec.wer.strongly.convex.2}
        Assume
        \eqref{asm.loss.convex},
        \eqref{asm.loss.c2},
        \eqref{asm.loss.smooth}, and
        \eqref{asm.finite.kern} hold true.
    For every $\epsilon>0$ satisfying condition
    \begin{equation}
    \label{eq.epsilon.small.2}
    \tag{SmallEps.2}
        \epsilon 
        < 
        \min
        \left(\frac{
                 \lambda
                (n+1)
            }{\beta_{\ell; 2}}
            \frac{
                \mu_{*}^{(n+1)}
            }{
                \frac{1}{n+1}\norm{K_{\bullet, n+1}}^2_2
            }
            , 1/2
        \right),
    \end{equation}
    almost-surely, $\vwer{\lambda}{v}{\bullet}$ is $ \lambda \mu_{*}^{(n+1)} (n+1)$-strongly convex over the
    range $\mathcal{R}(K)$ of $K$, for every $v \in \Omega_{\epsilon}$, where $\mu_{*}^{(n+1)}$ is smallest non-zero eigenvalue of $\frac{1}{n+1} K$.
    Moreover, for every $a \in \mathbb{R}^{n+1}$, $\nabla_2^2 
    \vwer{\lambda}{v}{a}$ is  invertible over $\mathcal{R}(K)$,
    with pseudo-inverse denoted by $\left[\nabla_2^2 \vwer{\lambda}{v}{a}\right]^+$.
\end{lemma}
\begin{proof}
    Let $\epsilon>0$ and $\sigma \in (0, 1)$.
    Let $v \in \Omega_{\epsilon}$.
    There exists $a  \in [\uu, \uw]$ such that $\|v - a\|_2<\epsilon$.
    It follows that there exists $t_1, t_2 \in [0, 1]$ with $t_1+t_2 = 1$,
    and $\paren{\delta_1, \ldots, \delta_{n+2}}$ with $|\delta_i| < \epsilon$,
    such that 
    $
    v_{n+1}
    =
    t_1 - \delta_{n+1}$,
    $
    v_{n+2}
    =
    t_2 - \delta_{n+2}
    $, and 
    for every $i \in \brac{1, \ldots, n}$, 
    $v_i = 1 - \delta_i$.

    If $\epsilon <1/2$ then for every $i \in \brac{1, \ldots, n}$, $v_i > 0$, and
    \begin{enumerate}
        \item if $v_{n+1} = t_1 - \delta_{n+1} \leq 0$, then $v_{n+2} = t_2 - \delta_{n+2} \geq 0$, and
        \item if $v_{n+2} = t_2 - \delta_{n+2} \leq 0$, then $v_{n+1} = t_1 - \delta_{n+1} \geq 0$.
    \end{enumerate}
    Let us consider the second case,
    $- \epsilon < v_{n+2} = t_2 - \delta_{n+2} \leq 0$ and $v_{n+1} \geq 0$.
    Let $a \in \mathbb{R}^{n+1}$.
    It follows that under \eqref{asm.loss.smooth},
    for every $u \in \mathcal{Y}$,
    $v_{n+2}\ddtwo \ell(u, a^T K_{\bullet, n+1}) \geq -\epsilon \beta_{\ell; 2}$,
    and under \eqref{asm.loss.convex},
    for every $i \in \brac{1, \ldots, n+1}$,
    $v_i \ddtwo \ell(u, a^T K_{\bullet, i}) \geq 0$.
    Let $b \in \mathcal{R}(K)$.
    Applying $\DDtwo \vwer{\lambda}{v}{a} $ to $(b, b)$ 
    \begin{align*}
        b^T \nabla_2^2 
        \vwer{\lambda}{v}{a}
        b 
        &
        \geq
        \frac{1}{n+1}
        \sum_{i=1}^{n}
        v_i \partial_2^2 \ell
        (Y_i, a^T K_{\bullet, i})
        b^T K_{\bullet, i} K_{i, \bullet} b
        \\
        &
        + \frac{v_{n+1}}{n+1}
        \partial_2^2 \ell
        (z, a^T K_{\bullet, n+1})
        b^T K_{\bullet, n+1} K_{n+1, \bullet} b
        \\
        &
        + \frac{v_{n+2}}{n+1}
        \partial_2^2 \ell
        (y, a^T K_{\bullet, n+1})
        b^T K_{\bullet, n+1} K_{n+1, \bullet} b
        + 2 \lambda b^T K b
        \\
        &
        \geq
        -
        \frac{\epsilon \beta_{\ell; 2}}{n+1}
        \norm{K_{\bullet, n+1}}^2 \norm{b}^2
        + 2 \lambda \mu_*^{(n+1)} (n+1) \norm{b}^2
        \\
        &
        \geq
        \left(
            - \epsilon
            \beta_{\ell; 2}
            \frac{\norm{K_{\bullet, n+1}}^2}{n+1}
            + 2 \lambda \mu_*^{(n+1)} (n+1)
        \right)
        \norm{b}^2
        \geq
 \lambda \mu_*^{(n+1)} (n+1) \|h\|_{\mathcal{H}}^2,
    \end{align*}
    where the last inequality hold true as soon as
    \begin{equation*}
        \epsilon 
        < 
        \min
        \left(
            \frac{
                 \lambda
                (n+1)
            }{\beta_{\ell; 2}}
            \frac{
                \mu_{*}^{(n+1)}
            }{
                \frac{1}{n+1}\norm{K_{\bullet, n+1}}^2_2
            }
            , 1/2
        \right),
    \end{equation*}
    where term on the left hand side is almost-surely non-zero
    since $\frac{1}{n+1}\norm{K_{\bullet, n+1}}_2$ is almost-surely finite
    by \eqref{asm.finite.kern},
    $\beta_{\ell; 2}$
    is finite by \eqref{asm.loss.smooth},
    $\mu_{*}^{(n+1)}>0$ is by definition non-zero,
    $\lambda >0$ and $\sigma \in (0, 1)$.

    Since for every $v \in \Omega_{\epsilon}$,
    $\vwer{\lambda}{v}{\bullet}$
    is strongly convex over the range of $K$, it follows that
    for every $a \in \mathbb{R}^{n+1}$,
    its Hessian $\nabla_2^2 \vwer{\lambda}{v}{a}$ is invertible over said range.
    Let us note
    $\left[\nabla_2^2 \vwer{\lambda}{v}{a}\right]^+$ its pseudo-inverse.
\end{proof}

\medskip

\begin{lemma}
	\label{lm.link.between.Hessians}
	 Assume \eqref{asm.loss.c2} holds true.
	For every weight $v \in \mathbb{R}^{n+2}$ and $b \in \mathbb{R}^{n+1}$,    
    both the linear operator $\partial_2^2 \hat{\mathbf{R}}_{\lambda}\paren{v; f} : \mathcal{H} \to \mathcal{H}$ (see Eq.~\ref{eq.risk.d2})
    used to compute the second-order Fr{\'e}chet differential the regularized empirical-risk functional
    $\hat{\mathbf{R}}_{\lambda}\paren{v; \bullet} : \mathcal{H} \mapsto \mathbb{R}$ at $f$ (see Eq.~\ref{eq.Emp.Risk.weighted.rgularized}),
    and the Hessian matrix $\nabla^2_2 \hat{R}_{\lambda}\paren{v; b} \in \mathbb{R}^{\paren{n+1} \times \paren{n+1}}$
	of $\hat{R}_{\lambda}\paren{v; \bullet} : \mathbb{R}^{n+1} \mapsto \mathbb{R}$ at $b$  (see Eq.~\ref{eq.vec.wer}) do satisfy,
    for every $a \in \mathbb{R}^{n+1}$, 
	\begin{align*}
        \dotprod{\partial_2^2 \hat{\mathbf{R}}_{\lambda}\paren{v; f}h}{K_{X_{j}}}{\mathcal{H}}
		=
		\paren{
			\nabla^2_2\hat{R}_{\lambda}\paren{v; b} a
		}_j,    \qquad	 \forall j \in \brac{1, \ldots, n+1},
	\end{align*}
	where the functions $f, h\in \mathcal{A}$
	are defined as
	$f := \sum_{i=1}^{n+1}b_i K_{X_i}$ and
	$h := \sum_{i=1}^{n+1}a_i K_{X_i}$.
\end{lemma}
\begin{proof}
	Let $v \in \mathbb{R}^{n+2}$, and
	$a, b \in \mathbb{R}^{n+1}$, and 
	$j \in \brac{1, \ldots, n+1}$.
	Let $f \in \mathcal{H}$ be the function defined as
	$f := \sum_{l=1}^{n+1} b_l K_{X_l}$, and
	$h \in \mathcal{H}$ be the function defined as
	$h := \sum_{l=1}^{n+1} a_l K_{X_l}$.

	For every $i \in \brac{1, \ldots, n+1}$,
	\begin{align*}
		\dotprod{h}{K_{X_i}}{\mathcal{H}}
		= \sum_{l=1}^{n+1}a_l \dotprod{K_{X_l}}{K_{X_i}}{\mathcal{H}}
		= \sum_{l=1}^{n+1}a_l \kappa_{\mathcal{H}}\paren{X_l, X_i}
		= \sum_{l=1}^{n+1}a_l K_{l, i}
		= \paren{K a}_{i},
	\end{align*}
	where the third equality follows from the definition of the Gram matrix $K$.

	Under \eqref{asm.loss.c2},
    Lemma~\ref{lm.second.diff.rer} ensure that
    the linear operator $\partial_2^2 \hat{\mathbf{R}}_{\lambda}\paren{v; f} : \mathcal{H} \to \mathcal{H}$ (see Eq.~\ref{eq.risk.d2})
    is well defined and thus,
	\begin{align*}
		\dotprod{\partial_2^2 \hat{\mathbf{R}}_{\lambda}\paren{v; f}h}{K_{X_{j}}}{\mathcal{H}}
		&
		=
		\frac{1}{n+1} \sum_{i=1}^{n}
		v_i \partial_{2}^2\ell\paren{Y_i, f\paren{X_i}}
        \dotprod{h}{K_{X_i}}{\mathcal{H}}
        \dotprod{K_{X_j}}{K_{X_i}}{\mathcal{H}}
		\\
		&
		\quad
		+
		\frac{v_{n+1}}{n+1}
		\partial_{2}^2\ell\paren{z, f\paren{X_{n+1}}}
		\dotprod{h}{K_{X_{n+1}}}{\mathcal{H}}
        \dotprod{K_{X_j}}{K_{X_{n+1}}}{\mathcal{H}}
		\\
		&
		\quad
		+
		\frac{v_{n+2}}{n+1}
		\partial_{2}^2\ell\paren{y, f\paren{X_{n+1}}}
		\dotprod{h}{K_{X_{n+1}}}{\mathcal{H}}
        \dotprod{K_{X_j}}{K_{X_{n+1}}}{\mathcal{H}}
		+
		2\lambda \doth{h}{K_{X_j}}
		\\
		&
		=
		\frac{1}{n+1} \sum_{i=1}^{n}
		v_i \partial_{2}^2\ell\paren{Y_i, f\paren{X_i}}
		K_{j, i} \paren{K a}_{i}
		\\
		&
		\quad
		+
		\frac{v_{n+1}}{n+1}
		\partial_{2}^2\ell\paren{z, f\paren{X_{n+1}}}
		K_{j, n+1} \paren{K a}_{n+1}
		\\
		&
		\quad
		+
		\frac{v_{n+2}}{n+1}
		\partial_{2}^2\ell\paren{y, f\paren{X_{n+1}}}
		K_{j, n+1} \paren{K a}_{n+1}
		+
		2\lambda \paren{K a}_{j}
		\\
		&
		=
		\paren{
			\croch{
				\frac{1}{n+1}
				K d_{\ell}\paren{v; f} K
				+ 2 \lambda K
			} a
		}_j \in \mathbb{R},
	\end{align*}
	where $d_{\ell}\paren{v; f} \in \mathbb{R}^{\paren{n+1} \times \paren{n+1}}$
	is the diagonal matrix such that for every $i \in \paren{1, \ldots, n+1}$,
	\begin{align*}
		\begin{aligned}
			\left[d_{\ell}\paren{v; f}\right]_{i, i}
			&=
			v_i \partial_2^2 \ell\paren{Y_i, f\paren{X_i}},
			&& \mbox{if }1\leq i \leq n
			\\
			\left[d_{\ell}\paren{v; f}\right]_{i, i}
			&=v_{n+1} \partial_2^2 \ell\paren{z, f\paren{X_{n+1}}}
			+v_{n+2} \partial_2^2 \ell\paren{z, f\paren{X_{n+1}}}
			&&\mbox{if }i=n+1.
		\end{aligned}
	\end{align*}
	
	Under \eqref{asm.loss.c2},
	the Hessian matrix $\nabla^2_2 \hat{R}_{\lambda}\paren{v; b} \in \mathbb{R}^{\paren{n+1} \times \paren{n+1}}$
	of the regularized empirical-risk vector function $\hat{R}_{\lambda}\paren{v; \bullet} : \mathbb{R}^{n+1} \mapsto \mathbb{R}$
	is given by
	\begin{align*}
		\nabla^2_2\hat{R}_{\lambda}\paren{v; b}
		&
		=
		\frac{1}{n+1}
		\sum_{i=1}^{n+1}
		v_i \partial_2^2\ell\paren{Y_i, b^T K_{\bullet, i}}
		K_{\bullet, i}
		K_{\bullet, i}^T
		\\
		&
		\quad
		+
		\frac{v_{n+1}}{n+1}
		\partial_2^2\ell\paren{z, b^T K_{\bullet, n+1}}
		K_{\bullet, n+1}
		K_{\bullet, n+1}^T
		\\
		&
		\quad
		+
		\frac{v_{n+2}}{n+1}
		\partial_2^2\ell\paren{y, b^T K_{\bullet, n+1}}
		K_{\bullet, n+1}
		K_{\bullet, n+1}^T
		+
		2\lambda K
		\\
		&
		=
		\frac{1}{n+1}
		K\paren{
			\sum_{i=1}^{n+1}
			v_i \partial_2^2\ell\paren{Y_i, b^T K_{\bullet, i}}
			e_i
			e_i^T
		}
		K^{T}
		\\
		&
		\quad
		+
		\frac{1}{n+1}
		K\paren{
			v_{n+1}
			\partial_2^2\ell\paren{z, b^T K_{\bullet, n+1}}
			e_{n+1}
			e_{n+1}^{T}
		}
		K^{T}
		\\
		&
		\quad
		+
		\frac{1}{n+1}
		K\paren{
			v_{n+2}
			\partial_2^2\ell\paren{y, b^T K_{\bullet, n+1}}
			e_{n+1}
			e_{n+1}^{T}
		}
		K^{T}
		+
		2\lambda K
		\\
		&
		=
		\frac{1}{n+1}
		Kd_{\ell}\paren{v; f}K + 2\lambda K,
	\end{align*}
	where the last equality holds for
	$f = \sum_{i=1}^{n+1}b_i K_{X_i}$.
	
	It follows that
	for $f = \sum_{i=1}^{n+1}b_i K_{X_i}$
	and for $h = \sum_{i=1}^{n+1}a_i K_{X_i}$,
	for every $j \in \brac{1, \ldots, n+1}$
	\begin{align*}
		\dotprod{\partial_2^2 \hat{\mathbf{R}}_{\lambda}\paren{v; f}h}{K_{X_{j}}}{\mathcal{H}}        
		=
		\paren{
			\croch{
				\frac{1}{n+1}
				K d_{\ell}\paren{v; f} K
				+ 2 \lambda K
			} a
		}_j 
		=
		\paren{
			\nabla^2_2\hat{R}_{\lambda}\paren{v; b} a
		}_j.        
	\end{align*}
\end{proof}

\medskip

\begin{lemma}
\label{lm.hess.invertible.2}
    Assume \eqref{asm.loss.convex},
        \eqref{asm.loss.c2},
        \eqref{asm.loss.smooth}, and
        \eqref{asm.finite.kern} hold true.
      Let us consider $\epsilon>0$ satisfying Eq.~\eqref{eq.epsilon.small}
    and Eq.~\eqref{eq.epsilon.small.2},
    almost-surely.

Then, for every $v \in \Omega_{\epsilon}$ and $f \in \mathcal{A}$,
    the linear operator $\partial_2^2 \hat{\mathbf{R}}_{\lambda}\paren{v; f} : \mathcal{H} \to \mathcal{H}$ (see Eq.~\ref{eq.risk.d2})
has an inverse
$\left[
    \DDtwo \wer{\lambda}{v}{f}
\right]^+$ over $\mathcal{A}$.
\end{lemma}
\begin{proof}
{%
Let $\epsilon >0$ satisfying Eq.~\eqref{eq.epsilon.small} and Eq.~\eqref{eq.epsilon.small.2},
and $v  \in \Omega_{\epsilon}$ be a weight vector, and $f \in \mathcal{A}$ and $g \in \mathcal{A}$.
Let us find $h \in \mathcal{A}$ such that
\begin{align}
\label{eq.find.inverse}
    \DDtwo \wer{\lambda}{v}{f} h= g.
\end{align}
Since $f, g, h \in \mathcal{A}$, there exist vectors
    $b, \tilde{a}, a \in \mathbb{R}^{n+1}$
    such that
    $f = \sum_{i=1}^{n+1} b_i K_{X_i}$,
    $g = \sum_{i=1}^{n+1} \tilde{a}_i K_{X_i}$,
    and $h = \sum_{i=1}^{n+1} a_i K_{X_i}$.
    Since \eqref{asm.loss.c2}
    hold true, Lemma~\ref{lm.link.between.Hessians}
    establishes that, for every $j \in \brac{1, \ldots, n+1}$,
    \begin{align*}
        \dotprod{\DDtwo \wer{\lambda}{v}{f} h}{K_{X_{j}}}{\mathcal{H}}
		= \paren{\nabla^2_2\hat{R}_{\lambda}\paren{v; b} a}_j.
    \end{align*}
    Moreover, it easy to check that
    for every $j \in \brac{1, \ldots, n+1}$,
    $g\paren{X_j} = K_{ \bullet, j}^T \tilde{a}$.
    It follows that the system of equation
    \begin{align*}
        \dotprod{\DDtwo \wer{\lambda}{v}{f} h}{K_{X_{j}}}{\mathcal{H}}
        = g\paren{X_j}
        && 1\leq j \leq n+1
    \end{align*}
    is equivalent to the linear equation
    \begin{align*}
        \nabla^2_2\hat{R}_{\lambda}\paren{v; b} a
        = K\tilde{a}.
    \end{align*}
    Since \eqref{asm.loss.convex},
    \eqref{asm.loss.c2},
    \eqref{asm.loss.smooth}, and
    \eqref{asm.finite.kern}  hold true, and since 
    $\epsilon >0$ satisfies Eq.~\eqref{eq.epsilon.small.2},
    then Lemma~\ref{lm.vec.wer.strongly.convex.2}
    establishes that $\nabla^2_2\hat{R}_{\lambda}\paren{v; b}$
    is invertible over the range $\mathcal{R}\paren{K}$
    of the Gram matrix $K$.
    Moreover, the term $K \tilde{a}$
    on the right hand side
    of the above linear equation is
    a element of said range.
    It follows that
    \begin{align*}
        a
        = \croch{\nabla^2_2\hat{R}_{\lambda}\paren{v; b}}^+ K \tilde{a}.
    \end{align*}

It follows that the function $h = \sum_{i=1}^{n+1} a_i K_{X_i}$ solves Eq.~\eqref{eq.find.inverse} and $\DDtwo \wer{\lambda}{v}{f}$ is surjective over $\mathcal{A}$.
Since \eqref{asm.loss.convex},
\eqref{asm.loss.c2},
\eqref{asm.loss.smooth}, and
\eqref{asm.finite.kern} hold true,
and since 
$\epsilon >0$ satisfies Eq.~\eqref{eq.epsilon.small}, 
by Lemma~\ref{lm.wer.strongly.convex.2},
$\wer{\lambda}{v}{\bullet}$ is strongly convex.
Conjoined Lemma~\ref{lm.coercive}, this implies that $\wer{\lambda}{v}{\bullet}$ is coercive.
As the linear operator definining its second differential,
$\DDtwo \wer{\lambda}{v}{f}$ is thus also coercive and therefore injective.
As a consequence, $\DDtwo \wer{\lambda}{v}{f}$ is bijective (invertible) over $\mathcal{A}$
and let us note $\left[
    \DDtwo \wer{\lambda}{v}{f}
\right]^+$ its inverse over $\mathcal{A}$.
}
\end{proof}

\medskip

\begin{lemma}
\label{lm.pred.func.c2}
Under \eqref{asm.loss.convex},
    \eqref{asm.loss.smooth},
    \eqref{asm.loss.c2}, and
    \eqref{asm.finite.kern},
let us consider $\epsilon >0$ satisfying conditions Eq.~\eqref{eq.epsilon.small} and Eq.~\eqref{eq.epsilon.small.2},
    almost-surely.
Then,    
    $v\in \Omega_{\epsilon} 
    \mapsto 
    \wpred{\lambda}{v}{\bullet} \in \mathcal{A}$
    is once continuously differentiable over $\Omega_{\epsilon}$.

Further assuming \eqref{asm.loss.c3} implies that
 $v\in \Omega_{\epsilon} \mapsto 
\wpred{\lambda}{v}{\bullet} \in \mathcal{A}$
is twice continuously differentiable over $\Omega_{\epsilon}$.

\end{lemma}
\begin{proof}
\label{proof.pred.func.c2}
Let $\epsilon >0$ and $\sigma \in (0, 1)$
fulfilling Eq.~\eqref{eq.epsilon.small}
and Eq.~\eqref{eq.epsilon.small.2}.

Let us introduce the function
$G(v; f) : \Omega_{\epsilon} \times \mathcal{A} \to \Omega_{\epsilon} \times \mathcal{A}$ given,
for every $(v; f) \in \Omega_{\epsilon} \times \mathcal{A}$, by
\begin{align*}
    G(v; f)
    = \paren{G_1(v; f);
    G_2(v; f)}
    := \paren{v;
    \Dtwo \wer{\lambda}{v}{f}}. 
\end{align*}
Since $\eqref{asm.loss.c2}$ hold true then
$G$ is once continuously differentiable.
Assuming
\eqref{asm.loss.convex},
\eqref{asm.loss.smooth},
$\eqref{asm.loss.c2}$, and
\eqref{asm.finite.kern},
Lemma~\ref{lm.hess.invertible.2}
imply the differential $\mathcal{D} G(v; f)$ of $G$
at $\paren{v, f}$, which is given by
\begin{align*}
    \mathcal{D} G(v; f) 
    &
    =
    \begin{pmatrix}
        \mathcal{D}_1 G_1(v; f)
        &
        \Dtwo G_1(v; f)
        \\
        \mathcal{D}_1 G_2(v; f)
        &
        \Dtwo G_2(v; f)
    \end{pmatrix}
    = 
    \begin{pmatrix}
        \mathcal{D}_1 v
        &
        \Dtwo v
        \\
        \mathcal{D}_1 \Dtwo \wer{\lambda}{v}{f}
        &
        \Dtwo \Dtwo \wer{\lambda}{v}{f}
    \end{pmatrix}
    \\
    &
    = 
    \begin{pmatrix}
        \mathrm{Id}
        &
        0_{\mathcal{H}}
        \\
        \Dtwo \wer{\lambda}{\bullet}{f}
        &
        \DDtwo \wer{\lambda}{v}{f}
    \end{pmatrix},
\end{align*}
is invertible over $\Omega_{\epsilon}  \times \mathcal{A}$,
for every $(v; f) \in \Omega_{\epsilon}  \times \mathcal{A}$. 
By the Inverse Function Theorem, 
in a small neighborhood of $(v; f)$, 
$G$ is one-to-one and thus invertible with its inverse 
$G^{+}$ being a once continuously differentiable
in a neighborhood of $G(v; f)$. 
Now replace $f$ with $\wpred{\lambda}{v}{\bullet}$.
$G(v; \wpred{\lambda}{v}{\bullet})
=
\paren{
    v; \Dtwo \wer{\lambda}{
    v}{\wpred{\lambda}{v}{\bullet}}
} = (v; 0_{\mathcal{H}})$, which establishes that in a neighborhood of
$(v, 0_{\mathcal{H}})$, $G^{+}$ is once continuously differentiable. 

Let $h \in \mathcal{A}$, such that $\Dtwo \wer{\lambda}{v}{h} = 0_{\mathcal{H}}$. 
Since %
\eqref{asm.loss.convex},
\eqref{asm.loss.smooth},
\eqref{asm.loss.c2}, and
\eqref{asm.finite.kern} hold true, Lemma~\ref{lm.wer.strongly.convex.2} proves that $\wer{\lambda}{v}{\bullet}$ is $\lambda$-strongly convex.
Therefore $h$ is a global minimizer.
Since the global minimizer is unique, $h = \wpred{\lambda}{v}{\bullet}$, thus 
$G^+(v, 0_{\mathcal{H}})
=
(v; \wpred{\lambda}{v}{\bullet})$.

For every $v \in \Omega_{\epsilon}$,
$v \mapsto (v; \wpred{\lambda}{v}{\bullet})$ is once continuously differentiable in a neighborhood of $(v, 0_{\mathcal{H}})$.
It follows that for every $v\in \Omega_{\epsilon}$,
$v \mapsto \wpred{\lambda}{v}{\bullet}$ is once continuously differentiable in some neighborhood of $v$.
Therefore, $v \mapsto \wpred{\lambda}{v}{\bullet}$ is once continuously differentiable over $\Omega_{\epsilon}$.

Assuming further $\eqref{asm.loss.c3}$ implies that
$G$ is now twice continuously differentiable.
A similar argument form above is used to prove
that $v \mapsto \wpred{\lambda}{v}{\bullet}$ is twice continuously differentiable over $\Omega_{\epsilon}$.
\end{proof}

\medskip

\begin{lemma}
\label{lm.DD.predictor.func}
    Assume \eqref{asm.loss.convex}, 
    \eqref{asm.loss.c3},
    \eqref{asm.loss.smooth}, and
    \eqref{asm.finite.kern} hold true.
    For every $\epsilon >0$
    fulfilling Eq.~\eqref{eq.epsilon.small}
    and Eq.~\eqref{eq.epsilon.small.2},
    almost-surely, and 
    for every $v\in \Omega_{\epsilon}$,
    \begin{align*}
        \mathcal{D}^2_1 \wpred{\lambda}{v}{\bullet}
        (\uw - \uu, \uw - \uu)
        &
        =
        - 
        \left[
            \DDtwo \wer{\lambda}{v}{\wpred{\lambda}{v}{\bullet}}
        \right]^{+}
        \left[
            \mathcal{D}_2^3 \wer{\lambda}{v}{\wpred{\lambda}{v}{\bullet}}
        \right]
        \left(h(v), h(v)
        \right)
        \\
        &
        \quad- 2 
        \left[
            \DDtwo \wer{\lambda}{v}{\wpred{\lambda}{v}{\bullet}}
        \right]^{+}
        \left[
            \DDtwo \wer{\lambda}{\uw - \uu}{\wpred{\lambda}{v}{\bullet}}
        \right]
        h(v),
    \end{align*}
    where $\uu := (1, \ldots, 1, 1, 0)$ and $\uw := (1, \ldots, 1, 0, 1)$ and
    \begin{align*}
        h(v) := \mathcal{D}_1 \wpred{\lambda}{v}{\bullet} (\uw - \uu).
    \end{align*}
\end{lemma}
\begin{proof}
    Let $\epsilon >0$ and $\sigma \in (0, 1)$
    fulfilling Eq.~\eqref{eq.epsilon.small}
    and Eq.~\eqref{eq.epsilon.small.2}.
        Since \eqref{asm.loss.convex}, 
        \eqref{asm.loss.c3},
        \eqref{asm.loss.smooth}, and
        \eqref{asm.finite.kern} hold true,
        Lemma~\ref{lm.pred.func.c2} establishes
        that the function $v \mapsto \hat{f}_{\lambda}\paren{v; \bullet}$
        is twice continuously differentiable
        over the open set $\Omega_{\epsilon}$ (see Definition~\ref{def.epsilon.neighborhood}).
        It follows that its second order differential is well-defined
        and the following aims to provide its expression.
    For every $v \in \Omega_{\epsilon}$,
    since $\wpred{\lambda}{v}{\bullet}$ is a minimizer of
    $\wer{\lambda}{v}{\bullet}$,
    $$\Dtwo \wer{\lambda}{v}{\wpred{\lambda}{v}{\bullet}} = 0_{\mathcal{H}}.$$
    Since $v \mapsto \Dtwo \wer{\lambda}{v}{\wpred{\lambda}{v}{\bullet}}$
    is constant over the open set $\Omega_{\epsilon}$,
    then taking the second directional differential at $v$ in the direction of $(\uw - \uu, \uw - \uu)$, 
    \begin{align*}
        0_{\mathcal{H}} 
        &
        =
        \mathbf{D}^2_{v} 
        \left[
            \Dtwo \wer{\lambda}{v}{\wpred{\lambda}{v}{\bullet}}
        \right]
        (\uw - \uu, \uw - \uu)
        \\
        &
        =
        \mathbf{D}_{v} 
        \left(
            \mathbf{D}_{v}
            \left[
                \Dtwo \wer{\lambda}{v}{\wpred{\lambda}{v}{\bullet}}
            \right]
            (\uw - \uu)
        \right)
        (\uw - \uu)
        \\
        &
        =
        \mathbf{D}_{v}
        \left(
            \DDtwo \wer{\lambda}{v}{\wpred{\lambda}{v}{\bullet}}
            \mathcal{D}_1 \wpred{\lambda}{v}{\bullet}
            (\uw - \uu)
            + \Dtwo \wer{\lambda}{\uw - \uu}{\wpred{\lambda}{v}{\bullet}}
        \right)
        (\uw - \uu)
        \\
        &
        =
        \DDtwo \wer{\lambda}{v}{\wpred{\lambda}{v}{\bullet}}
        \mathcal{D}^2_1 \wpred{\lambda}{v}{\bullet}
        (\uw - \uu, \uw - \uu)
        \\
        &
        \quad+
        \mathcal{D}_2^3 \wer{\lambda}{v}{\wpred{\lambda}{v}{\bullet}}
        \mathcal{D}_1 \wpred{\lambda}{v}{\bullet}
        \mathcal{D}_1 \wpred{\lambda}{v}{\bullet}
        (\uw - \uu)(\uw - \uu)
        \\
        &
        \quad+\mathcal{D}_1
        \DDtwo \wer{\lambda}{v}{\wpred{\lambda}{v}{\bullet}}
        \mathcal{D}_1 \wpred{\lambda}{v}{\bullet}
        (\uw - \uu)(\uw - \uu)
        \\
        &
        \quad+
        \DDtwo \wer{\lambda}{\uw - \uu}{\wpred{\lambda}{v}{\bullet}}
        \mathcal{D}_1 \wpred{\lambda}{v}{\bullet}(\uw - \uu)
        \\
        &
        =
        \DDtwo \wer{\lambda}{v}{\wpred{\lambda}{v}{\bullet}}
        \mathcal{D}^2_1 \wpred{\lambda}{v}{\bullet}
        (\uw - \uu, \uw - \uu)
        \\
        &
        \quad+
        \mathcal{D}_2^3 \wer{\lambda}{v}{\wpred{\lambda}{v}{\bullet}}
        \left( 
            \mathcal{D}_1 \wpred{\lambda}{v}{\bullet} (\uw - \uu),
            \mathcal{D}_1 \wpred{\lambda}{v}{\bullet} (\uw - \uu)
        \right)
        \\
        &
        \quad+ 2 \DDtwo \wer{\lambda}{\uw - \uu}{\wpred{\lambda}{v}{\bullet}}
        \mathcal{D}_1 \wpred{\lambda}{v}{\bullet}(\uw - \uu).
    \end{align*}
    Putting the last two terms in the other side of the equation, it results
    \begin{align*}
        \DDtwo \wer{\lambda}{v}{\wpred{\lambda}{v}{\bullet}}
        \mathcal{D}^2_1 \wpred{\lambda}{v}{\bullet}
        (\uw - \uu, \uw - \uu)
        &
        =
        -
        \mathcal{D}_2^3 \wer{\lambda}{v}{\wpred{\lambda}{v}{\bullet}}
        \left( 
            h(v), h(v)
        \right)
        \\
        &
        \quad- 2 \DDtwo \wer{\lambda}{\uw - \uu}{\wpred{\lambda}{v}{\bullet}} h(v).
    \end{align*}
The conclusion results from noticing that $\DDtwo \wer{\lambda}{v}{\wpred{\lambda}{v}{\bullet}}$ is invertible over $\mathcal{A}$ by Lemma~\ref{lm.hess.invertible.2}.
\end{proof}

\medskip

\begin{lemma}
\label{lm.bound.D3.wer}
Assume \eqref{asm.loss.c3}, \eqref{asm.loss_bounded_d3} hold true. 
Then for every $v  \in [\uu, \uw]$ (with $\uu := (1, \ldots, 1, 1, 0)$ and $\uw := (1, \ldots, 1, 0, 1)$) and $f, h \in \mathcal{A}$ (see Notation~\ref{not.span}), it comes
    \begin{align*}
          \left\|
            \mathcal{D}_2^3 \wer{\lambda}{v}{f} (h, h)
          \right\|_{\mathcal{H}}
          \leq \xi_{\ell} \left(
            \frac{1}{n+1} 
            \sum_{i=1}^{n+1} \paren{K_{i, i}}^{3/2}
        \right) \|h\|^2_{\mathcal{H}}.
      \end{align*}
\end{lemma}
\begin{proof}
    Set $v  \in [\uu, \uw]$ and $f, h \in \mathcal{A}$.
    Then a straightforward application of \eqref{asm.loss_bounded_d3} implies that each third derivative is uniformly bounded by $\xi_\ell$, which entails
    \begin{align*}
        \left\| 
            \mathcal{D}_2^3 \wer{\lambda}{v}{f} (h, h)
        \right\|_{\mathcal{H}}
        &
        \leq
        \frac{1}{n+1}
        \sum_{i=1}^n 
        v_{i} 
        \abss{\dddtwo \ell\paren{Y_i, f\paren{X_i}}}
        h(X_i)^2 
        \left\| 
            K_{X_i}
        \right\|_{\mathcal{H}}
        \\
        &
        \quad+ 
        \frac{v_{n+1}}{n+1}
        \abss{\dddtwo \ell\paren{z, f\paren{X_{n+1}}}}
        h(X_{n+1})^2
        \left\| 
        K_{X_{n+1}}
        \right\|_{\mathcal{H}}
        \\
        &
        \quad
        + 
        \frac{v_{n+2}}{n+1}
        \abss{\dddtwo \ell\paren{y, f\paren{X_{n+1}}}}
        h(X_{n+1})^2 
        \left\| 
            K_{X_{n+1}}
        \right\|_{\mathcal{H}}
        \\
        &
        \leq
        \frac{1}{n+1}
        \sum_{i=1}^n 
        v_{i} \xi_{\ell}
        \|h\|^2_{\mathcal{H}} 
        \|K_{X_i}\|^3_{\mathcal{H}}
        \\
        &
        \quad+ 
        (v_{n+1} +v_{n+2}) \frac{1}{n+1}
        \xi_{\ell}
        \|h\|^2_{\mathcal{H}} \|K_{X_i}\|^3_{\mathcal{H}}
        \\
        &
        \leq
        \xi_{\ell} \|h\|^2_{\mathcal{H}}
        \frac{1}{n+1} 
        \sum_{i=1}^{n+1} \|K_{X_i}\|^3_{\mathcal{H}}
        \leq
        \xi_{\ell} \|h\|^2_{\mathcal{H}}
        \frac{1}{n+1} 
        \sum_{i=1}^{n+1} \paren{K_{i, i}}^{3/2},
    \end{align*}
    \sloppy where we also used the reproducing property and Cauchy-Schwarz inequality $ h(X_i)^2 \leq  \norm{h}^2_{\mathcal{H}} \norm{K_{X_i}}_{\mathcal{H}}^2 $, for any index $i$.
\end{proof}

\medskip

\begin{lemma}
    \label{lm.bound.D2.wer}
        Assume
        \eqref{asm.loss.c2}, 
        \eqref{asm.loss.smooth} hold true. Then
    for every $v  \in [\uu, \uw]$ 
    and every $f, h\in \mathcal{A}$,
    \begin{align*}
        \left\|
            \left[
                \DDtwo \wer{\lambda}{\uw - \uu}{f}
            \right] (h)
        \right\|_{\mathcal{H}}
        \leq 2 \beta_{\ell; 2}
        \frac{K_{n+1, n+1 }}{n+1}
        \normh{h}.
    \end{align*}    
\end{lemma}
\begin{proof}
    Let $v\in [\uu, \uw]$ and $f, h \in \mathcal{A}$, 
    \begin{align*}
        \left\|
            \left[
                \DDtwo \wer{\lambda}{\uw - \uu}{f}
            \right] (h)
        \right\|_{\mathcal{H}}
        &
        =
        \left\|
            - \frac{1}{n+1}
            \ddtwo \ell\paren{z, f\paren{X_{n+1}}}
            h(X_{n+1}) K_{X_{n+1}} 
        \right.
        \\
        &
        \qquad
        +
        \left.
            \frac{1}{n+1}
            \ddtwo \ell\paren{y, f\paren{X_{n+1}}}
            h(X_{n+1}) K_{X_{n+1}}
        \right\|_{\mathcal{H}}
        \\
        &
        =
        \frac{1}{n+1}
        \left| 
            \partial_2^2 \ell(z, \bullet)
            - \partial_2^2 \ell(y, \bullet)
        \right|
        |h(X_{n+1})|
        \left\|
            K_{X_{n+1}}
        \right\|_{\mathcal{H}}
        \\
        &
        \leq
        2 \beta_{\ell; 2}
        \frac{\|K_{X_{n+1}}\|^2}{n+1}
        \|h\|_{\mathcal{H}}
        \leq
        2 \beta_{\ell; 2}
        \frac{K_{n+1, n+1 }}{n+1}
        \|h\|_{\mathcal{H}}.
    \end{align*}
\end{proof}

\medskip

\begin{lemma}
    \label{lm.bound.h.v}
        Assume \eqref{asm.loss.convex},
        \eqref{asm.loss.c2} and
        \eqref{asm.loss.smooth} hold true. Then
    for every $v  \in [\uu, \uw]$
    \begin{align*}
        \normh{h \paren{v}}
        =
        \left\|
            \mathcal{D}_1 \wpred{\lambda}{v}{\bullet} (\uw - \uu)
        \right\|_{\mathcal{H}}
        \leq \sqrt{K_{n+1, n+1}} 
        \frac{\tilde{\rho}^{(1)}_{\lambda}(y)}{\lambda (n+1)},
    \end{align*}
    where
    \begin{align*}
        \tilde{\rho}^{(1)}_{\lambda}(y) :=
            \left(
                1 + K_{n+1, n+1} \frac{\beta_{\ell; 2}}{\lambda (n+1)}
            \right) \rho^{(1)}_{\lambda}(y).
    \end{align*}
\end{lemma}
\begin{proof}
    Let $v  \in [\uu, \uw]$.
    Similarly to Proposition~\ref{prop.predictor.stability.smooth},
    since \eqref{asm.loss.convex} and \eqref{asm.loss.c2} hold true,
    it follows that,    
    \begin{align}
    \label{eq.h.v}
        \normh{h \paren{v}}
        &=
        \left\|
            \left[
                \DDtwo \wer{\lambda}{v}{\wpred{\lambda}{v}{\bullet}}
            \right]^+
            \Dtwo \wer{\lambda}{\uw - \uu}{\wpred{\lambda}{v}{\bullet}}
        \right\|_{\mathcal{H}} \notag
        \\
        &\leq
        \frac{1}{2 \lambda}
        \normh{
            \Dtwo \wer{\lambda}{\uw - \uu}{\wpred{\lambda}{v}{\bullet}}
        },
      \end{align}
    where the inequality follows
    Lemma~\ref{lm.wer.strongly.convex.1} stating that
    $\wer{\lambda}{v}{\bullet}$ is $2\lambda$-strongly convex
    since $v \in \croch{\uu, \uw} \subset \mathbb{R}_{+}^{n+2}$.
    Consider the right-most term
    \begin{align*}
        \normh{\Dtwo \wer{\lambda}{\uw - \uu}{\wpred{\lambda}{v}{\bullet}}} 
        &
        =
        \left\|
        - \frac{1}{n+1}
        \partial_2 \ell(z, \hat{f}_{\lambda}(v; X_{n+1}))
        K_{X_{n+1}}
        \right.
        \\
        &
        \qquad
        \left.
        + \frac{1}{n+1}
        \partial_2 \ell(y, \hat{f}_{\lambda}(v; X_{n+1}))
        K_{X_{n+1}}
        \right\|_{\mathcal{H}}
        \\
        &
        \leq
        \frac{\|K_{X_{n+1}}\|_{\mathcal{H}}}{n+1}
        \left\|
        - \partial_2 \ell(z, \hat{f}_{\lambda}(v; X_{n+1}))
        + \partial_2 \ell(y, \hat{f}_{\lambda}(v; X_{n+1}))
        \right\|_{\mathcal{H}}
        \\
        &
        \leq
        \frac{\sqrt{K_{n+1, n+1}}}{n+1}
        \left\|
        - \partial_2 \ell(z, \hat{f}_{\lambda}(v; X_{n+1}))
        + \partial_2 \ell(y, \hat{f}_{\lambda}(v; X_{n+1}))
        \right\|_{\mathcal{H}}.
    \end{align*}
        Since \eqref{asm.loss.c2} holds true,
        it follows that $\partial_2\ell\paren{u, \bullet}$
        is once continuously differentiable.
    By the mean-value theorem, there exists 
    $\xi \in \croch{\hat{f}_{\lambda}(\uu; X_{n+1}),
        \hat{f}_{\lambda}(v; X_{n+1})
    }$ such that
    \begin{align*}
        \rho(v)
        &
        :=
        \left|
        - \partial_2 \ell(z, \hat{f}_{\lambda}(v; X_{n+1}))
        + \partial_2 \ell(y, \hat{f}_{\lambda}(v; X_{n+1}))
    \right|
    \\
        &
        \leq
        \left|
        - \partial_2 \ell(z, \hat{f}_{\lambda}(\uu; X_{n+1}))
        - \partial_2^2 \ell(z, \xi)
        \left(
            \hat{f}_{\lambda}(v; X_{n+1})
            - \hat{f}_{\lambda}(\uu; X_{n+1})
        \right)
        \right.
        \\
        &
        \qquad
        \left.
        + \partial_2 \ell(y, \hat{f}_{\lambda}(\uu; X_{n+1}))
        + \partial_2^2 \ell(y, \xi)
        \left(
            \hat{f}_{\lambda}(v; X_{n+1})
            - \hat{f}_{\lambda}(\uu; X_{n+1})
        \right)
        \right|
        \\
        &
        \leq
        \left|
            - \partial_2 \ell(z, \hat{f}_{\lambda}(\uu; X_{n+1}))
            + \partial_2 \ell(y, \hat{f}_{\lambda}(\uu; X_{n+1}))
        \right|
        \\
        &
        \quad
        +
        \left|
        \left(
        - \partial_2^2 \ell(z, \xi)
        + \partial_2^2 \ell(y, \xi)
        \right)
        (\wpred{\lambda}{v}{\bullet} - \wpred{\lambda}{\uu}{\bullet})(X_{n+1})
        \right|
        \\
        &
        \leq
        \left|
            - \partial_2 \ell(z, \hat{f}_{\lambda}(\uu; X_{n+1}))
            + \partial_2 \ell(y, \hat{f}_{\lambda}(\uu; X_{n+1}))
        \right|
        \\
        &
        \quad
        +
        \left|
        - \partial_2^2 \ell(z, \xi)
        + \partial_2^2 \ell(y, \xi)
        \right|
        \|
            \wpred{\lambda}{v}{\bullet} - \wpred{\lambda}{\uu}{\bullet}
        \|_{\mathcal{H}}
        \|
            K_{X_{n+1}}
        \|_{\mathcal{H}}
        \\
        &
        \leq
        \left|
            - \partial_2 \ell(z, \hat{f}_{\lambda}(\uu; X_{n+1}))
            + \partial_2 \ell(y, \hat{f}_{\lambda}(\uu; X_{n+1}))
        \right|
        +
        2 \beta_{\ell; 2} \sqrt{K_{n+1, n+1 }}
        \|
            \wpred{\lambda}{v}{\bullet} - \wpred{\lambda}{\uu}{\bullet}
    \|_{\mathcal{H}}.
    \end{align*}

    Going back to Eq.~\eqref{eq.h.v}
    \begin{align*}
        \normh{h\paren{v}}
        &
        \leq
        \frac{\sqrt{ K_{n+1, n+1} }}{2 \lambda (n+1)}
        \rho(v)
        \\
        &
        \leq
        \frac{\sqrt{ K_{n+1, n+1 }}}{2 \lambda (n+1)}
        \left|
            - \partial_2 \ell(z, \hat{f}_{\lambda}(\uu; X_{n+1}))
            + \partial_2 \ell(y, \hat{f}_{\lambda}(\uu; X_{n+1}))
        \right|
        \\
        &
        +
        \frac{\sqrt{ K_{n+1, n+1} }}{2 \lambda (n+1)}
        2 \beta_{\ell; 2} \sqrt{ K_{n+1, n+1 }}
        \normh{
            \wpred{\lambda}{v}{\bullet} - \wpred{\lambda}{\uu}{\bullet}
        }
        \\
        &
        \leq
        \sqrt{ K_{n+1, n+1} }
        \frac{\rho^{(1)}_{\lambda}(y)}{\lambda (n+1)}
        +
        \frac{\sqrt{ K_{n+1, n+1} }}{\lambda (n+1)}
        \beta_{\ell; 2} \sqrt{ K_{n+1, n+1 }}
        \sqrt{ K_{n+1, n+1} } 
        \frac{\rho^{(1)}_{\lambda}(y)}{\lambda (n+1)}
        \\
        &
        \leq
        \sqrt{ K_{n+1, n+1} } \frac{\tilde{\rho}^{(1)}_{\lambda}(y)}{\lambda (n+1)},
    \end{align*}
    where the second inequality follows from Lemma~\ref{lm.local.predictor.stability}
    which holds true under \eqref{asm.loss.convex} and \eqref{asm.loss.c2}, and
    $\tilde{\rho}^{(1)}_{\lambda}(y)$ is given by Eq.~\eqref{eq.rho.tilde}.
\end{proof}

\subsection{Proofs of the Influence function part}
\label{sec.IF.proofs}

\begin{proof}
	Let $y \in \mathcal{Y}$, and $z \in \mathcal{Y}$,
	and $\lambda \in \paren{0, +\infty}$.
	By definition of the estimator function $T_{\lambda}\paren{\bullet}$
	and the empirical distribution $\widehat{P}_{n+1}^{z}$, and
	recalling that the vector $\mathbf{u} = \paren{1, \ldots, 1, 1, 0} \in \mathbb{R}^{n+2}$
	\begin{align*}
		T_{\lambda}\paren{\widehat{P}_{n+1}^{z}}
		&
		\in
		\argmin_{f \in \mathcal{Y}}
		\mathbb{E}_{
			\paren{X, Y}
			\sim
			\widehat{P}_{n+1}^{z}
		}\croch{
			\ell\paren{Y, f\paren{X}}
		}
		+ \lambda \normh{f}^2
		\\
		&
		=
		\argmin_{f \in \mathcal{Y}}
		\frac{1}{n+1}
		\ell\paren{z, f\paren{X_{n+1}}}
		+
		\sum_{i=1}^{n}
		\frac{1}{n+1}
		\ell\paren{Y_i, f\paren{X_i}}
		+ \lambda \normh{f}^2
		\\
		&
		=
		\argmin_{f \in \mathcal{Y}}
		\hat{\mathbf{R}}_{\lambda}\paren{\mathbf{u}; f}
		\\
		&
		=
		\brac{
			\hat{f}_{\lambda}\paren{\mathbf{u}; \bullet}
		},
	\end{align*}
	where the first equality follows from the definition of the expectation, and
	the second equality follows from the definition of the weighted regularized
	empirical-risk function, and
	the last equality follows the uniqueness of the minimizer
	of the aforementioned function which holds
	under \eqref{asm.loss.convex} and \eqref{asm.loss.lsc} (implied by \eqref{asm.loss.c2})
	by Lemma~\ref{lm.predictor.func.well.defined.1}.
	It follows that, under \eqref{asm.loss.convex} and \eqref{asm.loss.c2},
	$T_{\lambda}\paren{\widehat{P}_{n+1}^{z}} = \hat{f}_{\lambda}\paren{\mathbf{u}; \bullet}$.

	Let $t \in \paren{0, 1}$.
	By definition of the empirical distributions $\widehat{P}_{n+1}^{z}$
	and $\widehat{P}_{n+1}^{y}$,
	\begin{align*}
		\paren{1 - t} \widehat{P}_{n+1}^{z}
		+ t \widehat{P}_{n+1}^{y}
		&
		=
		\paren{1 - t}\frac{1}{n+1}
		\delta_{\paren{X_{n+1}, z}}
		+ t\frac{1}{n+1}
		\delta_{\paren{X_{n+1}, y}}
		+\sum_{i=1}^{n}
		\frac{1}{n+1}
		\delta_{\paren{X_i, Y_i}}.
	\end{align*}
	Recalling that the vector $\mathbf{w} = \paren{1, \ldots, 0, 1, 0} \in \mathbb{R}^{n+2}$,
	it follows that
	\begin{align*}
		T_{\lambda}\paren{
			\paren{1 - t} \widehat{P}_{n+1}^{z}
			+ t \widehat{P}_{n+1}^{y}
		}
		&
		\in
		\argmin_{f \in \mathcal{H}}
		\mathbb{E}_{\paren{X, Y}\sim \paren{
				\paren{1 - t} \widehat{P}_{n+1}^{z}
				+ t \widehat{P}_{n+1}^{y}
		}}\croch{
			\ell\paren{Y, f\paren{X}}
		}
		+ \lambda \normh{f}^2
		\\
		&
		=
		\argmin_{f \in \mathcal{H}}
		\paren{1 - t}
		\frac{1}{n+1}
		\ell\paren{z, f\paren{X_{n+1}}}
		+
		t
		\frac{1}{n+1}
		\ell\paren{y, f\paren{X_{n+1}}}
		\\
		&
		\qquad
		\qquad
		\qquad
		+
		\sum_{i=1}^{n}
		\frac{1}{n+1}
		\ell\paren{Y_i, f\paren{X_i}}
		+ \lambda \normh{f}^2
		\\
		&
		=
		\argmin_{f \in \mathcal{H}}
		\hat{\mathbf{R}}_{\lambda}\paren{
			\paren{1 - t}\mathbf{u} + t\mathbf{w}
		}
		\\
		&
		=
		\brac{
			\hat{f}_{\lambda}\paren{
				\paren{1 - t}\mathbf{u} + t\mathbf{w};
				\bullet
			}
		},
	\end{align*}
	where the first equality follows from the definition of the expectation, and
	the second equality follows from the definition of the weighted regularized
	empirical-risk function, and
	the last equality follows the uniqueness of the minimizer
	of the aforementioned function which holds
	under \eqref{asm.loss.convex} and \eqref{asm.loss.lsc} (implied by \eqref{asm.loss.c2})
	by Lemma~\ref{lm.predictor.func.well.defined.1}.
	It follows that
	$$T_{\lambda}\paren{
		\paren{1 - t} \widehat{P}_{n+1}^{z}
		+ t \widehat{P}_{n+1}^{y}
	}
	=
	\hat{f}_{\lambda}\paren{
		\paren{1 - t}\mathbf{u} + t\mathbf{w};
		\bullet
	}
	=
	\hat{f}_{\lambda}\paren{
		\mathbf{u} + t \paren{
			\mathbf{w} - \mathbf{u}
		};
		\bullet
	}.$$
	
	Let $\epsilon >0$ fulfilling Conditions~\eqref{eq.epsilon.small}
	and~\eqref{eq.epsilon.small.2}.
	Under Assumptions~\eqref{asm.finite.kern},
	\eqref{asm.loss.smooth}, \eqref{asm.loss.c2},
	and \eqref{asm.loss.convex},
	Lemma~\ref{lm.pred.func.c2} establishes
	that the predictor function $v \mapsto \hat{f}_{\lambda}\paren{v; \bullet}$
	is once continuously differentiable over an open set $\Omega_{\epsilon}$
	covering the segment $\croch{\mathbf{u}, \mathbf{w}}$.
	By noting $\mathcal{D}_1\hat{f}_{\lambda}\paren{
		\mathbf{u}; \bullet
	}\paren{\mathbf{w} - \mathbf{u}}$,
	the Fr{\'e}chet differential of the function $v \mapsto \hat{f}_{\lambda}\paren{
		v; \bullet
	}$ at the weight vector $\mathbf{u}$
	in the direction $\paren{\mathbf{w} - \mathbf{u}}$.
	It follows that
	\begin{align*}
		&
		\normh{
			\frac{
				T_{\lambda}\paren{
					\paren{1 - t}\widehat{P}_{n+1}^{z}
					+ t\widehat{P}_{n+1}^{y}
				} - T_{\lambda}\paren{\widehat{P}_{n+1}^{z}}
			}{
				t   
			} - \mathcal{D}_1\hat{f}_{\lambda}\paren{
				\mathbf{u}; \bullet
			}\paren{\mathbf{w} - \mathbf{u}}
		}
		\\
		&
		=
		\frac{ 
			\normh{
				\hat{f}_{\lambda}\paren{
					\mathbf{u} + t\paren{\mathbf{w} - \mathbf{u}};
					\bullet
				} - \hat{f}_{\lambda}\paren{\mathbf{u}; \bullet}
				- \mathcal{D}_1\hat{f}_{\lambda}\paren{
					\mathbf{u}; \bullet
				}\paren{
					t\paren{\mathbf{w} - \mathbf{u}}
				}
			}
		}{
			t
		}
		\\
		&
		=
		\frac{
			o\paren{\normh{
					t\paren{\mathbf{w} - \mathbf{u}}
			}}
		}{t}
		= 
		\frac{o\paren{t}}{t} \xrightarrow{t \rightarrow 0^{+}} 0,
	\end{align*}
	where the first equality follows from the last three equality
	proven earlier, and
	the second equality follows from the definition
	of the Fr{\'e}chet differential and the little o notation.
	By definition of the Bouligand influence function
	$\mathrm{BIF}\paren{\widehat{P}_{n+1}^{y}; T_{\lambda}, \widehat{P}_{n+1}^{z}}$
	of the estimator function $T_{\lambda}\paren{\bullet}$
	for the empirical distribution $\widehat{P}_{n+1}^{z}$
	in the direction of the empirical distribution $\widehat{P}_{n+1}^{y}$,
	the above result establishes that under Assumptions~\eqref{asm.finite.kern},
	\eqref{asm.loss.smooth}, \eqref{asm.loss.c2},
	and \eqref{asm.loss.convex},
	it exists and is given by
	\begin{align*}
		\mathrm{BIF}\paren{\widehat{P}_{n+1}^{y}; T_{\lambda}, \widehat{P}_{n+1}^{z}}
		= \mathcal{D}_1\hat{f}_{\lambda}\paren{
			\mathbf{u}; \bullet
		}\paren{\mathbf{w} - \mathbf{u}}.
	\end{align*}
	
    \bigskip{}
    \noindent{\emph{Expression.}}
	Under Assumptions~\eqref{asm.finite.kern},
	\eqref{asm.loss.smooth}, \eqref{asm.loss.c2},
	and \eqref{asm.loss.convex},
	Lemma~\ref{lm.predictor.func.well.defined.2} establishes
	that for every $v\in \Omega_{\epsilon}$, $\wpred{\lambda}{v}{\bullet}$ is a minimizer of 
	$\wer{\lambda}{v}{\bullet}$, then
	the representor of Fr{\'e}chet differential $\Dtwo \wer{\lambda}{v}{\wpred{\lambda}{v}{\bullet}}$
	of $\wer{\lambda}{v}{\bullet}$
	at $\wpred{\lambda}{v}{\bullet}$ is the constant and equal to the function $0_{\mathcal{H}}$.

	In the following, let us consider the function
	$v \in \Omega_{\epsilon} \mapsto \Dtwo \wer{\lambda}{v}{\wpred{\lambda}{v}{\bullet}} \in \mathcal{H}$
	where the output is the representation of the Fr{\'e}chet differential
	(which is a bounded linear operator) by a function in
	the hypothesis space $\mathcal{H}$.
	Since this function is constant over the open set $\Omega_{\epsilon}$,
	then for every $v \in \Omega_{\epsilon}$,
	its Fr{\'e}chet differential $\mathbf{D}_{v} \left[
	\Dtwo \wer{\lambda}{v}{\wpred{\lambda}{v}{\bullet}}
	\right] (\uw - \uu)$ at $v \in \Omega_{\epsilon}$
	along the direction $\uw - \uu$ is equal to zero. It follows that
	\begin{align*}
		0_{\mathcal{H}} 
		&
		=
		\mathbf{D}_{v} \left[
		\Dtwo \wer{\lambda}{v}{\wpred{\lambda}{v}{\bullet}}
		\right] (\uw - \uu)
		\\
		&
		=
		\left[
		\mathcal{D}_2
		\Dtwo \wer{\lambda}{v}{\wpred{\lambda}{v}{\bullet}}
		\right]
		\mathcal{D}_1 \wpred{\lambda}{v}{\bullet}
		(\uw - \uu)
		+
		\left[
		\mathcal{D}_1 \Dtwo \wer{\lambda}{v}{\wpred{\lambda}{v}{\bullet}}
		\right]
		(\uw - \uu)
		\\
		&
		=
		\DDtwo \wer{\lambda}{v}{\wpred{\lambda}{v}{\bullet}}
		\mathcal{D}_1 \wpred{\lambda}{v}{\bullet} (\uw - \uu)
		+ \Dtwo \wer{\lambda}{\uw - \uu}{\wpred{\lambda}{v}{\bullet}},
	\end{align*}
	where the second equality follows from the application of
	the chain rule with $\mathcal{D}_2$ denoting the Fr{\'e}chet differential with respect to the second argument
	and $\mathcal{D}_1$ w.r.t. the first argument, and
	last equality follows the definition of $\DDtwo \wer{\lambda}{v}{\bullet}$
	and the linearity of $\Dtwo \wer{\lambda}{v}{\bullet}$
	with respect to $v$ as described in Lemma~\ref{lm.first.diff.rer}.
	Putting the gradient term on the other side of the equation
	\begin{align*}
		\DDtwo \wer{\lambda}{v}{\wpred{\lambda}{v}{\bullet}}
		\mathcal{D}_1 \wpred{\lambda}{v}{\bullet} (\uw - \uu)
		= - \Dtwo \wer{\lambda}{\uw - \uu}{\wpred{\lambda}{v}{\bullet}}.
	\end{align*}
	Under \eqref{asm.loss.c2},
	the right hand side of the above equation is an element of $\mathcal{A}$
	(see Lemma~\ref{lm.first.diff.rer}).
	Moreover, since for every $v \in \Omega_{\epsilon}$,
	$\wpred{\lambda}{v}{\bullet} \in \mathcal{A}$, it follows
	that $\mathcal{D}_1 \wpred{\lambda}{v}{\bullet} (\uw - \uu) \in \mathcal{A}$.

	Under Assumptions~\eqref{asm.finite.kern},
	\eqref{asm.loss.smooth}, \eqref{asm.loss.c2},
	and \eqref{asm.loss.convex},
	Lemma~\ref{lm.hess.invertible.2} establishes
	the second order differential $\DDtwo \wer{\lambda}{v}{\wpred{\lambda}{v}{\bullet}}$
	of the weighted regularized empirical-risk function $\wer{\lambda}{v}{\wpred{\lambda}{v}{\bullet}}$
	seen as a linear operator from $\mathcal{H}$ to $\mathcal{H}$
	is invertible over $\mathcal{A}$. It follows that
	\begin{align*}
		\mathcal{D}_1 \wpred{\lambda}{v}{\bullet} (\uw - \uu)
		= - \croch{\DDtwo \wer{\lambda}{v}{\wpred{\lambda}{v}{\bullet}}}^{+} 
		\Dtwo \wer{\lambda}{\uw - \uu}{\wpred{\lambda}{v}{\bullet}}.
	\end{align*}
	Finally taking $v= \uu$
	and since $\uw - \uu = (0, \ldots, 0,-1, 1)$
	\begin{align*}
		\mathcal{D}_1 \wpred{\lambda}{\uu}{\bullet} (\uw - \uu) 
		= - \ifunc{z} + \ifunc{y} \in \mathcal{A}
	\end{align*} 
	where for every $z^{\prime} \in \mathcal{Y}$
	\begin{align*}
		\ifunc{z^{\prime}} 
		&
		:= 
		- \frac{1}{n+1}
		\dtwo \ell\paren{z^{\prime},
			\hat{f}_{\lambda}\paren{\mathbf{u}; X_{n+1}}}
		\left[
		\DDtwo \wer{\lambda}{\uu}{
			\hat{f}_{\lambda}\paren{\mathbf{u}; \bullet}
		}
		\right]^{+} 
		K_{X_{n+1}}
		\\
		&
		= - \frac{1}{n+1}
		\dtwo \ell\paren{z^{\prime},
			\hat{f}_{\lambda; D^{z}}\paren{X_{n+1}}}
		\left[
		\DDtwo \wer{\lambda}{\uu}{
			\hat{f}_{\lambda; D^{z}}
		}
		\right]^{+} 
		K_{X_{n+1}}.
	\end{align*}
\end{proof}

\begin{proof}
	Starting for every $z^{\prime} \in \mathcal{Y}$ from
	\begin{align*}
		\ifunc{z^{\prime}} 
		&
		:= 
		- \frac{1}{n+1}
		\dtwo \ell\paren{z^{\prime},
			\hat{f}_{\lambda}\paren{\mathbf{u}; X_{n+1}}}
		\left[
		\DDtwo \wer{\lambda}{\uu}{
			\hat{f}_{\lambda}\paren{\mathbf{u}; \bullet}
		}
		\right]^{+} 
		K_{X_{n+1}}
		\\
		&
		= - \frac{1}{n+1}
		\dtwo \ell\paren{z^{\prime},
			\hat{f}_{\lambda; D^{z}}\paren{X_{n+1}}}
		\left[
		\DDtwo \wer{\lambda}{\uu}{
			\hat{f}_{\lambda; D^{z}}
		}
		\right]^{+} 
		K_{X_{n+1}},
	\end{align*}
	$\ifunc{z^{\prime}}$ is the solution
	$g \in \mathcal{A}$ of the following equation
	\begin{align*}
		\DDtwo \wer{\lambda}{\uu}{\pred{\lambda;}{z}} g
		= - \frac{1}{n+1}
		\partial_2 \ell\paren{z^{\prime}, \pred{\lambda;}{z}(X_{n+1})}
		K_{X_{n+1}}.
	\end{align*}
	
	Under \eqref{asm.loss.convex} and \eqref{asm.loss.lsc} (implied by \eqref{asm.loss.c2}),
    Lemma~\ref{lm.predictor.func.well.defined.1}
	establishes that there exists a vector $\hat{a}_{\lambda}\paren{\mathbf{u}} \in \mathbb{R}^{n+1}$
	such that $\pred{\lambda;}{z} = \hat{f}_{\lambda}\paren{\mathbf{u}; \bullet}
	= \sum_{i=1}^{n+1} \croch{\hat{a}_{\lambda}\paren{\mathbf{u}}}_i K_{X_i}$.
	Moreover, since $g \in \mathcal{A}$, there exists a vector $b \in \mathbb{R}^{n+1}$
	such that $g = \sum_{i=1}^{n+1} b_i K_{X_i}$.
	
	Evaluating both sides of the equation at $X_j$ for $1\leq j \leq n+1$
	yields the following the system of equation
	\begin{align*}
		\dotprod{\DDtwo \wer{\lambda}{\uu}{\pred{\lambda;}{z}} g}{K_{X_{j}}}{\mathcal{H}}
		=- \frac{1}{n+1}
		\partial_2 \ell(z^{\prime}, \pred{\lambda;}{z}(X_{n+1}))
        \dotprod{
		K_{X_{n+1}}}{K_{X_{j}}}{\mathcal{H}}
		\quad
		\mbox{if }
		1 \leq j \leq n+1,
	\end{align*}
	which under \eqref{asm.loss.c2},
	Lemma~\ref{lm.link.between.Hessians} is equivalent to the following linear equation
	\begin{align*}
		\nabla^2_2\hat{R}_{\lambda}
		\paren{\mathbf{u}; \hat{a}_{\lambda}\paren{\mathbf{u}}} b
		= - \frac{1}{n+1}
		\partial_2 \ell(z^{\prime}, (\hat{a}_{\lambda}(\uu))^T K_{\bullet, n+1})
		K_{\bullet, n+1}.
	\end{align*}
	The right hand side of the equation is an element of the range of the gram matrix
	since it is a multiple a column of said matrix and said matrix is symmetric.
	Under \eqref{asm.loss.convex},
    \eqref{asm.loss.c2},
    \eqref{asm.loss.smooth}, and
    \eqref{asm.finite.kern} ,
	Lemma~\ref{lm.vec.wer.strongly.convex.2}
	ensures that $\nabla^2_2\hat{R}_{\lambda}
	\paren{v; \hat{a}_{\lambda}\paren{\mathbf{u}}}$ is invertible over $\mathcal{R}\paren{K}$.
	It follows that
	the vector $b$ that solves the above equation has the following expression
	\begin{align*}
		b = - \frac{1}{n+1}
		\partial_2 \ell\paren{z^{\prime}, \hat{a}_{\lambda}(\uu)^T  K_{\bullet, n+1} } 
		\croch{\nabla_2^2 \vwer{\lambda}{\uu}{\hat{a}_{\lambda}(\uu)}}^+
		K_{\bullet, n+1}.
	\end{align*}
	Therefore, for every $z^{\prime} \in \mathcal{Y}$,
	$\ifunc{z^{\prime}} = \sum_{i=1}^{n+1} \croch{I_{\hat{f}}\paren{X_{n+1}, z^{\prime}}}_i K_{X_i}$
	where 
	\begin{align*}
		I_{\hat{f}}\paren{X_{n+1}, z^{\prime}}
		:= - \frac{1}{n+1}
		\partial_2 \ell\paren{z^{\prime}, \hat{a}_{\lambda}(\uu)^T  K_{\bullet, n+1} } 
		\croch{\nabla_2^2 \vwer{\lambda}{\uu}{\hat{a}_{\lambda}(\uu)}}^+
		K_{\bullet, n+1}.
	\end{align*}	
\end{proof}

\subsection{Proofs of the approximate predictor/score accuracy part}
\label{sec.approximate.proofs}

\begin{proof}
Let's take $\epsilon >0$ fulfilling Eq.~\eqref{eq.epsilon.small} and Eq.~\eqref{eq.epsilon.small.2}.

On the one hand, let us assume \eqref{asm.loss.convex}, 
\eqref{asm.loss.c3},
\eqref{asm.loss.smooth},
and \eqref{asm.finite.kern} hold true.
Then, Lemma~\ref{lm.pred.func.c2} implies that 
$v\mapsto \wpred{\lambda}{v}{\bullet}$ is twice continuously differentiable on $\Omega_{\epsilon}$.
Then applying the mean-value theorem leads to
\begin{align*}
	\left\|
	\hat{f}_{\lambda; D^{y}} 
	- \tilde{f}_{\lambda; D^{y}}^{\mathrm{IF}}
	\right\|_{\mathcal{H}}
	&
	=
	\left\|
	\wpred{\lambda}{\uw}{\bullet} 
	- \left(
	\wpred{\lambda}{\uu}{\bullet}
	+ \mathcal{D}_1 \wpred{\lambda}{\uu}{\bullet}(\uw - \uu)
	\right)
	\right\|_{\mathcal{H}}
	\\
	&
	\leq
	\frac{1}{2}
	\sup_{v \in [\uu, \uw]}
	\left\|
	\mathcal{D}^2_1 \wpred{\lambda}{v}{\bullet} (\uw - \uu, \uw - \uu)
	\right\|_{\mathcal{H}}.
\end{align*}
Let $v  \in [\uu, \uw]$.
From \eqref{asm.loss.convex}, 
    \eqref{asm.loss.c3},
    \eqref{asm.loss.smooth},
    \eqref{asm.loss_bounded_d3}, and \eqref{asm.finite.kern} combined with Lemma~\ref{lm.DD.predictor.func}, it comes that
\begin{align*}
	&\left\|
	\mathcal{D}^2_1 \wpred{\lambda}{v}{\bullet} 
	(\uw - \uu, \uw - \uu)
	\right\|_{\mathcal{H}}
	\\
	\leq
	&
	\left\|
	\DDtwo \wer{\lambda}{v}{\wpred{\lambda}{v}{\bullet}}^+
	\right\|_{\mathrm{op}}
	\left\|
	\mathcal{D}_2^3 \wer{\lambda}{v}{\wpred{\lambda}{v}{\bullet}}
	\left(h\paren{v}, h\paren{v}
	\right)
	\right\|_{\mathcal{H}}
	\\
	&+ 2 \left\|
	\DDtwo \wer{\lambda}{v}{\wpred{\lambda}{v}{\bullet}}^+
	\right\|_{\mathrm{op}}
	\left\|
	\DDtwo \wer{\lambda}{\uw - \uu}{\wpred{\lambda}{v}{\bullet}}
	\right\|_{\mathrm{op}}
	\left\|
	h\paren{v}
	\right\|_{\mathcal{H}}
	\\
	\leq
	&\frac{1}{2 \lambda}
	\xi_{\ell} \left(\frac{1}{n+1} \sum_{i=1}^{n+1} \paren{K_{i, i}}^{3/2}\right)
	\|h\paren{v}\|_{\mathcal{H}}^2
	+ 2 
	\frac{1}{2 \lambda}
	2 \beta_{\ell; 2}
	\frac{K_{n+1, n+1}}{
		n+1
	}
	\left\|
	h\paren{v}
	\right\|_{\mathcal{H}}
	\\
	\leq
	&\frac{1}{2 \lambda}
	\xi_{\ell} \left(\frac{1}{n+1} \sum_{i=1}^{n+1} \paren{K_{i, i}}^{3/2}\right)
	\frac{K_{n+1, n+1} (\tilde{\rho}^{(1)}_{\lambda}(y))^2}{\lambda^2 (n+1)^2} + 2 
	\frac{1}{2 \lambda}
	2 \beta_{\ell; 2}
	\frac{K_{n+1, n+1} }{
		n+1
	}
	\frac{\sqrt{ K_{n+1, n+1} } \tilde{\rho}^{(1)}_{\lambda}(y)}{\lambda (n+1)}
	\\
	\leq
	&\sqrt{ K_{n+1, n+1} }
	\frac{\rho^{(2)}_{\lambda}(y)}{
		\lambda^3 (n+1)^2
	},
\end{align*}
    where the second inequality follows
    Lemma~\ref{lm.wer.strongly.convex.1} stating that
    $\wer{\lambda}{v}{\bullet}$ is $2\lambda$-strongly convex
    since $v \in \croch{\uu, \uw} \subset \mathbb{R}_{+}^{n+2}$
    and from Lemma~\ref{lm.bound.D3.wer} and
    Lemma~\ref{lm.bound.D2.wer}, and the third inequality follows from Lemma~\ref{lm.bound.h.v}, and finally
\begin{align*}
	\rho^{(2)}_{\lambda}(y) 
	:= \frac{\xi_{\ell}}{2} 
	\sqrt{ K_{n+1, n+1 }}
	\left(
	\frac{1}{n+1} 
	\sum_{i=1}^{n+1} 
	\paren{ K_{i, i} }^{3/2}
	\right)
	\paren{
		\tilde{\rho}^{(1)}_{\lambda}(y)
	}^2
	+ 2 \lambda ( K_{n+1, n+1 }) \beta_{\ell; 2} \tilde{\rho}^{(1)}_{\lambda}(y).
\end{align*}
where $\tilde{\rho}^{(1)}_{\lambda}(y)$ is defined in Eq.~\eqref{eq.rho.tilde}.

On the other hand,
\begin{align*}
	\left\|
	\hat{f}_{\lambda; D^{y}}
	- \tilde{f}_{\lambda; D^{y}}^{\mathrm{IF}}
	\right\|_{\mathcal{H}}
	&
	\leq
	\left\|
	\wpred{\lambda}{\uw}{\bullet}
	- \wpred{\lambda}{\uu}{\bullet}
	+ \mathcal{D}_1 \wpred{\lambda}{\uu}{\bullet}(\uw - \uu)
	\right\|
	\\
	&
	\leq
	\left\|
	\wpred{\lambda}{\uw}{\bullet}
	- \wpred{\lambda}{\uu}{\bullet}
	\right\|
	+ \left\|
	\mathcal{D}_1 \wpred{\lambda}{\uu}{\bullet}(\uw - \uu)
	\right\|
	\\
	&
	\leq
	\sqrt{ K_{n+1, n+1} } \frac{\rho^{(1)}_{\lambda}(y)}{\lambda (n+1)}
	+ \left\|
	\DDtwo \wer{\lambda}{\uu}{\wpred{\lambda}{\uu}{\bullet}}^{+} 
	\Dtwo \wer{\lambda}{\uw - \uu}{\wpred{\lambda}{\uu}{\bullet}}
	\right\|_{\mathcal{H}}
	\\
	&
	\leq
	2 \sqrt{ K_{n+1, n+1} } \frac{\rho^{(1)}_{\lambda}(y)}{\lambda (n+1)},
\end{align*}
where the last inequality follows from Lemma~\ref{lm.wer.strongly.convex.1}.
If we combine the two previous upper-bounds, it comes
\begin{align*}
	\left\|
	\hat{f}_{\lambda; D^{y}}
	- \tilde{f}_{\lambda; D^{y}}^{\mathrm{IF}}
	\right\|_{\mathcal{H}}
	\leq\sqrt{ K_{n+1, n+1 }} \min \paren{
		\frac{\rho^{(2)}_{\lambda}(y)}{
			\lambda^3 (n+1)^2},
		\frac{2 \rho^{(1)}_{\lambda}(y)}{\lambda (n+1)}
	}.
\end{align*}	
\end{proof}

\medskip

\begin{proof}(Proof of Theorem~\ref{thm.score.bound.very.smooth})
\label{proof.score.bound.very.smooth}
%
For every $i \in \brac{1, \ldots, n}$,
from the previous bound and under \eqref{asm.score.lipschitz},
the non-conformity score approximation quality is bounded from above
\begin{align*}
    \left|
        S_{\lambda; D^{y}} \paren{X_i, Y_i}
        - \widetilde{S}_{\lambda; D^{y}}^{\mathrm{IF}} \paren{X_{i}, Y_{i}}
    \right|
    &
    =
    \left| 
        s\paren{Y_i, \pred{\lambda;}{y}(X_i)}
        - s\paren{Y_i, \apred{\lambda;}{y}{{\mathrm{IF}}}(X_i)}
    \right|
    \\
    &
    \leq
    \
    \left|
        \pred{\lambda;}{y}(X_i)
        - \apred{\lambda;}{y}{{\mathrm{IF}}}(X_i)
    \right|
    \leq
    \
    \abss{
        \doth{\pred{\lambda;}{y} - \apred{\lambda;}{y}{{\mathrm{IF}}}}{K_{X_i}}}
    \\
    &
    \leq
    \
    \sqrt{ K_{i, i }}
    \normh{\pred{\lambda;}{y} - \apred{\lambda;}{y}{{\mathrm{IF}}}}
    \\
    &\leq
    \sqrt{ K_{i, i }}
    \sqrt{ K_{n+1, n+1} }
    \min \paren{
        \frac{\ \rho^{(2)}_{\lambda}(y)}{
        \lambda^3 (n+1)^2},
        \frac{2 \gamma  \rho^{(1)}_{\lambda}(y)}{\lambda (n+1)}
    }.
\end{align*}
Following  a similar logic, for $i = n+1$
\begin{align*}
    \left|
        S_{\lambda; D^{y}} \paren{X_i, y}
        - \widetilde{S}_{\lambda; D^{y}}^{\mathrm{IF}} \paren{X_{i}, y}
    \right|
    \leq
    \sqrt{K_{i, i }}
    \sqrt{ K_{n+1, n+1}} 
    \min \paren{
        \frac{\gamma \rho^{(2)}_{\lambda}(y)}{
        \lambda^3 (n+1)^2},
        \frac{2 \gamma  \rho^{(1)}_{\lambda}(y)}{\lambda (n+1)}
    }.
\end{align*}
\end{proof}

\subsection{Proofs of the approximate region accuracy part}
\label{proof.very.smooth}

\begin{lemma}
\label{lm.variation.stable}
For every continuous function $f : \mathbb{R} \to \mathbb{R}$ for which
there exists a constant $\eta<1$ such that $f$ is $\eta$-Lipschitz continuous,
the function $g : \mathbb{R} \to \mathbb{R}$, $y \mapsto g(y) = y + f(y)$
is strictly increasing and therefore is a bijection. 
Furthermore, its inverse $g^{-1}$ is a $\frac{1}{1 - \eta}$-Lipschitz continuous,
strictly increasing function.
\end{lemma}
\begin{proof}
\label{proof.variation.stable}
Let $y_1, y_2 \in \mathbb{R}$ such that $y_1 < y_2$.
Since $f$ is $\eta$-Lipschitz continuous
\begin{align*}
    g(y_2) -  g(y_1) 
    &
    = y_2 - y_1 + f(y_2) - f(y_1)
    \\
    &
    \geq y_2 - y_1 - \abss{f(y_2) - f(y_1)}
    \\
    &
    \geq y_2 - y_1 - \eta \abss{y_2 - y_1}
    \\
    &
    \geq (1 - \eta) \abss{y_2 - y_1} > 0.
\end{align*}
It follows that $g(y_1) < g(y_2)$. Since $f$ is continuous, then so is $g$.
Moreover, since $g$ is now strictly increasing, then $g$ is a bijection
and $g^{-1}$ is strictly increasing.

Let $y_1, y_2 \in \mathbb{R}$ such that $y_1 \neq y_2$. Since $g^{-1}$
is strictly increasing, then $g^{-1}(y_1) \neq g^{-1}(y_2)$ and
\begin{align*}
    \abss{
        \frac{
        g^{-1}(y_1) - g^{-1}(y_2)
    }{
        y_1 - y_2
    }
    }
    &
    =
    \abss{
        \frac{
        y_1 - y_2
    }{
        g^{-1}(y_1) - g^{-1}(y_2)
    }
    }^{-1}
    \\
    &
    =
    \abss{
        \frac{
        g\paren{g^{-1}(y_1)} - g\paren{g^{-1}(y_2)}
    }{
        g^{-1}(y_1) - g^{-1}(y_2)
    }
    }^{-1}
    \\
    &
    =
    \abss{
        \frac{
        g^{-1}(y_1) - g^{-1}(y_2) + f\paren{g^{-1}(y_1)} - f\paren{g^{-1}(y_2)}
    }{
        g^{-1}(y_1) - g^{-1}(y_2)
    }
    }^{-1}
    \\
    &
    =
    \abss{
        1 + \frac{
        f\paren{g^{-1}(y_1)} - f\paren{g^{-1}(y_2)}
    }{
        g^{-1}(y_1) - g^{-1}(y_2)
    }
    }^{-1}
    \\
    &
    =
    \paren{
        1 + \frac{
        f\paren{g^{-1}(y_1)} - f\paren{g^{-1}(y_2)}
    }{
        g^{-1}(y_1) - g^{-1}(y_2)
    }
    }^{-1}
    \\
    &
    \leq
    \paren{
        1 - \eta
    }^{-1}.
\end{align*}
where the second to last inequality follow from the fact that $f$ is a contraction and
the last inequality follows from the fact that $f$ is $\eta$-Lipschitz continuous.
\end{proof}

\begin{lemma}[Second bound on thickness]
\label{lm.bound.conf.region.gap.2}
    Assume that there exists an almost-surely finite random variable
    $\Tau$ such that for every $y \in \mathcal{Y}$ and every $i \in \brac{1, \ldots, n+1}$,
    $\widehat{\tau}_i(y) \leq \Tau$ almost-surely.

    Assume that there exists a constant $0 \leq \beta < 1$ such
that for every $i \in \brac{1, \ldots, n+1}$,
the prediction function $y \mapsto \apred{}{y}{}(X_i)$ is $\frac{\beta}{2}$-Lipschitz continuous.

    It follows from Lemma~\ref{lm.sandwiching} that the thickness is bounded from above
\begin{align}
\label{eq.confidence.region.gap.bound.2}
    \thicc{} \leq
        \mathcal{V}\left(
            \ufcpr{} 
            \setminus
            \lfcpr{}
        \right) \leq  \frac{12\Tau}{1 - \beta}.
        \quad\mbox{a.s.}
\end{align}
\end{lemma}
\begin{proof}
\label{proof.bound.conf.region.gap.2}
For ever $i \in \brac{1, \ldots, n}$, and $y \in \mathcal{Y}$,
almost-surely
\begin{align*}
\begin{aligned}
    &&\widetilde{S}_{D^{y}} \paren{X_i, Y_i}
    + \widehat{\tau}_{i}(y)
    &\geq
    \widetilde{S}_{D^{y}} \paren{X_{n+1}, y}
    - \widehat{\tau}_{n+1}(y)
    \\
    \implies
    &&
    \widetilde{S}_{D^{y}} \paren{X_i, Y_i}
    + \Tau
    &\geq
    \widetilde{S}_{D^{y}} \paren{X_{n+1}, y}
    - \Tau
    \\
    \implies
    &&
    \widetilde{S}_{D^{y}} \paren{X_i, Y_i}
    + 2\Tau
    &\geq
    \widetilde{S}_{D^{y}} \paren{X_{n+1}, y}
    \\
    \implies
    &&
    \abss{Y_i - \apred{}{y}{}\paren{X_i}}
    + 2\Tau
    &\geq
    \abss{y - \apred{}{y}{}\paren{X_{n+1}}}
    \\
    \implies
    &&
    - \abss{Y_i - \apred{}{y}{}\paren{X_i}}
    - 2\Tau
    &\leq
    y - \apred{}{y}{}\paren{X_{n+1}}\mbox{ and } 
    \\
    &&
    \abss{Y_i - \apred{}{y}{}\paren{X_i}}
    + 2\Tau
    &\geq
    y - \apred{}{y}{}\paren{X_{n+1}}
    \\
    \implies
    &&
    - 2\Tau
    &\leq
    y - \apred{}{y}{}\paren{X_{n+1}} + \abss{Y_i - \apred{}{y}{}\paren{X_i}}
    \mbox{ and } 
    \\
    &&
    2\Tau
    &\geq
    y - \apred{}{y}{}\paren{X_{n+1}} - \abss{Y_i - \apred{}{y}{}\paren{X_i}}.
\end{aligned}
\end{align*}
Since for every $i \in \brac{1, \ldots, n+1}$,
the function $y \mapsto \apred{}{y}{} \paren{X_{i}}$ is assumed to be $\frac{\beta}{2}$-Lipschitz continuous,
the function $y \mapsto - \apred{}{y}{} \paren{X_{n+1}} + \abss{Y_i - \apred{}{y}{} \paren{X_i}}$
and $y \mapsto -\apred{}{y}{} \paren{X_{n+1}} - \abss{Y_i - \apred{}{y}{} \paren{X_i}}$
are $\beta$-Lipschitz continuous.
It follows from Lemma~\ref{lm.variation.stable} that the functions
\begin{align*}
\begin{aligned}
    h_i^\lo 
    &: \mathbb{R} \to \mathbb{R},
    &&
    y \mapsto h_i^\lo(y) := y - \apred{}{y}{} \paren{X_{n+1}}
        + \abss{Y_i - \apred{}{y}{} \paren{X_i}},
    \\
    h_i^\up 
    &: \mathbb{R} \to \mathbb{R},
    &&
    y \mapsto h_i^\up(y) := y - \apred{}{y}{} \paren{X_{n+1}}
        - \abss{Y_i - \apred{}{y}{} \paren{X_i}},
\end{aligned}
\end{align*}
are $\frac{1}{1 - \beta}$-Lipschitz continuous, strictly increasing one-to-one mappings. 
Moreover, these functions verify the following inequality 
$y \in \mathcal{Y}$, $h_i^\lo(y) \geq h_i^\up(y)$.
This inequality along with the monotony of these function implies
the following inequalities, almost-surely
\begin{align*}
    \paren{h_i^\lo}^{-1}\paren{
        -2 \Tau
    }
    \leq 
    \paren{h_i^\lo}^{-1}(0)
    \leq \paren{h_i^\up}^{-1}(0)
    \leq \paren{h_i^\up}^{-1}
    \paren{
        2 \Tau
    }.
\end{align*}
Since $\paren{h_i^\lo}^{-1}$ and $\paren{h_i^\up}^{-1}$ 
are $\frac{1}{1 - \beta}$-Lipschitz continuous and following the previous inequalities
\begin{align*}
\begin{aligned}
    &&\widetilde{S}_{D^{y}} \paren{X_i, Y_i}
    + \widehat{\tau}_{i}(y)
    &\geq
    \widetilde{S}_{D^{y}} \paren{X_{n+1}, y}
    - \widehat{\tau}_{n+1}(y)
    \\
    \implies
    &&
    \paren{h_i^\lo}^{-1}\paren{
        -2 \Tau
    }
    &\leq
    y
    \leq 
    \paren{h_i^\up}^{-1}\paren{
        2 \Tau
    }
    \\
    \implies
    &&
    \paren{h_i^\lo}^{-1}(0)
    - \frac{1}{1 - \beta} \paren{2 \Tau}
    &\leq
    y
    \leq 
    \paren{h_i^\up}^{-1}(0) + \frac{1}{1 - \beta} \paren{2 \Tau}
\end{aligned}
\end{align*}
Let us note $\tilde{\Tau} := \frac{2\Tau}{1-\beta}$ and 
for every $i \in \brac{1, \ldots, n}$,
$A_{i} := \paren{h_i^\lo}^{-1}(0)$ and $B_{i} := \paren{h_i^\up}^{-1}(0)$.
It follows form the above implications that the upper-approximate p-value function
is bounded from above by the following function $f_{\tilde{\Tau}}$
defined as, for every $y \in \mathcal{Y}$, almost-surely
\begin{align*}
    f_{\tilde{\Tau}}(y)
    :=
    \frac{
        1
        + \sum_{i=1}^{n}
        \mathbbm{1}
        \brac{
            A_i - \tilde{\Tau} \leq y \leq  B_i + \tilde{\Tau}
        }
    }{n+1}
    \geq
    \ufcpv{}{y}.
\end{align*}
Since $y \mapsto \apred{}{y}{} \paren{X_{n+1}}$ is assumed to be $\frac{\beta}{2}$-Lipschitz continuous,
it follows from Lemma~\ref{lm.variation.stable} that the function 
\begin{align*}
\begin{aligned}
    h^\midd &: \mathbb{R} \to \mathbb{R},
    &&y \mapsto h^\midd(y) := y - \apred{}{y}{} \paren{X_{n+1}}.
\end{aligned}
\end{align*}
is a strictly increasing bijection.
Since for every $y \in \mathbb{R}$,
$h_i^\lo(y) \geq h^\midd (y) \geq h_i^\up (y)$,
the following inequalities hold for every  $i \in \brac{1, \ldots, n}$, almost-surely
\begin{align*}
    A_i - \tilde{\Tau}
    \leq \paren{h^\midd}^{-1} (0) - \tilde{\Tau}
    \leq \paren{h^\midd}^{-1} (0) + \tilde{\Tau}
    \leq B_i + \tilde{\Tau}.
\end{align*}
Let us note $A_{(1)} \leq \ldots \leq A_{(n)}$ the values
among  $A_{1}, \ldots, A_{n}$ sorted in increasing order and 
$B_{(1)} \leq \ldots \leq B_{(n)}$ for $B_{1}, \ldots, B_{n}$.
The previous inclusions imply the following inequalities
\begin{align}
\label{eq.intersection.non.empty}
    A_{(1)} - \tilde{\Tau} 
    \leq \ldots 
    \leq A_{(n)} - \tilde{\Tau}
    &
    \leq \paren{h^\midd}^{-1} (0)
    - \tilde{\Tau}
    \\
    &
    \leq \paren{h^\midd}^{-1} (0)
    + \tilde{\Tau}
    \leq B_{(1)} + \tilde{\Tau} 
    \leq \ldots 
    \leq B_{(n)} + \tilde{\Tau}.\notag
\end{align}
The function $f_{\tilde{\Tau}}$ can be rewritten as, for every $y \in \mathcal{Y}$
\begin{align*}
    f_{\tilde{\Tau}}(y)
    &= \frac{
        1
        + \sum_{i=1}^{n}
        \mathbbm{1} \brac{A_{i} - \tilde{\Tau} \leq y \leq B_{(1)} + \tilde{\Tau}}
        + \sum_{i=1}^{n}
        \mathbbm{1} \brac{B_{(1)} + \tilde{\Tau} < y \leq B_{n} + \tilde{\Tau}}
    }{n+1}
    \\
    &= \frac{
        1
        + \sum_{i=1}^{n}
        \mathbbm{1} \brac{A_{(i)} - \tilde{\Tau} \leq y \leq B_{(1)} + \tilde{\Tau}}
        + \sum_{i=1}^{n}
        \mathbbm{1} \brac{B_{(1)} + \tilde{\Tau} < y \leq B_{(i)} + \tilde{\Tau}}
    }{n+1}.
\end{align*}
and almost-surely
\begin{itemize}
    \item over $\paren{- \infty, A_{(1)} - \tilde{\Tau}}$ and $\paren{B_{(n)} + \tilde{\Tau}, +\infty}$,
    this function is constant and equal to $0$,
    \item over $\paren{- \infty, A_{(n)} - \tilde{\Tau}}$, this function is increasing,
    \item over $\croch{A_{(n)} - \tilde{\Tau}, B_{(1)} + \tilde{\Tau}}$, this function is constant and equal to $1$, and
    \item over $\paren{B_{(1)} + \tilde{\Tau}, + \infty}$, this function is decreasing.
\end{itemize}
Therefore, the upper-approximate prediction-region is contained in the following interval, almost-surely
\begin{align*}
    \ufcpr{}
    \subseteq
    \brac{
        y \in \mathcal{Y} :
        f_{\tilde{\Tau}}(y) > \alpha
    }
    = \croch{
        A_{(j_{n, \alpha})} - \tilde{\Tau},
        B_{(i_{n, \alpha})} + \tilde{\Tau}
    },
\end{align*}
where $j_{n, \alpha} = \floor{(n+1)\alpha - 1} + 1$ and $i_{n, \alpha} = \ceil{(n+1)(1-\alpha)}$.

Following a similar reasoning, the lower-approximate p-value function is bounded 
below by the function following $f_{- \tilde{\Tau}}$ defined as, for every $y \in \mathcal{Y}$, almost-surely
\begin{align*}
    f_{-\tilde{\Tau}}(y)
    := 
    \frac{
        1
        + \sum_{i = 1}^{n}
        \mathbbm{1}
        \brac{
            A_i + \tilde{\Tau} \leq y \leq B_i - \tilde{\Tau}
        }
    }{n+1}
    \leq \lfcpv{}{y}.
\end{align*}

Let us first consider the case where $A_{(n)} - \tilde{\Tau} \leq B_{(1)} + \tilde{\Tau}$.
The following inequalities hold, almost-surely
\begin{align*}
    A_{(1)} + \tilde{\Tau}
    \leq \ldots
    \leq A_{(n)} + \tilde{\Tau}
    \leq B_{(1)} - \tilde{\Tau}
    \leq \ldots
    \leq B_{(n)} - \tilde{\Tau}.
\end{align*}
The function $f_{-\tilde{\Tau}}$ can be rewritten as, for every $y \in \mathcal{Y}$
\begin{align*}
    f_{-\tilde{\Tau}}(y)
    = \frac{
        1
        + \sum_{i=1}^{n}
        \mathbbm{1} \brac{A_{(i)} + \tilde{\Tau} \leq y \leq B_{(1)} - \tilde{\Tau}}
        + \sum_{i=1}^{n}
        \mathbbm{1} \brac{B_{(1)} - \tilde{\Tau} < y \leq B_{(i)} - \tilde{\Tau}}
    }{n+1}.
\end{align*}
and
\begin{itemize}
    \item over $\paren{- \infty, A_{(1)} + \tilde{\Tau}}$ and $\paren{B_{(n)} - \tilde{\Tau}, +\infty}$,
    this function is constant and equal to $0$,
    \item over $\paren{- \infty, A_{(n)} + \tilde{\Tau}}$, this function is increasing,
    \item over $\croch{A_{(n)} + \tilde{\Tau}, B_{(1)} - \tilde{\Tau}}$, this function is constant and equal to $1$, and
    \item over $\paren{B_{(1)} - \tilde{\Tau}, + \infty}$, this function is decreasing.
\end{itemize}
Therefore, the lower-approximate region contains the following region
\begin{align*}
    \lfcpr{} \supseteq \brac{y \in \mathcal{Y} : f_{- \tilde{\Tau}}(y) > \alpha}
    = \croch{A_{\paren{j_{n, \alpha}}} + \tilde{\Tau}, B_{\paren{i_{n, \alpha}}} - \tilde{\Tau}}.
\end{align*}
Most importantly, it follows that the confidence gap is almost-surely bounded from above as
\begin{align*}
    \thicc{}
    &
    \leq
    \leb{
        \ufcpr{} \setminus \lfcpr{}
    }
    \\
    &
    \leq \leb{
        \croch{A_{\paren{j_{n, \alpha}}} - \tilde{\Tau}, B_{\paren{i_{n, \alpha}}} + \tilde{\Tau}}
        \setminus
        \croch{A_{\paren{j_{n, \alpha}}} + \tilde{\Tau}, B_{\paren{i_{n, \alpha}}} - \tilde{\Tau}}
    }
    \\
    &
    \leq  \frac{8\Tau}{1 - \beta}.
\end{align*}

Let us now consider the case where $A_{(n)} - \tilde{\Tau} > B_{(1)} + \tilde{\Tau}$.
Let us note $\mathcal{I}$ and $\mathcal{J}\subseteq \brac{1, \ldots, n}$ the lists of indices defined as
\begin{align*}
    \mathcal{I}
    &:= \brac{i \in \brac{1, \ldots, n} : A_{i} + \tilde{\Tau} \leq B_{(1)} + \tilde{\Tau}},
    \\
    \mathcal{J}
    &:=
    \brac{i \in \brac{1, \ldots, n}: A_{(n)} + \tilde{\Tau} \leq B_{i} - \tilde{\Tau}}.
\end{align*}
For every $i \in \mathcal{I}\cap\mathcal{J}^c$ and every $y \in \mathcal{Y}$, almost-surely
\begin{align*}
    \mathbbm{1}
    \brac{A_{i} + \tilde{\Tau} \leq y \leq B_{i} - \tilde{\Tau}}
    \geq \mathbbm{1}
    \brac{A_{i} + \tilde{\Tau}\leq y \leq B_{(1)} - \tilde{\Tau}},
\end{align*}
for every $i \in \mathcal{I}^{c}\cap\mathcal{J}$ and every $y \in \mathcal{Y}$,, almost-surely
\begin{align*}
    \mathbbm{1}
    \brac{A_{i} + \tilde{\Tau} \leq y \leq B_{i} - \tilde{\Tau}}
    \geq \mathbbm{1}
    \brac{A_{(n)} + \tilde{\Tau}\leq y \leq B_{i} - \tilde{\Tau}},
\end{align*}
and for every $i \in \mathcal{I}\cap\mathcal{J}$ and every $y \in \mathcal{Y}$, almost-surely
\begin{align*}
    \mathbbm{1}
    \brac{A_{i} + \tilde{\Tau} \leq y \leq B_{i} - \tilde{\Tau}}
    \geq \mathbbm{1}
    \brac{A_{i} + \tilde{\Tau}\leq y \leq B_{(1)} - \tilde{\Tau}}
    + \mathbbm{1} \brac{A_{(n)} + \tilde{\Tau}\leq y \leq B_{i} - \tilde{\Tau}}.
\end{align*}
Let us note $k_{\tilde\Tau}$ and $l_{\tilde\Tau} \in \brac{1, \ldots, n}$  the indices such that
\begin{align*}
    k_{\tilde{\Tau}}
    &
    :=
    \argmax
    \brac{
        A_{(i)} + \tilde{\Tau} : 
        i \in \brac{1, \ldots, n}, \quad
        A_{(i)} + \tilde{\Tau} \leq B_{(1)} - \tilde{\Tau}
    },
    \\
    l_{\tilde{\Tau}}
    &
    :=
    \argmin
    \brac{
        B_{(i)} - \tilde{\Tau} : 
        i \in \brac{1, \ldots, n}, \quad
        A_{(n)} + \tilde{\Tau} \leq B_{(i)} - \tilde{\Tau}
    }.
\end{align*}
It follows that, almost-surely
\begin{align*}
    A_{(1)} + \tilde{\Tau}
    \leq \ldots
    \leq A_{\paren{k_{\tilde{\Tau}}}} + \tilde{\Tau}
    \leq B_{(1)} - \tilde{\Tau}
    < A_{(n)} + \tilde{\Tau}
    \leq B_{\paren{l_{\tilde{\Tau}}}} - \tilde{\Tau} 
    \leq \ldots
    \leq B_{(n)} - \tilde{\Tau},
\end{align*}
and the function $f_{-\tilde{\Tau}}$ is in turn bounded from below almost-surely,
for every $y \in \mathcal{Y}$ as
\begin{align*}
    f_{-\tilde{\Tau}}(y)
    &= \frac{1 + \sum_{i=1}^{n}
    \mathbbm{1} \brac{A_i + \tilde{\Tau} \leq y \leq B_i - \tilde{\Tau}}}{n+1}
    \\
    &
    \geq
    \frac{1 + \sum_{i \in \mathcal{I} \cup \mathcal{J}}
    \mathbbm{1} \brac{A_i + \tilde{\Tau} \leq y \leq B_i - \tilde{\Tau}}}{n+1}
    \\
    &
    \geq
    \frac{1}{n+1}
    \left(
        1
        + \sum_{i \in \mathcal{I} \cap \mathcal{J}^c}
        \mathbbm{1} \brac{A_{i} + \tilde{\Tau}\leq y \leq B_{(1)} - \tilde{\Tau}}
        + \sum_{i \in \mathcal{I}^c \cap \mathcal{J}}
        \mathbbm{1} \brac{A_{(n)} + \tilde{\Tau} \leq y \leq B_i - \tilde{\Tau}}
    \right.
    \\
    &
    \qquad\qquad\qquad+
    \left.
        \sum_{i \in \mathcal{I} \cap \mathcal{J}}
        \paren{
            \mathbbm{1} \brac{A_i + \tilde{\Tau} \leq y \leq B_{(1)} - \tilde{\Tau}}
            + \mathbbm{1} \brac{A_{(1)} + \tilde{\Tau} \leq y \leq B_i - \tilde{\Tau}}
        }
    \right)
    \\
    &
    \geq
    \frac{
        1 
        + \sum_{i \in \mathcal{I}}
        \mathbbm{1} \brac{A_i + \tilde{\Tau} \leq y \leq B_{(1)} - \tilde{\Tau}}
        + \sum_{i \in \mathcal{J}}
        \mathbbm{1} \brac{A_{(n)} + \tilde{\Tau} \leq y \leq B_i - \tilde{\Tau}}
    }{n+1}
    \\
    &
    \geq
    \frac{
        1 + \sum_{i=1}^{k_{\tilde{\Tau}}}
        \mathbbm{1} \brac{
            A_{(i)} + \tilde{\Tau}
            \leq y \leq B_{(1)} - \tilde{\Tau}
        }
        +
        \sum_{i=l_{\tilde{\Tau}}}^{n}
        \mathbbm{1} \brac{
            A_{(n)} + \tilde{\Tau}
            \leq y \leq B_{(i)} - \tilde{\Tau}
        }
    }{
        n+1
    }
    =: \tilde f_{-\tilde{\Tau}}(y).
\end{align*}
and
\begin{itemize}
    \item over $\paren{- \infty, A_{(1)} + \tilde{\Tau}}$, $\paren{B_{(1)} - \tilde{\Tau}, A_{(n)} + \tilde{\Tau}}$
    and $\paren{B_{(n)} - \tilde{\Tau}, +\infty}$,
    this function is constant and equal to $0$,
    \item over $\paren{- \infty, B_{(1)} - \tilde{\Tau}}$, this function is increasing,
    \item over $\paren{A_{(n)} + \tilde{\Tau}, + \infty}$, this function is decreasing.
\end{itemize}
Therefore the lower-approximate region contains the following region, almost-surely
\begin{align*}
    \lfcpr{} \supseteq \brac{y \in \mathcal{Y} : \tilde f_{- \tilde{\Tau}}(y) > \alpha},
\end{align*}
which is equal to
\begin{enumerate}
    \item $\croch{
        A_{\paren{j_{n, \alpha}}} + \tilde{\Tau},
        B_{(1)} - \tilde{\Tau}
    }
    \cup
    \croch{
        A_{(n)} + \tilde{\Tau},
        B_{\paren{i_{n, \alpha}}} - \tilde{\Tau}
    }$ if $j_{n, \alpha} \leq k_{\tilde \Tau}$ and $l_{\tilde \Tau} \leq i_{n, \alpha} $,
    \item $
    \croch{
        A_{(n)} + \tilde{\Tau},
        B_{\paren{i_{n, \alpha}}} - \tilde{\Tau}
    }$ if $k_{\tilde \Tau} < j_{n, \alpha}$ and $l_{\tilde \Tau} \leq i_{n, \alpha} $,
    \item $\croch{
        A_{\paren{j_{n, \alpha}}} + \tilde{\Tau},
        B_{(1)} - \tilde{\Tau}
    }$ if $j_{n, \alpha} \leq l_{\tilde \Tau}$ and $i_{n, \alpha} < l_{\tilde \Tau}  $,
    \item $\emptyset$ if $k_{\tilde \Tau} < j_{n, \alpha}$ and $i_{n, \alpha} < l_{\tilde \Tau} $.
\end{enumerate}
It follows that the confidence gap is bounded from above, almost-surely
\begin{align*}
    \thicc{}
    &
    \leq
    \leb{
        \ufcpr{} \setminus \lfcpr{}
    }
    \\
    &
    \leq
    \leb{
        \croch{A_{\paren{j_{n, \alpha}}} - \tilde{\Tau}, B_{\paren{i_{n, \alpha}}} + \tilde{\Tau}}
        \setminus
        \brac{y \in \mathcal{Y} : \tilde f_{- \tilde{\Tau}}(y) > \alpha}
    }
\end{align*}
which is almost-surely equal to
\begin{enumerate}
    \item $
    A_{(n)} - B_{(1)} + 6 \tilde{\Tau}$
    if $j_{n, \alpha} \leq l_{\tilde \Tau}$ and $l_{\tilde \Tau} \leq i_{n, \alpha} $,
    \item $
    A_{(n)} - A_{\paren{j_{n, \alpha}}} + 4 \tilde{\Tau}$
    if $k_{\tilde \Tau} < j_{n, \alpha}$ and $l_{\tilde \Tau} \leq i_{n, \alpha} $,
    \item $
    B_{\paren{i_{n, \alpha}}} - B_{(1)} + 4 \tilde{\Tau}$
    if $j_{n, \alpha} \leq l_{\tilde \Tau}$ and $i_{n, \alpha} < l_{\tilde \Tau}  $,
    \item $
        B_{\paren{i_{n, \alpha}}} - A_{\paren{j_{n, \alpha}}} + 2 \tilde{\Tau}
    $ if $k_{\tilde \Tau} < j_{n, \alpha}$ and $i_{n, \alpha} < l_{\tilde \Tau}  $.
\end{enumerate}
From Equation \ref{eq.intersection.non.empty},
$A_{(n)} - \tilde{\Tau}
\leq \paren{h^\midd}^{-1} (0) - \tilde{\Tau}
\leq \paren{h^\midd}^{-1} (0) + \tilde{\Tau}
\leq B_{(1)} + \tilde{\Tau}$.
It follows that $B_{(1)} + \tilde{\Tau} - A_{(n)} + \tilde{\Tau} \geq 2 \tilde{\Tau}$,
and  $A_{(n)} - B_{(1)} \leq 0$. Therefore, in the case 1
\begin{align*}
    \thicc{} \leq A_{(n)} - B_{(1)} + 6 \tilde{\Tau} \leq 6\tilde{\Tau}.
\end{align*}
If $k_{\tilde \Tau} < j_{n, \alpha}$, then
$A_{\paren{j_{n, \alpha}}} + \tilde{\Tau} > B_{(1)} - \tilde{\Tau}$,
and $- A_{\paren{j_{n, \alpha}}} < -B_{(1)} + 2\tilde{\Tau}$. Therefore, in case 2
\begin{align*}
    \thicc{} \leq A_{(n)} - A_{\paren{j_{n, \alpha}}} + 4 \tilde{\Tau}
    \leq A_{(n)} -B_{(1)} + 6 \tilde{\Tau} \leq 6 \tilde{\Tau}.
\end{align*}
If $j_{n, \alpha} < l_{\tilde \Tau}$, then
$A_{(n)} + \tilde{\Tau} > B_{\paren{i_{n, \alpha}}} - \tilde{\Tau}$,
and $B_{\paren{i_{n, \alpha}}} < A_{(n)} + 2 \tilde{\Tau}$. Therefore, in case 3
\begin{align*}
    \thicc{} \leq B_{\paren{i_{n, \alpha}}} - B_{(1)} + 4 \tilde{\Tau}
    \leq A_{(n)} -B_{(1)} + 6 \tilde{\Tau} \leq 6 \tilde{\Tau}.
\end{align*}
If $k_{\tilde \Tau} < j_{n, \alpha}$ and $i_{n, \alpha} < l_{\tilde \Tau}  $,
$- A_{\paren{j_{n, \alpha}}} < -B_{(1)} + 2\tilde{\Tau}$ and
$B_{\paren{i_{n, \alpha}}} < A_{(n)} + 2 \tilde{\Tau}$. Therefore, in case 4
\begin{align*}
    \thicc{} \leq B_{\paren{i_{n, \alpha}}} - A_{\paren{j_{n, \alpha}}} + 2 \tilde{\Tau}
    \leq A_{(n)} -B_{(1)} + 6\tilde{\Tau} \leq 6 \tilde{\Tau}.
\end{align*}
In sum, the thickness gap is bounded from above
\begin{align*}
    \thicc{} \leq \frac{12 \Tau}{1 - \beta}.
    \quad \mbox{a.s.}
\end{align*}
\end{proof}

\begin{proof}(Proof of Theorem~\ref{thm.very.smooth})
\label{proof.thm.very.smooth}
\noindent{\emph{Coverage.}}
Following a similar reasoning as in Theorem~\ref{thm.non.smooth},
the new approximate region is a confidence-region.

\bigskip{}
\noindent{\emph{Crude upper-bound on the thickness.}}
Let us note
\begin{align*}
    \hat{\kappa}^2 := \max_{1\leq i \leq n+1} \kappa_{\mathcal{H}} \paren{X_i, X_i}.
\end{align*}
For every $y \in \mathcal{Y}$
\begin{align*}
    \tilde{\rho}^{(1)}_{\lambda}(y)
    = \paren{
        1 + K_{n+1, n+1} \frac{\beta_{\ell; 2}}{\lambda(n+1)}
    } \rho_{\lambda}^{(1)}\paren{y}
    \leq \paren{
        1 + \frac{\hat{\kappa}^2 \beta_{\ell; 2}}{\lambda(n+1)}
    } \rho,
\end{align*}
where the last inequality follows from \eqref{asm.loss.lipschitz}.
Under \eqref{asm.bounded.kernel},
it follows that
\begin{align*}
    \rho^{(2)}_{\lambda}(y) 
    &
    = \frac{1}{2} 
    \sqrt{K_{n+1, n+1}}
    \left(
        \frac{1}{n+1} 
        \sum_{i=1}^{n+1} 
        \paren{K_{i, i}}^{3/2}
    \right)
    \paren{
        \tilde{\rho}^{(1)}_{\lambda}(y)
    }^2
    \xi_{\ell}
    + 2 \lambda (K_{n+1, n+1}) \tilde{\rho}^{(1)}_{\lambda}(y) \beta_{\ell; 2}
    \\
    &
    \leq
    \frac{1}{2} 
    \hat{\kappa}^4 \left(
        1 + \frac{\hat{\kappa}^2 \beta_{\ell; 2}}{\lambda (n+1)}
    \right)^2 \rho^2 \xi_{\ell}
    + 2 \lambda \hat{\kappa}^2 \beta_{\ell; 2} 
    \left(
        1 + \frac{\hat{\kappa}^2 \beta_{\ell; 2}}{\lambda (n+1)}
    \right) \rho.
    \\
    &
    \leq
    \frac{1}{2} 
    \hat{\kappa}^4 \left(
        1 + 2\frac{\hat{\kappa}^2 \beta_{\ell; 2}}{\lambda (n+1)}
        +\frac{\hat{\kappa}^4 \beta_{\ell; 2}^2}{\lambda^2 (n+1)^2}
    \right) \rho^2 \xi_{\ell}
    + 2 \lambda \hat{\kappa}^2 \rho \beta_{\ell; 2}
    + 2\frac{\hat{\kappa}^4 \beta_{\ell; 2}^2}{(n+1)}
    \\
    &
    \leq
    \frac{1}{2} 
    \kappa^4 \left(
        1 + 2\frac{\kappa^2 \beta_{\ell; 2}}{\lambda (n+1)}
        +\frac{\kappa^4 \beta_{\ell; 2}^2}{\lambda^2 (n+1)^2}
    \right) \rho^2 \xi_{\ell}
    + 2 \lambda \kappa^2 \rho \beta_{\ell; 2}
    + 2\frac{\kappa^4 \beta_{\ell; 2}^2}{(n+1)}
    \\
    &
    \leq
    \frac{\kappa^4 \rho^2 \xi_{\ell}}{2} 
    + \frac{\kappa^6 \rho^2 \beta_{\ell; 2} \xi_{\ell}}{\lambda (n+1)}
    + \frac{\kappa^8 \rho^2 \beta_{\ell; 2}^2 \xi_{\ell}}{2 \lambda^2 (n+1)^2}
    + 2 \lambda \kappa^2 \rho \beta_{\ell; 2}
    + 2\frac{\kappa^4 \beta_{\ell; 2}^2}{(n+1)}
    \\
    &
    \leq
    \frac{\kappa^4 \rho^2 \xi_{\ell}}{2} 
    + O \paren{\frac{1}{\lambda(n+1)}}
    + o \paren{\frac{1}{\lambda(n+1)}}
    + 2 \lambda \kappa^2 \rho \beta_{\ell; 2}
    + O \paren{\frac{1}{\lambda(n+1}}
    \\
    &
    \leq \frac{\kappa^4 \rho^2 \xi_{\ell}}{2}
    + 2 \lambda \kappa^2 \rho \beta_{\ell; 2}
    + O \paren{\frac{1}{\lambda (n+1)}}.
\end{align*}
Going back to the upper-bound on the score approximation,
for every $i \in \brac{1, \ldots, n+1}$
and every $y \in \mathcal{Y}$
\begin{align}
    \widehat{\tau}_{\lambda; i}^{(2)}
    &= \sqrt{K_{i, i}}
    \sqrt{K_{n+1, n+1}}
    \min \paren{
        \frac{\rho^{(2)}_{\lambda}(y)}{
        \lambda^3 (n+1)^2},
        \frac{2 \rho^{(1)}_{\lambda}(y)}{\lambda (n+1)}
    } \notag
    \\
    &
    \leq
    \min \paren{
        \frac{
            \kappa^6 \rho^2 \xi_{\ell}
        }{
            2 \lambda^3 (n+1)^2
        }
        + 
        \frac{
            2 \kappa^4 \rho \beta_{\ell; 2}
        }{2 \lambda^2 (n+1)^2}
        + O \paren{\frac{1}{\lambda^4 (n+1)^3}},
        \frac{2 \kappa^2 \rho}{\lambda (n+1)}
    } \notag
    \\
    &
    \leq
    \min \paren{
        \frac{
            \kappa^6 \rho^2 \xi_{\ell}
        }{
            2 \lambda^3 (n+1)^2
        }
        + o \paren{\frac{1}{\lambda^3 (n+1)^2}}
        + o \paren{\frac{1}{\lambda^3 (n+1)^2}},
        \frac{2 \kappa^2 \rho}{\lambda (n+1)}
    } \notag
    \\
    &
    \leq
    \min \paren{
        O \paren{
            \frac{
                \kappa^6 \rho^2 \xi_{\ell}
            }{
                \lambda^3 (n+1)^2
            }
        },
        \frac{2 \kappa^2 \rho}{\lambda (n+1)}
    }
    =: \widehat{\tau}^{(2)}.\label{eq.def.T2}
\end{align}
On the other hand, for every $i \in \brac{1, \ldots, n+1}$
\begin{align*}
    \abss{
        \tilde{f}_{\lambda; D^{y}}^{\mathrm{IF}} \paren{X_i}
        - \hat{f}_{\lambda; D^{z}} \paren{X_i}
    }
    &
    =
    \abss{
        \hat{f}_{\lambda; D^{z}} \paren{X_i}
        - \ifunc{z}\paren{X_i} + \ifunc{y}\paren{X_i}
        - \hat{f}_{\lambda; D^{z}} \paren{X_i}
    }
    \\
    &
    =
    \frac{1}{n+1}
    \abss{
        - \partial_2
        \ell
        \paren{
            z, \hat{f}_{\lambda; D^{z}}\paren{X_{n+1}}
        }
        + \partial_2
        \ell
        \paren{
            y, \hat{f}_{\lambda; D^{z}}\paren{X_{n+1}}
        }
    }
    \\
    &
    \quad \quad \quad \times \abss{
        \paren{
            \croch{
                \partial_2^2 \wer{\lambda}{\uu, \hat{f}_{\lambda; D^{z}}}
            }^+
            K_{X_{n+1}}
        }
        \paren{X_i}
    }
    \\
    &
    \leq
    \frac{2 \rho_{\lambda}^{(1)}\paren{y}}{n+1}
    \normh{
        \croch{
            \partial_2^2 \wer{\lambda}{\uu, \hat{f}_{\lambda; D^{z}}}
        }^+
        K_{X_{n+1}}
    }
    \normh{
        K_{X_i}
    }
    \\
    &
    \leq
    \frac{2 \rho}{n+1}
    \normop{
        \croch{
            \partial_2^2 \wer{\lambda}{\uu, \hat{f}_{\lambda; D^{z}}}
        }^+
    }
    \normh{
        K_{X_{n+1}}
    }
    \sqrt{K_{i, i}}
    \\
    &
    \leq
    \frac{2 \rho}{n+1}
    \frac{1}{2\lambda}
    \sqrt{K_{n+1, n+1}}
    \sqrt{K_{i, i }}
    \leq
    \frac{\hat{\kappa}^2 \rho}{\lambda \paren{n+1}}
    \\
    &
    \leq
    \frac{\kappa^2 \rho}{\lambda \paren{n+1}},
\end{align*}
where the third inequality follows from Lemma~\ref{lm.wer.strongly.convex.1}
which ensures that $\hat{\mathbf{R}}_{\lambda}\paren{\mathbf{u},\bullet}$
is $2\lambda$-strongly convex.
Let the prediction-region $\ufcprr{}{\lambda; \alpha}$
be defined as
\begin{align*}
    \ufcprr{}{\lambda; \alpha}
    :=
    \brac{
        y \in \mathcal{Y} :
        \widetilde{\pi}_{\lambda; D}^{\up}
        \paren{X_{n+1}, y} > \alpha
    },
\end{align*}
where for every $y \in \mathcal{Y}$, the conformal p-value function is defined as
\begin{align*}
    \widetilde{\pi}_{\lambda; D}^{\up}
    \paren{X_{n+1}, y} 
    := \frac{1
    + \sum_{i=1}^{n}
    \mathbbm{1} \brac{
        \widetilde{S}_{\lambda; D^y}\paren{X_i, Y_i}
        - \tilde{\tau}_{\lambda; i}^{(2)}\paren{y}
        \geq
        \widetilde{S}_{\lambda; D^y}\paren{X_{n+1}, y}
        + \tilde{\tau}_{\lambda; n+1}^{(2)}\paren{y}
    }
    }{n+1},
\end{align*}
with the non-conformity scores approximation defined in Eq.~\eqref{eq.approximate.score.0},
where for every $i \in \brac{1, \ldots, n+1}$
and every $y \in \mathcal{Y}$
\begin{align*}
    \tilde{\tau}_{\lambda; i}^{(2)}\paren{y}
    :=
    \widehat{\tau}_{\lambda; i}^{(2)}\paren{y}
    + \frac{\kappa^2 \rho}{\lambda \paren{n+1}}.
\end{align*}
It follows that the approximate full-conformal prediction-region
$\ufcprr{,(2)}{\lambda; \alpha}$ is contained in $\ufcprr{}{\lambda; \alpha}$.
Therefore, the thickness $\thick{(2)}{\lambda; \alpha}$ of the former
is bounded from above
by the thickness $\thick{}{\lambda; \alpha}$ of the latter that is, 
\begin{align*}
    \thick{(2)}{\lambda; \alpha}
    \leq \thick{}{\lambda; \alpha}.
\end{align*}
Moreover, for every $i \in \brac{1, \ldots, n+1}$
and every $y \in \mathcal{Y}$
\begin{align*}
    \tilde{\tau}_{\lambda; i}^{(2)}\paren{y}
    =
    \widehat{\tau}_{\lambda; i}^{(2)}\paren{y}
    + \frac{\kappa^2 \rho}{\lambda \paren{n+1}}
    \leq
    \widehat{\tau}^{(2)} + \frac{\kappa^2 \rho}{\lambda \paren{n+1}}.
\end{align*}
    Since the non-conformity score approximations
    used to define the prediction $\ufcprr{}{\lambda; \alpha}$
    do not depend on the output value $y \in \mathcal{Y}$,
it follows from Lemma~\ref{lm.bound.conf.region.gap} that
\begin{align*}
    \thick{}{\lambda; \alpha}
    \leq 8 \paren{
        \frac{\kappa^2 \rho}{\lambda \paren{n+1}}
        +
        \widehat{\tau}^{(2)}
    }
    \leq
    O \paren{
        \frac{\kappa^2 \rho}{\lambda (n+1)}
    }\quad \mbox{a.s.}
\end{align*}
Therefore, the thickness of the approximate full-conformal prediction-region is bounded from above as
\begin{align*}
    \thick{(2)}{\lambda; \alpha} \leq O \paren{
        \frac{\kappa^2 \rho}{\lambda (n+1)}
    }\quad \mbox{a.s.}
\end{align*}
\bigskip
\noindent{\emph{Refined upper-bound on the thickness.}}
Under \eqref{asm.loss.smooth.2},
for every $i \in \brac{1, \ldots, n+1}$ and every $y \in \mathcal{Y}$,
\begin{align*}
    \abss{
        \partial_y \apred{\lambda;}{y}{\mathrm{IF}} \paren{X_i}
    }
    &
    = \abss{
        \partial_y 
        \paren{
            \pred{\lambda;}{z}\paren{X_i}
            - \paren{\ifunc{y}}\paren{X_i}
            + \paren{\ifunc{z}}\paren{X_i}
        }
    }
    \\
    &
    =
    \abss{
        - \partial_y \paren{\ifunc{y}}\paren{X_i}    
    }
    \\
    &
    =
    \abss{
        \frac{1}{n+1}
        \partial_y
        \partial_2
        \ell\paren{y, \pred{\lambda;}{z}\paren{X_{n+1}}}
        \paren{
            \croch{
                \DDtwo \wer{\lambda}{\uu}{\pred{\lambda;}{z}}
            }^+ K_{X_{n+1}}
        } \paren{X_i}
    }
    \\
    &
    =
    \frac{1}{n+1}
    \abss{
        \partial_1
        \partial_2
        \ell\paren{y, \pred{\lambda;}{z}\paren{X_{n+1}}}
    }
    \abss{
        \paren{
            \croch{
                \DDtwo \wer{\lambda}{\uu}{\pred{\lambda;}{z}}
            }^+ K_{X_{n+1}}
        } \paren{X_i}
    }
    \\
    &
    \leq
    \frac{1}{n+1}
    \beta_{\ell; 1}
    \abss{
        \doth{
            \croch{
                \DDtwo \wer{\lambda}{\uu}{\pred{\lambda;}{z}}
            }^+ K_{X_{n+1}}}{K_{X_i}}
    }
    \\
    &
    \leq
    \frac{1}{n+1}
    \beta_{\ell; 1}
    \normh{\croch{
                \DDtwo \wer{\lambda}{\uu}{\pred{\lambda;}{z}}
            }^+ K_{X_{n+1}}
    }
    \normh{
        K_{X_i}
    }
    \\
    &
    \leq
    \frac{1}{n+1}
    \beta_{\ell; 1}
    \normop{
        \croch{
                \DDtwo \wer{\lambda}{\uu}{\pred{\lambda;}{z}}
            }^+
    }
    \normh{K_{X_{n+1}}
    }
    \sqrt{K_{i, i}} 
    \\
    &
    \leq
    \frac{
        \hat{\kappa}^2 \beta_{\ell; 1}
    }{2\lambda (n+1)}    
    \leq
    \frac{\kappa^2 \beta_{\ell; 1}}{2\lambda (n+1)}
    \quad \mbox{a.s.},
\end{align*}
where the last inequality follows from \eqref{asm.bounded.kernel} and
Lemma~\ref{lm.wer.strongly.convex.1}, which ensures that $\hat{\mathbf{R}}_{\lambda}\paren{\mathbf{u},\bullet}$
is $2\lambda$-strongly convex.
It follows that when
\begin{align*}
    \lambda (n+1) > \frac{\kappa^2 \beta_{\ell; 1}}{2},
\end{align*}
for every $i \in \brac{1, \ldots, n+1}$ and every $y \in \mathcal{Y}$,
\begin{align*}
    \abss{
        \partial_y \apred{\lambda;}{y}{\mathrm{IF}} \paren{X_i}
    } < 1.
\end{align*}
By combining Lemma~\ref{lm.bound.conf.region.gap.2} and the definition of $\widehat{\tau}^{(2)}$ from Eq.~\eqref{eq.def.T2}, choosing $\beta = \frac{\beta_{\ell; 1} \kappa^2}{\lambda (n+1)}$ leads to upper-bound the thickness $\thick{(2)}{\lambda; \alpha}$ as
\begin{align*}
    \thick{(2)}{\lambda; \alpha}
    &
    \leq
    \frac{12 }{1 - \frac{\beta_{\ell; 1} \kappa^2}{\lambda (n+1)}} \widehat{\tau}^{(2)}
    \\
    &
    \leq
    \frac{12}{
        1 - \frac{\beta_{\ell; 1} \kappa^2}{\lambda (n+1)}
    }
    \min \paren{
        O \paren{
            \frac{
                \kappa^6 \rho^2 \xi_{\ell}
            }{
                \lambda^3 (n+1)^2
            }
        },
        \frac{2 \kappa^2 \rho}{\lambda (n+1)}
    }
    \\
    &
    \leq
    \paren{
        12 + O\paren{\frac{\beta_{\ell; 1} \kappa^2}{\lambda (n+1)}}
    }
    \min \paren{
        O \paren{
            \frac{
                \kappa^6 \rho^2 \xi_{\ell}
            }{
                \lambda^3 (n+1)^2
            }
        },
        O \paren{\frac{\kappa^2 \rho}{\lambda (n+1)}}
    }
    \\
    &
    \leq
    \min \paren{
            O \paren{
            \frac{
                \kappa^6 \rho^2 \xi_{\ell}
            }{
                \lambda^3 (n+1)^2
            }
        },
            O \paren{\frac{\kappa^2 \rho}{\lambda (n+1)}}
    }.
    \quad \mbox{a.s.}
\end{align*}
\end{proof}

\subsection{Proof of Corollary~\ref{cor.very.smooth}}
\begin{proof}
    \label{proof.cor.very.smooth}
    Under the assumptions mentionned in Theorem~\ref{thm.very.smooth},
    the prediction-region $\tilde{C}_{\lambda; \alpha}^{\mathrm{up, (2)}}\paren{X_{n+1}}$
    is an approximate full-conformal prediction-region.

    Since it is assumed that
    the exists a constant $c \in \paren{0, +\infty}$
    such that the conditional probability density function
    of the output $p\paren{\bullet | D, X_{n+1}} : \mathcal{Y} \mapsto \mathbb{R}_+$
    is bounded from above by said constant $c$, and
    that the non-conformity scores are almost-surely distinct, 
    it follows from Lemma~\ref{lm.confidence.gap} that
    \begin{align*}
        \mathbb{P}\croch{
            Y_{n+1}\in \tilde{C}_{\alpha}^{\mathrm{up, (2)}}\paren{X_{n+1}}
        }
        &
        \leq
        1 - \alpha + c
        \mathbb{E}_{D, X_{n+1}}\croch{
            \mathrm{THK}_{\lambda; \alpha}^{\mathrm{(2)}}
        }
        \\
        &
        \leq
        1 - \alpha
        + c \min \paren{
            O \paren{
            \frac{
                \kappa^6 \rho^2 \xi_{\ell}
            }{
                \lambda^3 (n+1)^2
            }
        },
            O \paren{\frac{\kappa^2 \rho}{\lambda (n+1)}}
        },
    \end{align*}
    where the second inequality
    follows from the upper-bound
    on the thickness provided in Eq.~\eqref{eq.bound.thickness.very.smooth}
    which holds under the assumptions mentionned in Theorem~\ref{thm.very.smooth}.
\end{proof}

\section{Examples of loss-functions and further numerical experiments}
\label{sec.example.loss.function}
The present section reviews the loss-functions provided
at the Section~\ref{sec.predictor}. For each example,
it provides the values for the smoothness constants
discussed before and if not already provided in the body of the text,
the results of some experiments.
To be more specific, Section~\ref{sec.logcosh} deals with the Logcosh loss-function. Section~\ref{sec.pseudo.huber} deals with the pseudo-Huber loss-function, while Section~\ref{sec.smoothed.pinball} deals with the smoothed-pinball loss-function.

\subsection{Logcosh loss \citep{saleh2022statistical}}
\label{sec.logcosh}
Let $a \in \paren{0, +\infty}$.
For every $\paren{y, u} \in \mathcal{Y} \times \mathcal{Y}$,
\begin{align*}
    \ell\paren{y, u} = a\log\paren{
        \cosh\paren{\frac{y - u}{a}}
    }.
\end{align*}

\begin{lemma}
\label{lm.logcosh.d1}
For every $y \in \mathcal{Y}$,
the function $u \in \mathcal{Y} \mapsto \ell\paren{y, u} \in \mathbb{R}$
is $\rho$-Lipschitz continuous with $\rho = 1$.
\end{lemma}
\begin{proof}
    For every $\paren{y, u} \in \mathcal{Y} \times \mathcal{Y}$,
    \begin{align*}
        \partial_2\ell\paren{y, u} = -\tanh\paren{\frac{y - u}{a}}.
    \end{align*}
    It follows that
    \begin{align*}
        \abss{\partial_2\ell\paren{y, u}} = 
        \abss{\tanh\paren{\frac{y - u}{a}}} \leq 1, 
    \end{align*}
    therefore the function $u \in \mathcal{Y} \mapsto \ell\paren{y, u} \in \mathbb{R}$
    is $\rho$-Lipschitz continuous with $\rho = 1$.
\end{proof}

\begin{lemma}
\label{lm.logcosh.d2}
For every $y \in \mathcal{Y}$,
the function $u \in \mathcal{Y} \mapsto \ell\paren{y, u} \in \mathbb{R}$ is convex, and
the function $u \in \mathcal{Y} \mapsto \partial_2 \ell\paren{y, u} \in \mathbb{R}$
is $\beta_{\ell; 2}$-Lipschitz continuous with $\beta_{\ell; 2} = \frac{1}{a}$.
\end{lemma}
\begin{proof}
   For every $\paren{y, u} \in \mathcal{Y} \times \mathcal{Y}$,
\begin{align*}
    \partial_2^2\ell\paren{y, u} = \frac{\mathrm{sech}^2\paren{\frac{y - u}{a}}}{a} \geq 0,
\end{align*}
therefore, the function $u \in \mathcal{Y} \mapsto \ell\paren{y, u} \in \mathbb{R}$ is convex.
Moreover,
\begin{align*}
    \abss{
        \partial_2^2\ell\paren{y, u}
    } = \abss{
        \frac{\mathrm{sech}^2\paren{\frac{y - u}{a}}}{a}
    } \leq \frac{1}{a},
\end{align*}
therefore, the function $u \in \mathcal{Y} \mapsto \partial_2 \ell\paren{y, u} \in \mathbb{R}$
is $\beta_{\ell; 2}$-Lipschitz continuous with $\beta_{\ell; 2} = \frac{1}{a}$. 
\end{proof}

\begin{lemma}
\label{lm.logcosh.d3}
For every $y \in \mathcal{Y}$,
the function $u \in \mathcal{Y} \mapsto \partial_2^2 \ell\paren{y, u} \in \mathbb{R}$
is $\xi_{\ell}$-Lipschitz continuous with
\begin{align*}
    \xi_{\ell} = \frac{2 \operatorname{sech}^{2}\left(\operatorname{arcsinh}\left(\frac{1}{\sqrt{2}}\right)\right) \tanh\left(\operatorname{arcsinh}\left(\frac{1}{\sqrt{2}}\right)\right)}{a^{2}} \leq \frac{1}{a^2}.
\end{align*}
\end{lemma}
\begin{proof}
    For every $\paren{y, u} \in \mathcal{Y} \times \mathcal{Y}$,
    \begin{align*}
        \partial_2^3\ell\paren{y, u} = 
        \frac{2 \operatorname{sech}^{2}\left(\frac{y - u}{a}\right) \tanh\left(\frac{y - u}{a}\right)}{a^{2}}.
    \end{align*}
    It follows that, for every $\paren{y, u} \in \mathcal{Y} \times \mathcal{Y}$,
\begin{align*}
    \partial_2^4\ell\paren{y, u} = 
    \frac{4 \sinh^{2}\left(\frac{y - u}{a}\right) - 2}{a^{3} \cosh^{4}\left(\frac{y - u}{a}\right)},
\end{align*}
with
\begin{align*}
    \partial_2^4\ell\paren{y, y + \operatorname{arcsinh}\left(\frac{1}{\sqrt{2}}\right) \, a} &= 0,
    \\
    \partial_2^4\ell\paren{y, y - \operatorname{arcsinh}\left(\frac{1}{\sqrt{2}}\right) \, a} &= 0.
\end{align*}
and $\partial_2^4\ell\paren{y, u} \leq 0$ if and only if
$u \in \croch{
    y - \operatorname{arcsinh}\left(\frac{1}{\sqrt{2}}\right) \, a,
    y + \operatorname{arcsinh}\left(\frac{1}{\sqrt{2}}\right) \, a
}$.
It follows that the function
$u \in \mathcal{Y} \mapsto \partial_2^3\ell\paren{y, u} \in \mathbb{R}$
is increasing over $\left(
    - \infty, y - \operatorname{arcsinh}\left(\frac{1}{\sqrt{2}}\right) \, a
\right]$ and $\left[
    y + \operatorname{arcsinh}\left(\frac{1}{\sqrt{2}}\right) \, a,
    +\infty
\right)$,
and decreasing over $\croch{
    y - \operatorname{arcsinh}\left(\frac{1}{\sqrt{2}}\right) \, a,
    y + \operatorname{arcsinh}\left(\frac{1}{\sqrt{2}}\right) \, a
}$. It follows that for every $u \in \mathcal{Y}$
\begin{align*}
    \partial_2^3\ell\paren{y, y - \operatorname{arcsinh}\left(\frac{1}{\sqrt{2}}\right) \, a}
    \geq \partial_2^3\ell\paren{y, u}
    \geq \partial_2^3\ell\paren{y, y + \operatorname{arcsinh}\left(\frac{1}{\sqrt{2}}\right) \, a}
\end{align*}
with
\begin{align*}
    \partial_2^3\ell\paren{y, y - \operatorname{arcsinh}\left(\frac{1}{\sqrt{2}}\right) \, a}
    &
    = \frac{2 \operatorname{sech}^{2}\left(\frac{\operatorname{arcsinh}\left(\frac{1}{\sqrt{2}}\right) \, a}{a}\right) \tanh\left(\frac{\operatorname{arcsinh}\left(\frac{1}{\sqrt{2}}\right) \, a}{a}\right)}{a^{2}}
    \\
    &
    = \frac{2 \operatorname{sech}^{2}\left(\operatorname{arcsinh}\left(\frac{1}{\sqrt{2}}\right)\right) \tanh\left(\operatorname{arcsinh}\left(\frac{1}{\sqrt{2}}\right)\right)}{a^{2}} \geq 0,
\end{align*}
and
\begin{align*}
    \partial_2^3\ell\paren{y, y + \operatorname{arcsinh}\left(\frac{1}{\sqrt{2}}\right) \, a}
    &
    = - \frac{2 \operatorname{sech}^{2}\left(\operatorname{arcsinh}\left(\frac{1}{\sqrt{2}}\right)\right) \tanh\left(\operatorname{arcsinh}\left(\frac{1}{\sqrt{2}}\right)\right)}{a^{2}} \leq 0.
\end{align*}
Thus,
\begin{align*}
    \abss{
        \partial_2^3\ell\paren{y, u}
    } \leq \frac{2 \operatorname{sech}^{2}\left(\operatorname{arcsinh}\left(\frac{1}{\sqrt{2}}\right)\right) \tanh\left(\operatorname{arcsinh}\left(\frac{1}{\sqrt{2}}\right)\right)}{a^{2}},
\end{align*}
therefore, the function $u \in \mathcal{Y} \mapsto \partial_2^2\ell\paren{y, u} \in \mathbb{R}$
is $\xi_{\ell}$-Lipschitz continuous with
\begin{align*}
    \xi_{\ell} = \frac{2 \operatorname{sech}^{2}\left(\operatorname{arcsinh}\left(\frac{1}{\sqrt{2}}\right)\right) \tanh\left(\operatorname{arcsinh}\left(\frac{1}{\sqrt{2}}\right)\right)}{a^{2}} \leq \frac{1}{a^2}.
\end{align*}
\end{proof}

\subsection{Pseudo-Huber loss \citep{charbonnier1994two}}
\label{sec.pseudo.huber}
Let $a \in \paren{0, +\infty}$.
For every $\paren{y, u} \in \mathcal{Y} \times \mathcal{Y}$,
\begin{align*}
    \ell\paren{y, u} = a^2 \paren{
        \sqrt{1 + \paren{\frac{y - u}{a}}^2} - 1
    }.
\end{align*}

\begin{lemma}
\label{lm.pseudo.huber.d1}
For every $y \in \mathcal{Y}$,
the function $u \in \mathcal{Y} \mapsto \ell\paren{y, u} \in \mathbb{R}$
is $\rho$-Lipschitz continuous with $\rho = a$.
\end{lemma}
\begin{proof}
For every $\paren{y, u} \in \mathcal{Y} \times \mathcal{Y}$,
\begin{align*}
    \partial_2\ell\paren{y, u} = -\frac{y - u}{\sqrt{\frac{\left(y - u\right)^{2}}{a^{2}} + 1}}.
\end{align*}
It follows that, for every $\paren{y, u} \in \mathcal{Y} \times \mathcal{Y}$,
\begin{align*}
    \partial_2^2\ell\paren{y, u} = 
    \frac{1}{\left(\frac{\left(y - u\right)^{2}}{a^{2}} + 1\right)^{\frac{3}{2}}}.
\end{align*}
Thus, for every $u \in \mathcal{Y}$,
$\partial_2^2\ell\paren{y, u} > 0$. It follows that
the function $u \in \mathcal{Y} \mapsto \partial_2\ell\paren{y, u}$
is increasing over $\mathcal{Y}$. It follows that
for every $u \in \mathcal{Y}$
\begin{align*}
    \lim_{u \to - \infty}
    \partial_2\ell\paren{y, u}
    \leq \partial_2\ell\paren{y, u}
    \leq
    \lim_{u \to +\infty}\partial_2\ell\paren{y, u},
\end{align*}
with
\begin{align*}
    \lim_{u \to - \infty}
    \partial_2\ell\paren{y, u}
    &
    = \lim_{u \to - \infty}
    -\frac{y - u}{\sqrt{\frac{\left(y - u\right)^{2}}{a^{2}} + 1}}
    = -a,
    \\
    \lim_{u \to - \infty}
    \partial_2\ell\paren{y, u}
    &
    = \lim_{u \to + \infty}
    -\frac{y - u}{\sqrt{\frac{\left(y - u\right)^{2}}{a^{2}} + 1}}
    = a.
\end{align*}
It follows that,
\begin{align*}
    \abss{
        \partial_2\ell\paren{y, u}
    } \leq a,
\end{align*}
therefore, the function $u \in \mathcal{Y} \mapsto \ell\paren{y, u} \in \mathbb{R}$
is $\rho$-Lipschitz continuous with $\rho = a$.
\end{proof}

\begin{lemma}
\label{lm.pseudo.huber.d2}
For every $y \in \mathcal{Y}$,
the function $u \in \mathcal{Y} \mapsto \ell\paren{y, u}$ is convex, and
the function $u \in \mathcal{Y} \mapsto \partial_2\ell\paren{y, u}$
is $\beta_{\ell; 2}$-Lipschitz continuous with $\beta_{\ell; 2} = 1$.
\end{lemma}
\begin{proof}
For every $\paren{y, u} \in \mathcal{Y} \times \mathcal{Y}$,
\begin{align*}
    \partial_2^2\ell\paren{y, u} = 
    \frac{1}{\left(\frac{\left(y - u\right)^{2}}{a^{2}} + 1\right)^{\frac{3}{2}}} \geq 0,
\end{align*}
therefore, the function $u \in \mathcal{Y} \mapsto \ell\paren{y, u}$ is convex.
Moreover,
\begin{align*}
    \abss{
        \partial_2^2\ell\paren{y, u}
    } = \abss{
        \frac{1}{\left(\frac{\left(y - u\right)^{2}}{a^{2}} + 1\right)^{\frac{3}{2}}}
    } \leq 1,
\end{align*}
therefore, the function $u \in \mathcal{Y} \mapsto \partial_2\ell\paren{y, u}$
is $\beta_{\ell; 2}$-Lipschitz continuous with $\beta_{\ell; 2} = 1$.
\end{proof}

\begin{lemma}
\label{lm.pseudo.huber.d3}
For every $y \in \mathcal{Y}$,
the function $u \in \mathcal{Y} \mapsto \partial_2^2 \ell\paren{y, u}\in \mathbb{R}$
is $\xi_{\ell}$-Lipschitz with $\xi_{\ell} = \frac{1}{a} \frac{3}{2} \paren{\frac{4}{5}}^{5/2}$.
\end{lemma}
\begin{proof}
For every $\paren{y, u} \in \mathcal{Y} \times \mathcal{Y}$,
\begin{align*}
    \partial_2^3\ell\paren{y, u} = 
    \frac{3 \left(y - u\right)}{a^{2} \left(\frac{\left(y - u\right)^{2}}{a^{2}} + 1\right)^{\frac{5}{2}}}.
\end{align*}
It follows that, for every $\paren{y, u} \in \mathcal{Y} \times \mathcal{Y}$,
\begin{align*}
    \partial_2^4\ell\paren{y, u} = 
    \frac{3a^{3} \left(2u - 2y - a\right) \left(2u - 2y + a\right)}{\left(\left(u - y\right)^{2} + a^{2}\right)^{\frac{7}{2}}}.
\end{align*}
with
\begin{align*}
    \partial_2^4\ell\paren{y, \frac{2y + a}{2}} &= 0,
    \\
    \partial_2^4\ell\paren{y, -\frac{a - 2y}{2}} &= 0,
\end{align*}
and $\partial_2^4\ell\paren{y, u} \leq 0$ if and only if
$u \leq \croch{
    -\frac{a - 2y}{2},
    \frac{2y + a}{2}
}$. It follows that the function
$u \in \mathcal{Y} \mapsto \partial_2^3 \ell\paren{y, u} \in \mathbb{R}$
is increasing over $\left(
    - \infty,
    -\frac{a - 2y}{2}
\right]$ and $\left[
    \frac{2y + a}{2},
    + \infty
\right)$, and decreasing over
$\croch{
    -\frac{a - 2y}{2},
    \frac{2y + a}{2}
}$. It follows that for every $u \in \mathcal{Y}$
\begin{align*}
    \partial_2^3\ell\paren{y,
    -\frac{a - 2y}{2}
    }
    \geq
    \partial_2^3\ell\paren{y, u}
    \geq
    \partial_2^3\ell\paren{y,
    \frac{2y + a}{2}
    },
\end{align*}
with
\begin{align*}
    \partial_2^3\ell\paren{y,
    -\frac{a - 2y}{2}
    }
    &
    = \frac{1}{a} \frac{3}{2} \paren{\frac{4}{5}}^{5/2} \geq 0
    \\
    \partial_2^3\ell\paren{y,
    \frac{2y + a}{2}
    }
    &
    = -\frac{1}{a} \frac{3}{2} \paren{\frac{4}{5}}^{5/2} \leq 0.
\end{align*}
Thus
\begin{align*}
    \abss{
        \partial_2^3\ell\paren{y, u}
    } \leq \frac{1}{a} \frac{3}{2} \paren{\frac{4}{5}}^{5/2},
\end{align*}
therefore, the function $u \in \mathcal{Y} \mapsto \partial_2^2 \ell\paren{y, u}\in \mathbb{R}$
is $\xi_{\ell}$-Lipschitz with $\xi_{\ell} = \frac{1}{a} \frac{3}{2} \paren{\frac{4}{5}}^{5/2}$.
\end{proof}

\begin{figure}[H]
    \centering
    \includegraphics[width=0.60\textwidth]{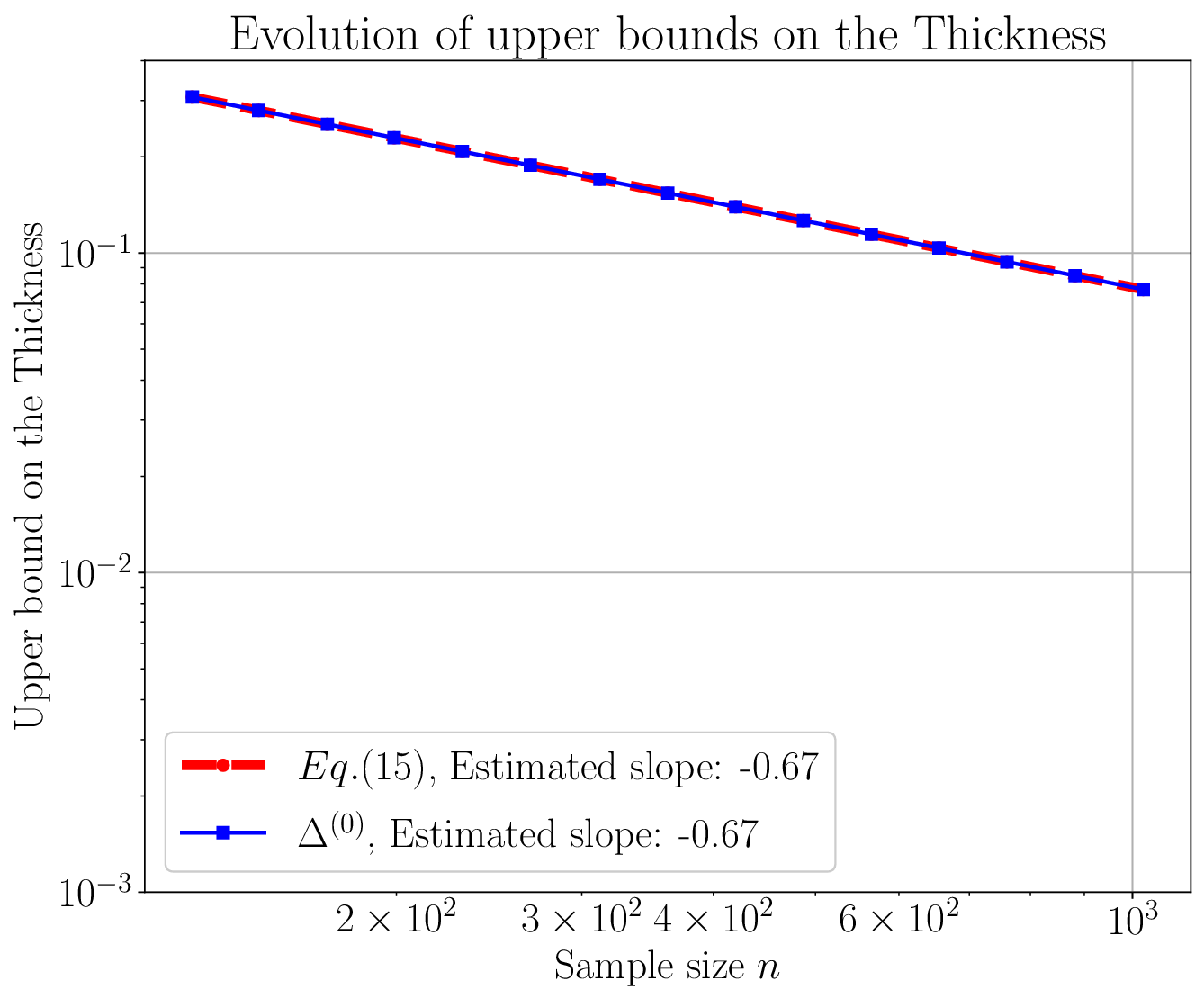}
    \caption{
        Evolution of
        the upper-bound in Eq.~\eqref{eq.bound.thickness.non.smooth} (dashed red-line)
        and the quantity $\Delta^{(0)}$ (solid blue-line)
        as a function of the sample size $n$
        in $\log\log$ scale (to appreciate the rate).
        The data is sampled from $\mathrm{sklearn}$ synthetic data set
        make\_friedman1(sample\_size=n).
        The kernel $k_\mathcal{H} \paren{\bullet, \bullet}$
        is set to be the Laplacian kernel (gamma=None).
        The loss-function $\ell\paren{\bullet, \bullet}$
        is set to be the pseudo-Huber loss ($a=1.0$).
        The regularization parameter is set to decay as $\lambda \propto (n+1)^{-0.33}$.
        The fixed output value is set at $z=0$.        
        The non-conformity function is set to be $s(\bullet, \bullet) : (y, u) \mapsto \abss{y - u}$.
    }
    \label{fig.comp.bound.non.smooth.2}
\end{figure}

\begin{figure}[H]
    \centering
    \includegraphics[width=0.60\textwidth]{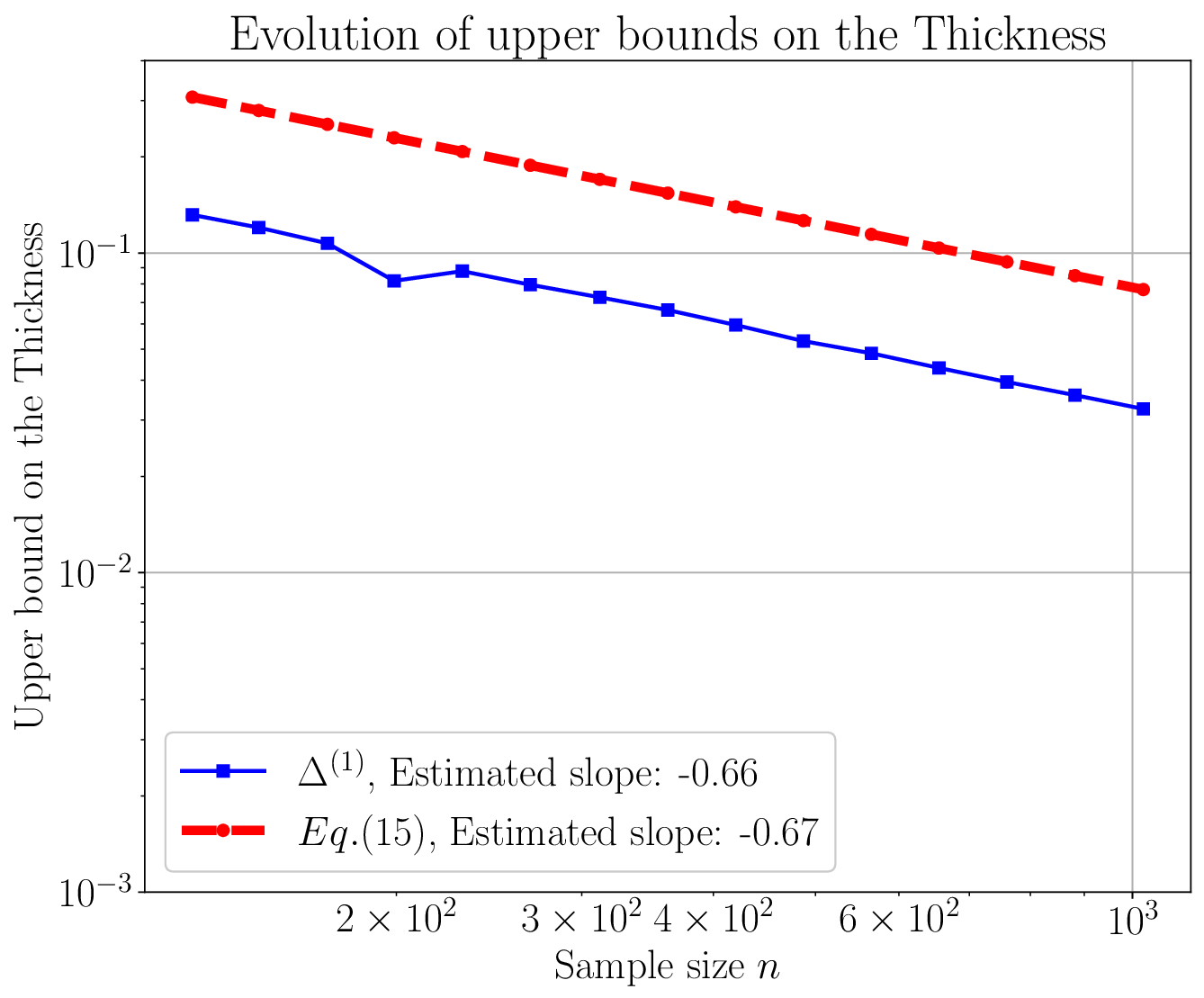}
    \caption{
        Evolution of
        the upper-bound in Eq.~\eqref{eq.bound.thickness.non.smooth} (dashed red-line)
        and the quantity $\Delta^{(1)}$ (solid blue-line)
        as a function of the sample size $n$
        in $\log\log$ scale (to appreciate the rate).
        The data is sampled from $\mathrm{sklearn}$ synthetic data set
        make\_friedman1(sample\_size=n).
        The kernel $k_\mathcal{H} \paren{\bullet, \bullet}$
        is set to be the Laplacian kernel (gamma=None).
        The loss-function $\ell\paren{\bullet, \bullet}$
        is set to be the pseudo-Huber loss ($a=1.0$).
        The regularization parameter is set to decay as $\lambda \propto (n+1)^{-0.33}$.
        The fixed output value is set at $z=0$.        
        The non-conformity function is set to be $s(\bullet, \bullet) : (y, u) \mapsto \abss{y - u}$.
    }
    \label{fig.comp.bound.smooth.2}
\end{figure}

\begin{figure}[H]
    \centering
    \includegraphics[width=0.60\textwidth]{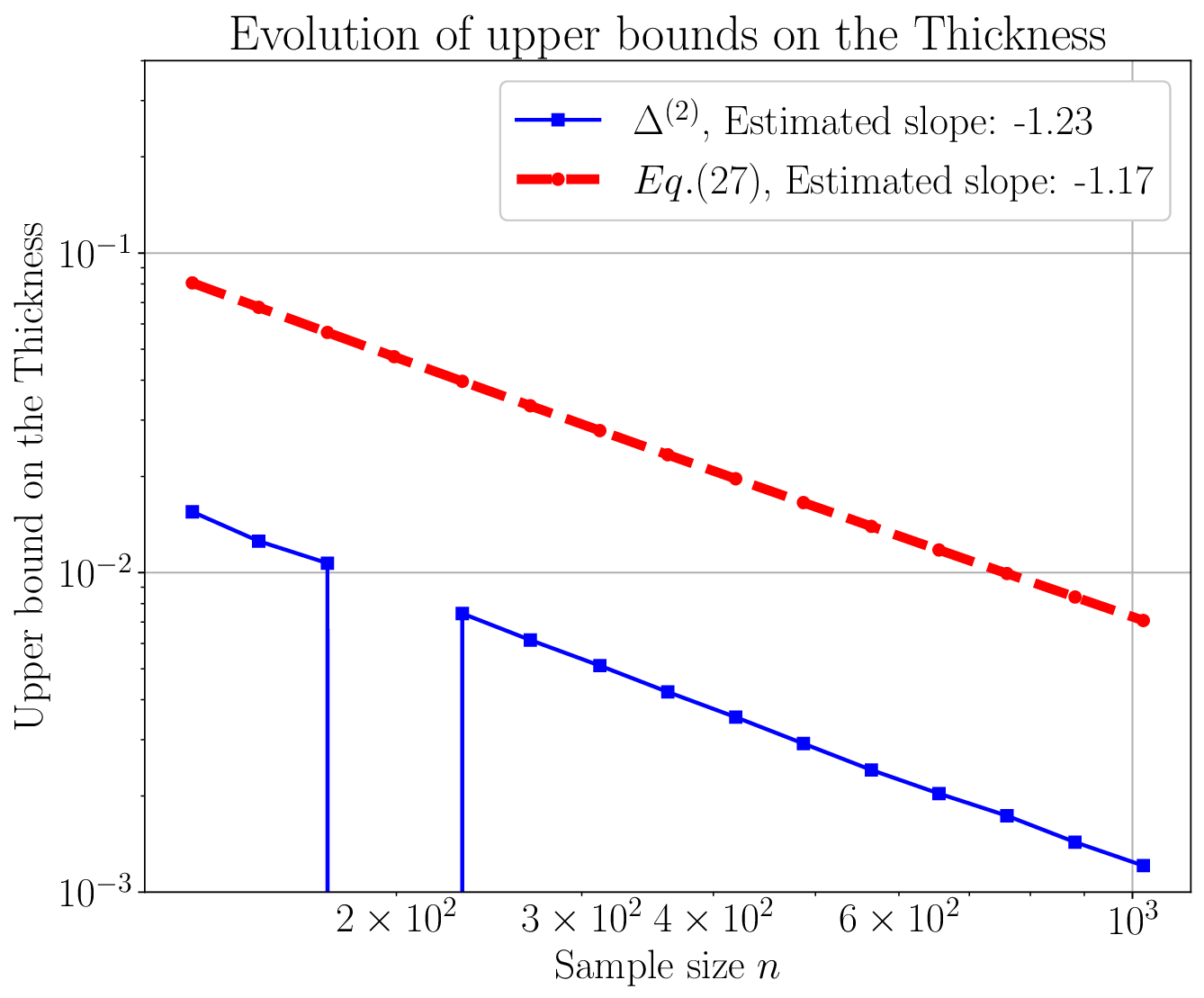}
    \caption{
        Evolution of
        the upper-bound in Eq.~\eqref{eq.bound.thickness.very.smooth} (dashed red-line)
        and the quantity $\Delta^{(2)}$ (solid blue-line)
        as a function of the sample size $n$
        in $\log\log$ scale (to appreciate the rate).
        The data is sampled from $\mathrm{sklearn}$ synthetic data set
        make\_friedman1(sample\_size=n).
        The kernel $k_\mathcal{H} \paren{\bullet, \bullet}$
        is set to be the Laplacian kernel (gamma=None).
        The loss-function $\ell\paren{\bullet, \bullet}$
        is set to be the pseudo-Huber loss ($a=1.0$).
        The regularization parameter is set to decay as $\lambda \propto (n+1)^{-0.33}$.
        The fixed output value is set at $z=0$.        
        The non-conformity function is set to be $s(\bullet, \bullet) : (y, u) \mapsto \abss{y - u}$.
    }
    \label{fig.comp.bound.very.smooth.2}
\end{figure}

\subsection{Smoothed-pinball loss \citep{zheng2011gradient}}
\label{sec.smoothed.pinball}
Let $a \in \paren{0, +\infty}$ and $t \in \paren{0, 1}$.
For every $\paren{y, u} \in \mathcal{Y} \times \mathcal{Y}$,
\begin{align*}
    \ell\paren{y, u} = t \paren{y - u} +
    a \log\paren{
        1 + \mathrm{e}^{-\frac{y - u}{a}}
    }.
\end{align*}

\begin{lemma}
\label{lm.smooth.pinball.d1}
For every $y \in \mathcal{Y}$,
the function $u \in \mathcal{Y} \mapsto \ell\paren{y, u} \in \mathbb{R}$
is $\rho$-Lipschitz continuous with $\rho = \max\paren{t , 1 - t}$.
\end{lemma}
\begin{proof}
    For every $\paren{y, u} \in \mathcal{Y} \times \mathcal{Y}$,
\begin{align*}
    \partial_2\ell\paren{y, u} =
    \frac{\mathrm{e}^{-\frac{y - u}{a}}}{\mathrm{e}^{-\frac{y - u}{a}} + 1} - t.
\end{align*}
It follows that
\begin{align*}
    \abss{\partial_2\ell\paren{y, u}}
    \leq \max\paren{t , 1 - t},
\end{align*}
therefore, the function $u \in \mathcal{Y} \mapsto \ell\paren{y, u} \in \mathbb{R}$
is $\rho$-Lipschitz continuous with $\rho = \max\paren{t , 1 - t}$.
\end{proof}

\begin{lemma}
\label{lm.smooth.pinball.d2}
For every $y \in \mathcal{Y}$,
the function $u \in \mathcal{Y} \mapsto \ell\paren{y, u} \in \mathbb{R}$ is convex, and
the function $u \in \mathcal{Y} \mapsto \partial_2\ell\paren{y, u} \in \mathbb{R}$
is $\beta_{\ell; 2}$-Lipschitz continuous with $\beta_{\ell; 2} = \frac{1}{4a}$
\end{lemma}
\begin{proof}
For every $\paren{y, u} \in \mathcal{Y} \times \mathcal{Y}$,
\begin{align*}
    \partial_2^2\ell\paren{y, u} =
    \frac{\mathrm{e}^{\frac{y - u}{a}}}{a \left(\mathrm{e}^{\frac{y - u}{a}} + 1\right)^{2}} \geq 0,
\end{align*}
therefore, the function $u \in \mathcal{Y} \mapsto \ell\paren{y, u} \in \mathbb{R}$ is convex.
Moreover,
\begin{align*}
    \abss{\partial_2^2\ell\paren{y, u}}
    \leq \frac{1}{4a},
\end{align*}
therefore, the function $u \in \mathcal{Y} \mapsto \partial_2\ell\paren{y, u} \in \mathbb{R}$
is $\beta_{\ell; 2}$-Lipschitz continuous with $\beta_{\ell; 2} = \frac{1}{4a}$.    
\end{proof}

\begin{lemma}
\label{lm.smooth.pinball.d3}
For every $y \in \mathcal{Y}$,
the function $u \in \mathcal{Y} \mapsto \partial_2^2 \ell\paren{y, u}\in \mathbb{R}$
is $\xi_{\ell}$-Lipschitz with $\xi_{\ell} = \frac{1}{a^2}
\frac{5 + 3\sqrt{3}}{\paren{\sqrt{3}+3}^3}$.
\end{lemma}
\begin{proof}
For every $\paren{y, u} \in \mathcal{Y} \times \mathcal{Y}$,
\begin{align*}
    \partial_2^3\ell\paren{y, u} =
    \frac{\left(\mathrm{e}^{\frac{y - u}{a}} - 1\right) \mathrm{e}^{\frac{y - u}{a}}}{a^{2} \left(\mathrm{e}^{\frac{y - u}{a}} + 1\right)^{3}}.
\end{align*}
It follows that, for every $\paren{y, u} \in \mathcal{Y} \times \mathcal{Y}$,
\begin{align*}
    \partial_2^4\ell\paren{y, u} =
    \frac{\mathrm{e}^{\frac{y - u}{a}} \left(\mathrm{e}^{\frac{2 \left(y - u\right)}{a}} - 4\mathrm{e}^{\frac{y - u}{a}} + 1\right)}{a^{3} \left(\mathrm{e}^{\frac{y - u}{a}} + 1\right)^{4}}.
\end{align*}
with
\begin{align*}
    \partial_2^4\ell\paren{y, y - \ln\left(\sqrt{3} + 2\right) \, a} &= 0,
    \\
    \partial_2^4\ell\paren{y, y - \ln\left(2 - \sqrt{3}\right) \, a} &= 0,
\end{align*}
and $\partial_2^4\ell\paren{y, u} \leq 0$ if and only if
$u \leq \croch{
    y - \ln\left(\sqrt{3} + 2\right) \, a,
    y - \ln\left(2 - \sqrt{3}\right) \, a
}$. It follows that the function
$u \in \mathcal{Y} \mapsto \partial_2^3 \ell\paren{y, u} \in \mathbb{R}$
is increasing over $\left(
    - \infty,
    y - \ln\left(\sqrt{3} + 2\right) \, a
\right]$ and $\left[
    y - \ln\left(2 - \sqrt{3}\right) \, a,
    + \infty
\right)$, and decreasing over
$\croch{
    y - \ln\left(\sqrt{3} + 2\right) \, a,
    y - \ln\left(2 - \sqrt{3}\right) \, a
}$. It follows that for every $u \in \mathcal{Y}$
\begin{align*}
    \partial_2^3\ell\paren{y,
    y - \ln\left(\sqrt{3} + 2\right) \, a
    }
    \geq
    \partial_2^3\ell\paren{y, u}
    \geq
    \partial_2^3\ell\paren{y,
    y - \ln\left(2 - \sqrt{3}\right) \, a
    },
\end{align*}
with
\begin{align*}
    \partial_2^3\ell\paren{y,
    y - \ln\left(\sqrt{3} + 2\right) \, a
    }
    &
    = \frac{1}{a^2}
    \frac{5 + 3\sqrt{3}}{\paren{\sqrt{3}+3}^3}
    \geq 0
    \\
    \partial_2^3\ell\paren{y,
    y - \ln\left(2 - \sqrt{3}\right) \, a
    }
    &
    = -\frac{1}{a^2}
    \frac{5 + 3\sqrt{3}}{\paren{\sqrt{3}+3}^3} \leq 0.
\end{align*}
Thus
\begin{align*}
    \abss{
        \partial_2^3\ell\paren{y, u}
    } \leq \frac{1}{a^2}
    \frac{5 + 3\sqrt{3}}{\paren{\sqrt{3}+3}^3},
\end{align*}
therefore, the function $u \in \mathcal{Y} \mapsto \partial_2^2 \ell\paren{y, u}\in \mathbb{R}$
is $\xi_{\ell}$-Lipschitz with $\xi_{\ell} = \frac{1}{a^2}
    \frac{5 + 3\sqrt{3}}{\paren{\sqrt{3}+3}^3}$.
\end{proof}

\begin{figure}[H]
    \centering
    \includegraphics[width=0.60\textwidth]{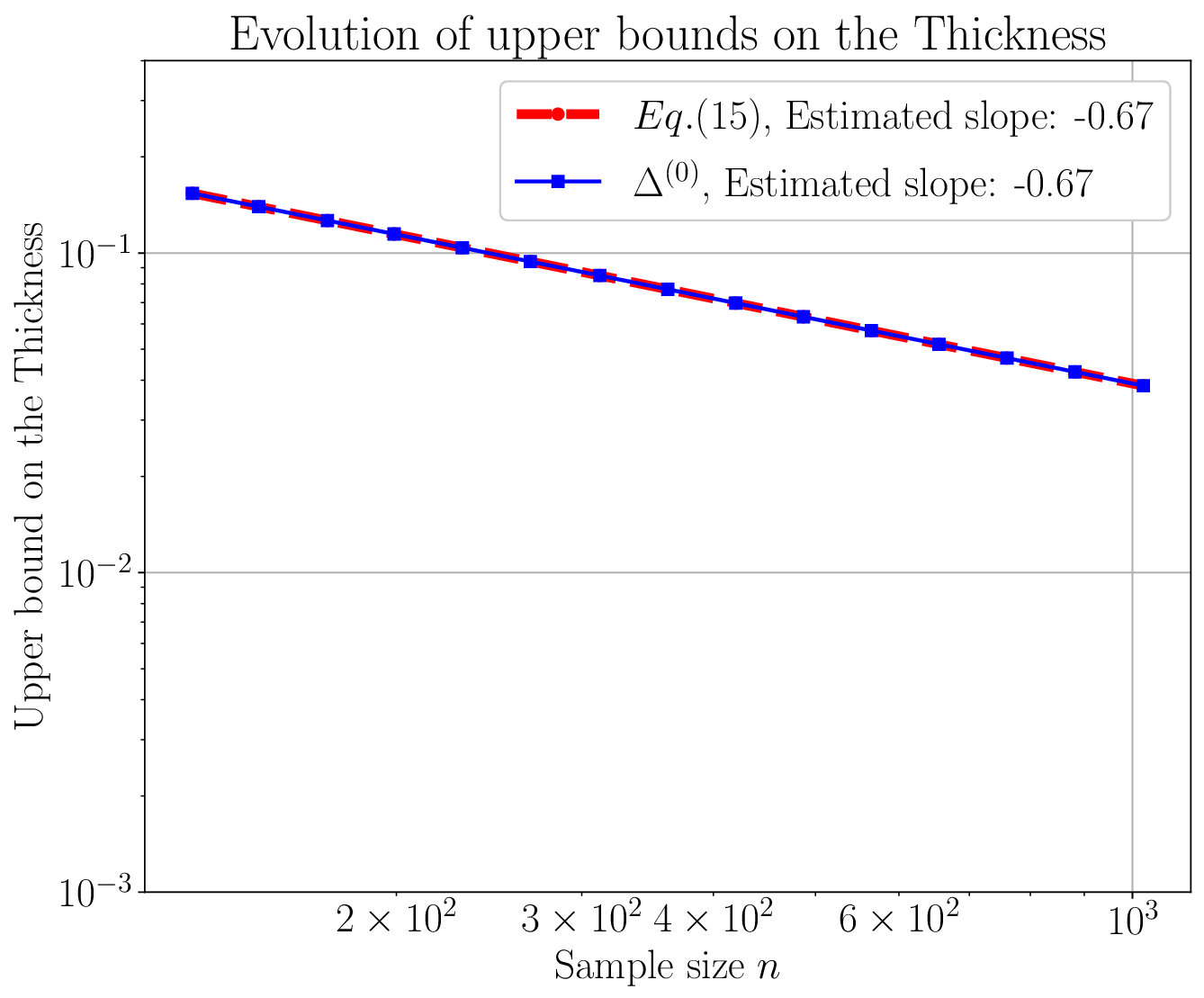}
    \caption{
        Evolution of
        the upper-bound in Eq.~\eqref{eq.bound.thickness.non.smooth} (dashed red-line)
        and the quantity $\Delta^{(0)}$ (solid blue-line)
        as a function of the sample size $n$
        in $\log\log$ scale (to appreciate the rate).
        The data is sampled from $\mathrm{sklearn}$ synthetic data set
        make\_friedman1(sample\_size=n).
        The kernel $k_\mathcal{H} \paren{\bullet, \bullet}$
        is set to be the Laplacian kernel (gamma=None).
        The loss-function $\ell\paren{\bullet, \bullet}$
        is set to be the smoothed pinball loss ($t = 5, a = 0.2$).
        The regularization parameter is set to decay as $\lambda \propto (n+1)^{-0.33}$.
        The fixed output value is set at $z=0$.        
        The non-conformity function is set to be $s(\bullet, \bullet) : (y, u) \mapsto \abss{y - u}$.
    }
    \label{fig.comp.bound.non.smooth.1}
\end{figure}

\begin{figure}[H]
    \centering
    \includegraphics[width=0.60\textwidth]{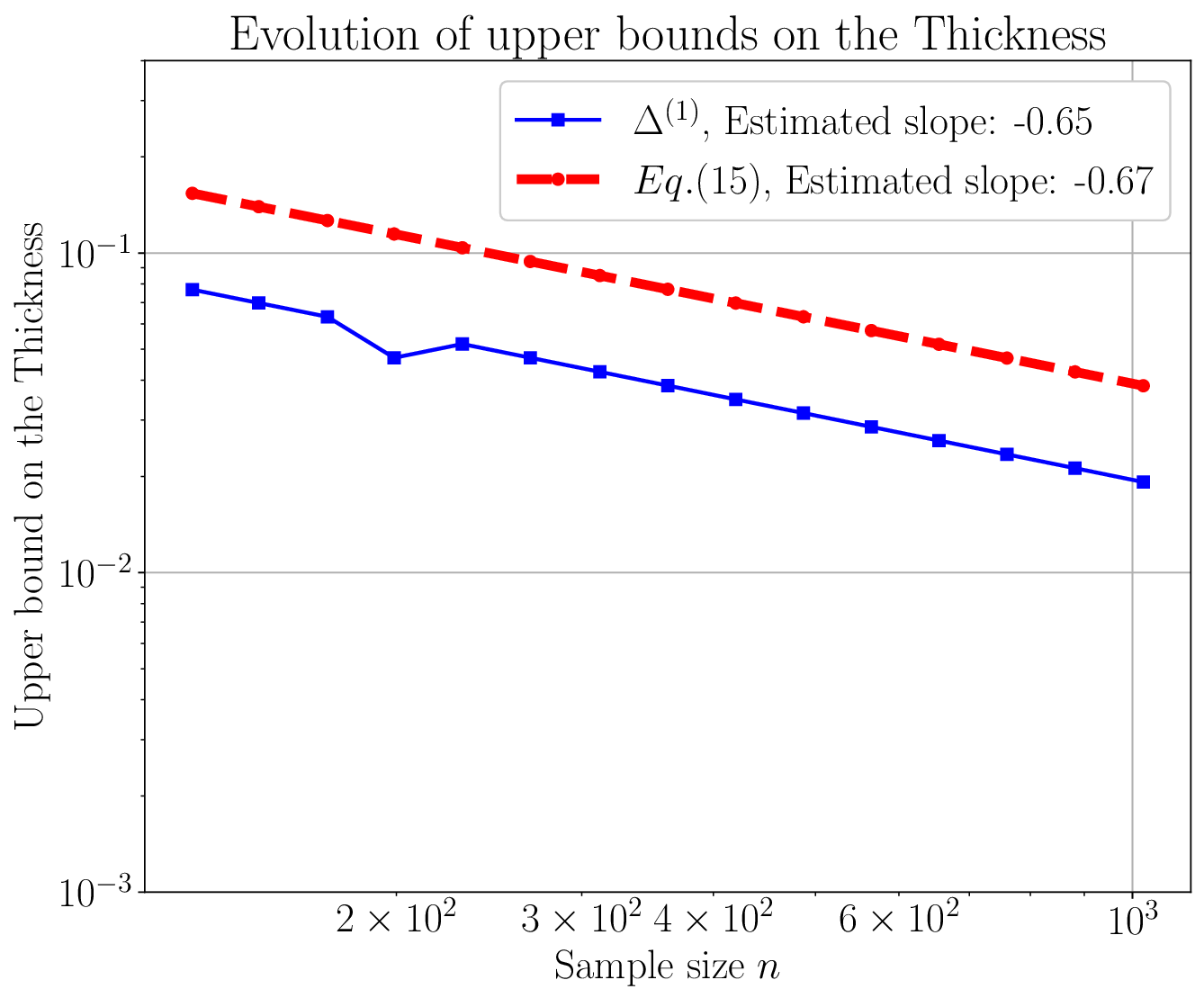}
    \caption{
        Evolution of
        the upper-bound in Eq.~\eqref{eq.bound.thickness.non.smooth} (dashed red-line)
        and the quantity $\Delta^{(1)}$ (solid blue-line)
        as a function of the sample size $n$
        in $\log\log$ scale (to appreciate the rate).
        The data is sampled from $\mathrm{sklearn}$ synthetic data set
        make\_friedman1(sample\_size=n).
        The kernel $k_\mathcal{H} \paren{\bullet, \bullet}$
        is set to be the Laplacian kernel (gamma=None).
        The loss-function $\ell\paren{\bullet, \bullet}$
        is set to be the smoothed pinball loss ($t = 5, a = 0.2$).
        The regularization parameter is set to decay as $\lambda \propto (n+1)^{-0.33}$.
        The fixed output value is set at $z=0$.        
        The non-conformity function is set to be $s(\bullet, \bullet) : (y, u) \mapsto \abss{y - u}$.
    }
    \label{fig.comp.bound.smooth.1}
\end{figure}

\begin{figure}[H]
    \centering
    \includegraphics[width=0.60\textwidth]{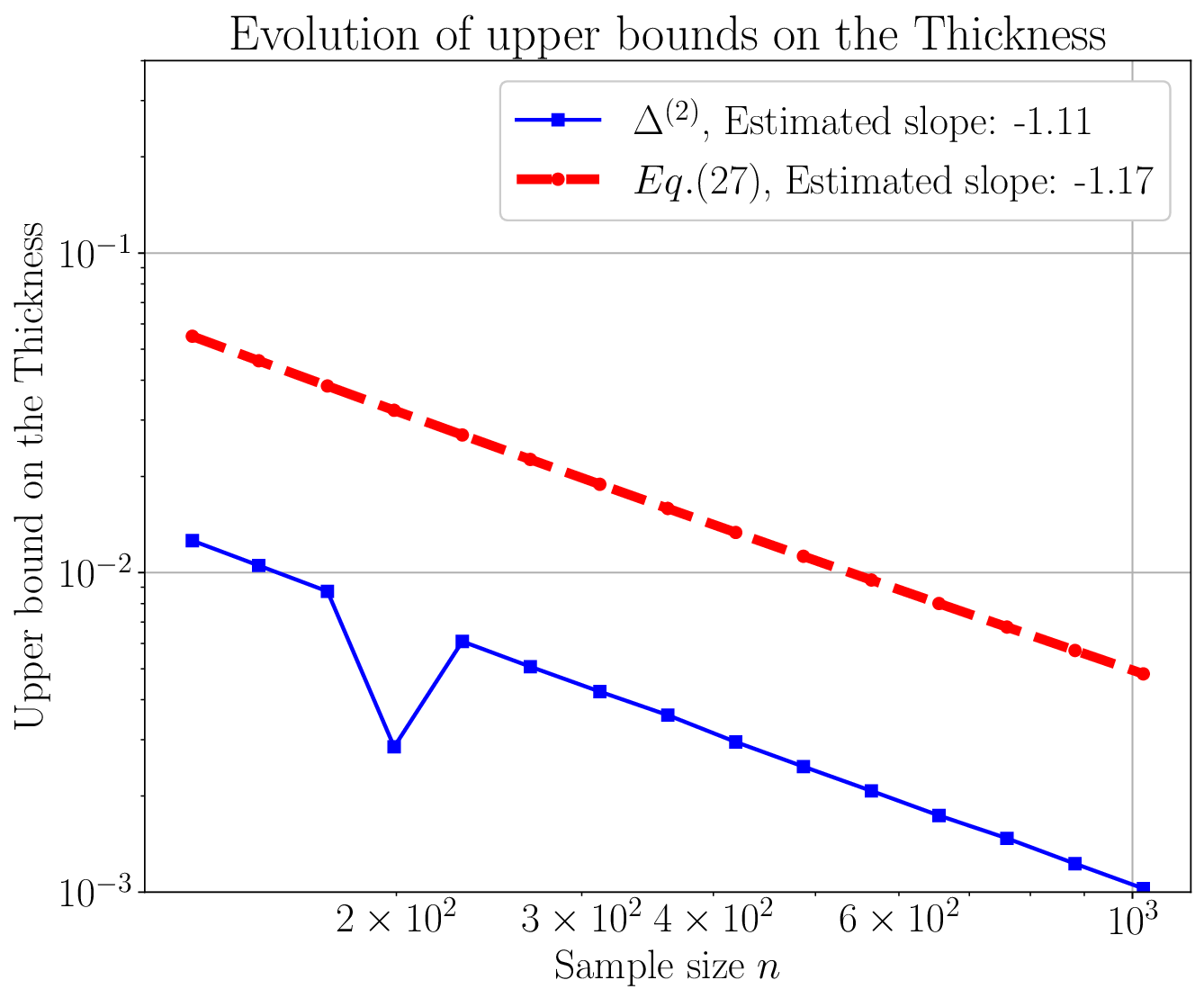}
    \caption{
        Evolution of
        the upper-bound in Eq.~\eqref{eq.bound.thickness.very.smooth} (dashed red-line)
        and the quantity $\Delta^{(2)}$ (solid blue-line)
        as a function of the sample size $n$
        in $\log\log$ scale (to appreciate the rate).
        The data is sampled from $\mathrm{sklearn}$ synthetic data set
        make\_friedman1(sample\_size=n).
        The kernel $k_\mathcal{H} \paren{\bullet, \bullet}$
        is set to be the Laplacian kernel (gamma=None).
        The loss-function $\ell\paren{\bullet, \bullet}$
        is set to be the smoothed pinball loss ($t = 5, a = 0.2$).
        The regularization parameter is set to decay as $\lambda \propto (n+1)^{-0.33}$.
        The fixed output value is set at $z=0$.        
        The non-conformity function is set to be $s(\bullet, \bullet) : (y, u) \mapsto \abss{y - u}$.
    }
    \label{fig.comp.bound.very.smooth.1}
\end{figure}

\vskip 0.2in
\bibliography{sample}

\end{document}